\documentclass[final,12pt]{colt2026} 

\title[Limitations of SGD for Multi-Index Models Beyond Statistical Queries]{Limitations of SGD for Multi-Index Models Beyond Statistical Queries}
\usepackage{times}

\usepackage[T1]{fontenc}
\usepackage{thm-restate}
\usepackage{enumitem}
\usepackage{amssymb}

\newcommand{\poly}{\mathrm{poly}}

\newcommand{\reals}{\mathbb{R}}
\newcommand{\R}{\mathbb{R}}

\newcommand{\N}{\mathbb{N}}

\newcommand{\zero}{\boldsymbol{0}}

\newcommand{\abs}[1]{\left| #1 \right|}

\newcommand{\vertiii}[1]{{\left\vert\kern-0.25ex\left\vert\kern-0.25ex\left\vert #1 
    \right\vert\kern-0.25ex\right\vert\kern-0.25ex\right\vert}}

\renewcommand{\L}{\mathcal{L}}
\newcommand{\Ecal}{\mathcal{E}}

\newcommand{\bmu}{\boldsymbol{\mu}}

\usepackage{xcolor}

\newcommand{\beq}{\begin{eqnarray*}}
\newcommand{\eeq}{\end{eqnarray*}}
\newcommand{\beqn}{\begin{eqnarray}}
\newcommand{\eeqn}{\end{eqnarray}}

\usepackage{ifmtarg}
\usepackage{xifthen}%
\newcommand{\ent}[1][]{%
\ifthenelse{\isempty{#1}}{%
\mathrm{H}
}{
\mathrm{H}^{(#1)}
}}

\newcommand{\loch}[1][]{%
\ifthenelse{\isempty{#1}}{%
\mathrm{h}
}{
\mathrm{h}^{(#1)}
}}

\newcommand{\Dcal}{\mathcal{D}}
\newcommand{\Fcal}{\mathcal{F}}

\newcommand{\Ical}{\mathcal{I}}

\newcommand{\Lcal}{\mathcal{L}}

\newcommand{\Ncal}{\mathcal{N}}

\newcommand{\bigo}{\mathcal{O}}

\newcommand{\E}{\mathbb{E}}

\newcommand{\sign}{\mathrm{sign}}

\newcommand{\var}{\text{Var}}

\newcommand{\Sphere}{\mathbb{S}}

\newcommand{\be}{\mathbf{e}}
\newcommand{\bx}{\mathbf{x}}

\newcommand{\bw}{\mathbf{w}}
\newcommand{\bg}{\mathbf{g}}

\newcommand{\bu}{\mathbf{u}}
\newcommand{\bv}{\mathbf{v}}
\newcommand{\bz}{\mathbf{z}}

\newcommand{\by}{\mathbf{y}}

\newcommand{\Ocal}{\mathcal{O}}

\newcommand{\Xcal}{\mathcal{X}}

\newcommand{\norm}[1]{\left\|#1\right\|}
\newcommand{\inner}[1]{\left\langle#1\right\rangle}

\newcommand{\polylog}{\text{polylog}}
\newcommand{\sym}{\text{Sym}}
\newcommand{\op}{\mathrm{op}}
\newcommand{\Row}{\mathrm{Row}}
\newcommand{\fsq}{{L_2}}
\newcommand{\kappaT}{\kappa_{\scriptscriptstyle T}}
\newcommand{\rank}{\text{rank}}

\newtheorem{assumption}{Assumption}

\newcommand{\secref}[1]{Section~\ref{#1}}

\renewcommand{\eqref}[1]{Eq.~(\ref{#1})}
\newcommand{\lemref}[1]{Lemma~\ref{#1}}
\newcommand{\thmref}[1]{Theorem~\ref{#1}}
\newcommand{\propref}[1]{Proposition~\ref{#1}}
\newcommand{\appref}[1]{Appendix~\ref{#1}}

\coltauthor{%
 \Name{Daniel Barzilai} \Email{daniel.barzilai@weizmann.ac.il} \\
 \addr Weizmann Institute of Science
 \AND
 \Name{Ohad Shamir} \Email{ohad.shamir@weizmann.ac.il} \\
 \addr Weizmann Institute of Science and University of Toronto
}

\begin{document}

\maketitle

\begin{abstract}%
    Understanding the limitations of gradient methods, and stochastic gradient descent (SGD) in particular, is a central challenge in learning theory. To that end, a commonly used tool is the Statistical Queries (SQ) framework, which studies performance limits of algorithms based on noisy interaction with the data. 
    However, it is known that the formal connection between the SQ framework and SGD is tenuous: Existing results typically rely on adversarial or specially-structured gradient noise that does not reflect the noise in standard SGD, and (as we point out here) can sometimes lead to incorrect predictions. 
    Moreover, many analyses of SGD for challenging problems rely on non-trivial algorithmic modifications, such as restricting the SGD trajectory to the sphere or using very small learning rates. To address these shortcomings, we develop a new, non-SQ framework to study the limitations of standard vanilla SGD, for single-index and multi-index models (namely, when the target function depends on a low-dimensional projection of the inputs). Our results apply to a broad class of settings and architectures, including (potentially deep) neural networks. 
\end{abstract}

\begin{keywords}%
  Multi-Index, Lower Bounds, SGD, Algorithm-specific Complexity
\end{keywords}

\section{Introduction}

Stochastic gradient descent (SGD) methods are the workhorse of modern machine learning algorithms, and are very successful in many practical applications. However, SGD is not universally effective, and 
it has long been recognized that there exist learning problems on which it can fail (even on problems that are tractable with other methods; see for example \cite{shalev2017failures}).  Rigorously understanding the complexity of gradient-based methods such as SGD remains an important open question in learning theory. 

For such algorithm-specific complexity results, a standard and increasingly popular tool is the Statistical Query (SQ) framework \citep{kearns1998efficient}, which studies how certain problems are provably intractable with noise-tolerant algorithms. In that framework, the algorithm can only receive noisy estimates of various quantities associated with the data. This framework is frequently invoked as a proxy for SGD methods, based on the intuition that a stochastic gradient step effectively serves as a noisy estimate of the population gradient. As such, SQ bounds have been used in recent years to suggest computational and statistical complexity of various learning tasks \citep{feldman2017statistical, shamir2018distribution, chen2020learning, song2021cryptographic, abbe2023sgd, diakonikolas2023algorithmic, damian2024computational, joshi2024complexity, damian2025generative, diakonikolas2026algorithms}. Consequently, SQ-based lower bounds are often viewed as a heuristic indicator for the limitations of gradient-based methods.

However, as noted in many of these works, the \emph{formal} connection between SGD and the SQ framework is tenuous, due to differences in the gradient noise model. In the SQ framework, it is assumed that the noise can be adversarial (or at best, stochastic with a particular convenient structure, such as an isotropic Gaussian, see \cite{yang2005new,abbe2022non,abbelearning25,shoshani2025hardness}). In contrast, in actual SGD, this noise is the difference between the stochastic gradient on a particular data sample (or mini-batch) and the population gradient. This zero-mean noise vector is not adversarial, but also has a distribution that is data-dependent, generally anisotropic, and evolves throughout training. It is therefore notoriously difficult to analyze SGD directly. Thus, existing complexity results often apply to stylized variants of SGD, where it is assumed that some additional stochastic or adversarial noise is added to the gradient updates. 

Moreover, although SQ-based lower bounds for SGD are believed to be indicative in many settings, there do exist settings where these bounds do not fully explain the performance of actual SGD. One interesting example is learning parities, namely functions of the form $\bx\mapsto f_{\Ical}(\bx):=\sum_{i\in \Ical}x_i \mod 2$ over the Boolean hypercube. 
Such functions are well-known to be efficiently learnable (via Gaussian elimination), but provably not by any SQ-based algorithm \citep{kearns1998efficient}. This is consistent with strong empirical evidence that parities are hard to learn using standard predictors trained using SGD  (e.g., \citet{shalev2017failures,barak2022hidden,nachum2020symmetry}).
However, more recently, \citet{abbe2020universality,abbe2021power} showed that parities \emph{can} be learned with SGD, at least when used to train a certain (highly synthetic) neural network. 
Thus, the SQ-based explanation for the hardness of learning parities with SGD appears to be incomplete.
Moreover, as we discuss in more detail in Appendix \ref{app: sq}, there exist several other natural learning problems, where the behavior of standard learning procedures is not correctly captured by the SQ framework, due to its adversarial noise assumption being too pessimistic. These include situations such as matrix mean estimation \citep{li2019mean}, tensor PCA problems \citep{dudeja2021statistical}, sparse prediction, and more. This raises the concern that SQ-based predictions of the behavior of SGD may sometimes be inaccurate. 

Beyond these issues, existing analyses for SGD often make additional non-trivial assumptions about the algorithm, such as restricting the parameter vector to a sphere, using a very small learning rate so that SGD is approximated by gradient flow, layer-wise training, etc. (see \citep{bruna2025survey} for a survey).

We also mention the low-degree polynomial methods and sum-of-squares lower frameworks, which have been highly influential in predicting statistical-computational gaps for a range of problems, including planted clique and tensor PCA \citep{hopkins2018statistical, barak2019nearly, wein2025computational}. However, these frameworks generally do not apply to SGD-trained neural networks, and cannot yield rigorous lower bounds in the settings that we consider.

Motivated by all this, we consider the following natural question:
\begin{quote}
    \emph{Can we rigorously understand why standard SGD (without SQ-based or other algorithmic modifications) on natural predictors struggles to solve certain learning problems?}
\end{quote}

In this paper, we consider this question for a well-studied family of high-dimensional settings, where the goal is to learn a single-index or a multi-index target function which depends on a low-dimensional linear subspace of the inputs. This topic has received much attention in recent years, including both upper and lower bounds for gradient-based methods, often invoking SQ machinery \citep{arous2021online, barak2022hidden, mousavineural, abbe2023sgd, bietti2023learning, kou2024matching, dandibenefits, simsek2025learning, ren2025emergence, diakonikolas2025robust, bruna2025survey}. 
In such problems, achieving a small loss requires the predictor parameters to align with a small number of task-relevant directions. 

Our main contribution is a new framework for deriving formal lower bounds for SGD in such settings. Sidestepping SQ-based arguments, we directly show that vanilla SGD over individual samples (without any modifications) can fail to build meaningful correlations with task-relevant directions over a long time horizon, despite the absence of statistical or representational barriers. This failure will, in turn, be used to provide lower bounds on the number of samples required by SGD to achieve a non-trivial loss. In more detail, our contributions are as follows: 

\begin{itemize}
    \item In Section \ref{sec: main}, we present a general lower bound for learning multi-index models with standard SGD. The bound is based on formalizing the intuition that for difficult problems (where the gradients carry a weak signal about the target function), SGD fails to align with any specific task-relevant directions. To show this, we identify a concrete ``gradient condition number'' quantity, which ensures that the martingale processes governing the alignment evolve similarly to those of an isotropic random walk. 
    The bound we obtain applies to a wide range of architectures and problem settings, including neural networks with a linear first layer. We precede this with an informal discussion of our approach, presented in \secref{sec: capturing_hardness}.

    \item In Section \ref{sec:applications}, we provide several applications of our general result, yielding concrete lower bounds for SGD in various settings. We begin in Section \ref{sec: periodic} with periodic target functions  (which can be seen as an extension of parity functions to continuous input distributions), and show that standard SGD on $2$-layer networks cannot learn these with samples/iterations polynomial in the dimension $d$. Previous results of this kind relied on SQ arguments and did not apply to standard SGD. We then consider single-index and multi-index target functions in Section~\ref{sec: multi-index}. We show that for a broad set of predictors, SGD requires $\tilde\Omega(d^{\max(k_\star - 1, 1)})$ samples/iterations, where $k_\star$ is the \emph{information-exponent} of the target function. While such lower bounds were previously shown in SQ-based frameworks or for stylized variants of SGD, here we prove that such bounds apply to standard SGD.
\end{itemize}

Moreover, in Appendix \ref{app: sq}, we further discuss limitations of the SQ framework, presenting several examples where SQ-based arguments provably lead to overly-pessimistic predictions about the behavior of standard SGD. 

Overall, we hope our results help to clarify when and why standard SGD struggles, 
and provide insights that may guide the analysis and design of gradient-based algorithms.

\section{Setting and Preliminaries}\label{sec:prelim}
\paragraph{Predictor class.} 
We study the use of SGD for training predictors that depend on inputs $\bx \in \R^d$ through a linear map. Formally, such predictors take the form
\begin{equation}\label{eq:model}
f_{\theta}(\bx) := h\left(W\bx \,;\, \bar\theta \right),
\end{equation}
where $h$ is a differentiable function\footnote{One may also consider neural networks with activations like ReLU, which have a well-defined sub-gradient.},  $W$ is a learned linear map (a matrix of size $m\times d$ for a width parameter $m$), $\bar \theta$ denotes optional additional parameters, and $\theta=(\bar{\theta},W)$. Crucially, all dependence on the inputs is through $W\bx$. In particular, $f_\theta$ may be a neural network of any architecture with a linear first layer (such as MLPs and transformers). Our general result does not depend explicitly on the specific structure of $h$, only on regularity assumptions about its gradients (formalized later).

\paragraph{Low-dimensional target function.}\label{subsec:target}
In many learning problems, the target function generating the output depends on the input only through a small number of task-relevant directions. 
To capture this, fix an integer $p\ll d$ and an unknown matrix $U\in\R^{p\times d}$ that defines the task-relevant directions.
We assume the target function depends on inputs $\bx$ only through the $p$-dimensional representation $U\bx$:
\begin{equation}\label{eq:target}
f^\star(\bx) = \phi(U\bx),
\end{equation}
for some function $\phi:\R^p\to \R$.
We think of $p$ as a small intrinsic dimension, while $d$ is large. Target functions of this form are often referred to as \emph{multi-index} models, or \emph{single-index} if $p=1$. Some examples include generalized linear models and single neurons (for $p=1$);  Intersection of half-spaces (or signed sparse parities) $f^\star(\bx)=\prod_{i=1}^p \sign(\bu_i^\top \bx)$ for $\{\bu_i\}_{i=1}^p \in \R^d$; and neural networks with a width-$p$ first layer, e.g. $U_k \sigma\left(U_{k-1}\left(\ldots, U_1 \bx\right)\right)$ for weight matrices $U_i$ and depth $k$.

\paragraph{Sample and population losses.}\label{subsec:loss}
Let $\ell:\R\times \R\to \R$ be a differentiable loss function, which measures the discrepancy between a prediction and the target function's output. Given an input $\bx\in\R^d$ and parameters $\theta$, this induces the \emph{sample loss}
\[
\ell(\theta ; \bx) := \ell\left(f_{\theta}(\bx), f^\star(\bx)\right),
\]
where in the above, we slightly abuse the $\ell$ notation to take two different types of arguments (the meaning should be clear from context).
Let $\Dcal$ be a distribution over inputs $\bx\in\R^d$.
The population loss is defined as
\[
\Lcal(\theta) :=
\E_{\bx\sim\Dcal}\big[\ell(\theta;\bx)\big].
\]

\paragraph{Standard SGD.} 
Given a learning rate  parameter $\eta>0$ and initial parameters $\theta_0$, we consider parameters updated by vanilla stochastic gradient descent (SGD) on individual samples:
\begin{equation}\label{eq:sgd}
\theta_t = \theta_{t-1} - \eta \nabla_\theta \ell(\theta_{t-1};\bx_t),
\qquad \text{where }
\bx_t\overset{\scriptstyle{\text{i.i.d.}}}{\sim}\Dcal.
\end{equation}
While SGD updates all parameters of the predictor, our analysis focuses on tracking only the evolution of the first-layer weights $W_t$. Our lower bounds are based on the first layer weights $W_t$ failing to capture the relevant directions in $U$. All other parameters are treated as part of the function $h$, and our theorems are applicable even if they are not updated using SGD, but rather via some other method (as long as the required assumptions hold). 

In this paper, we focus on a batch size of $1$ for simplicity (namely, the update uses the stochastic gradient w.r.t. a single sample $\bx_t$, rather than the average of $B>1$ such stochastic gradients). However, we expect our results to extend naturally to the $B>1$ regime at the cost of notational overhead. In particular, i.i.d. minibatches naturally fit within the martingale concentration framework that we employ in our proofs.
In such a case, one should distinguish between the number of iterations and the sample complexity. While using a batch size $B>1$ may reduce the lower bounds on the number of iterations needed, the bounds on the sample complexity should remain the same.

\paragraph{Additional notation.}
For $d\in\N$, let $[d]:=\{1,\dots,d\}$.
For a matrix $A \in \R^{m \times d}$, we denote by $s_{1}(A), \ldots, s_{\min(m,d)}(A)$ its singular values in decreasing order, so $s_i(A) \geq s_{i+1}(A)$. We use $A^\dagger \in \R^{d \times m}$ to denote the (Moore-Penrose) pseudoinverse of $A$. We let $\Row(A)\subseteq \R^d$ be the row span of the matrix $A$. We denote by $P_A\in\R^{d\times d}$ the orthogonal projector onto $\Row(A)$, or equivalently, $P_A = A^\dagger A$. We use $\norm{\cdot}_\op$ for the operator norm, $\norm{\cdot}_F$ for the Frobenius norm, and unless specified otherwise, $\langle \cdot, \cdot \rangle$ and $\norm{\cdot}$ denote the standard (Frobenius) inner product and norm. 
Given filtrations $\left(\Fcal_t\right)_{t=1}^\infty$, we write $\Pr_t(\cdot):=\Pr(\cdot\mid \Fcal_{t-1})$ and for any random variable $Z$, $\E_t[Z]:=\E[Z\mid \Fcal_{t-1}]$. We use standard big-O notation, with $\bigo(\cdot)$, $\Theta(\cdot)$ and $\Omega(\cdot)$ hiding constants, and $\tilde{\bigo}(\cdot)$,$\tilde{\Theta}(\cdot)$ $\tilde{\Omega}(\cdot)$ additionally hiding logarithmic factors. 

\section{Capturing Hardness for SGD}\label{sec: capturing_hardness}

In this section, we informally describe the main ideas used to derive our results. We begin by noting that there are 
several ways in which the population loss landscape may cause SGD to struggle. One example is a 
landscape with many poor local minima that are difficult to avoid. However, empirical and theoretical evidence suggests that either such landscapes do not arise naturally in many high-dimensional cases, or that the local minima are not too bad \citep{choromanska2015loss, li2018visualizing,safran2018spurious}.
Here, we focus on a different scenario, in which SGD struggles because the population-loss landscape is \emph{too flat}. In this scenario, optimization is hindered not by spurious minima of the population objective, but by stochastic gradients whose signal is too weak relative to their noise.
More specifically, consider a case where the zero-mean \emph{SGD noise}, defined as $\nabla\ell(\theta_{t-1} ; \bx_t) - \nabla\Lcal\left(\theta_{t-1}\right)$, is large relative to the population gradients $\nabla\Lcal\left(\theta_{t-1}\right)$.
In this case, the parameter updates are dominated by the SGD noise, and one may imagine that SGD behaves approximately as a zero-mean random walk in parameter space.
The main difficulty in formalizing such intuition lies in the structure of the SGD noise, which is data-dependent, anisotropic, and whose distribution evolves as the parameters change.  

\paragraph{Alignment bounds the population gradient.}
A key feature of our setting, and the central source of optimization complexity, is that the target depends
on $\bx$ only through a low-dimensional \emph{task subspace} given by the row-span of $U$, which we denote by $\Row(U)\subseteq \R^d$. In contrast, the predictor we train accesses the input only through the
\emph{learned subspace} $\Row(W)\subseteq \R^d$.
Intuitively, if $\Row(W)$ has little overlap with $\Row(U)$, then for high-dimensional isotropic input distributions, $W\bx$ is nearly uncorrelated with $U\bx$.
As we will later show, in this regime, both predictions and gradients provide only a weak signal about the target, implying that SGD noise dominates the dynamics for many iterations. We capture this phenomenon through the \emph{alignment}
\[
    \rho(W,U) ~:=~ \|P_W P_U\|_{\op}\in[0,1]~,
\]
where $P_W$ and $P_U$ denote the orthogonal projection matrices onto $\Row(W)$ and $\Row(U)$, respectively. In particular, $\rho(W,U)$ is the cosine of the smallest principal angle between the two subspaces, and is negligible unless $W$ is well aligned with $U$. When $p=1$, the region where the alignment is larger than some threshold reduces to a (double) cone centered at $\pm\bu$. Thus, SGD fails as long as the SGD trajectory does not enter this cone. We note that while alignment-based analyses have appeared in many prior works (e.g. \citet{arous2021online, dandibenefits, ren2024learning}), they are typically tailored to spherical SGD, where the scale of the weights plays no role. In such a case, the SGD dynamics depend entirely on the alignment, and the analysis reduces to analyzing the one-dimensional dynamics of the alignment. For standard SGD, the situation is more complicated. 

\paragraph{Connection between SGD and isotropic random walks.} To explain the intuition of our approach, consider the case where the width parameter $m$ equals $1$. In that case, the predictor matrix $W$ reduces to a vector $\bw$, and at iteration $t$, the gradient w.r.t. the current iterate $\bw_{t-1}$ is given by
\begin{align}\label{eq:grad_decomp}
\nabla_{\bw_{t-1}} \ell\left(\theta_{t-1} ; \bx_t\right) = a_t \bx_t \qquad \text{where} \qquad a_t := \frac{\partial}{\partial (\langle \bw_{t-1} \bx_t\rangle)} \ell\left(\theta_{t-1} ; \bx_t\right) \in \R~.    
\end{align}
Since the weights are updated by SGD (\eqref{eq:sgd}), at time $T$ they can be written as
\[
\bw_{T} = \bw_{0} + \sum_{t=1}^T (\bw_{t} - \bw_{t-1}) = \bw_0 - \eta \sum_{t=1}^T a_t \bx_t~.
\]

Now, suppose for simplicity that the inputs $\bx_t$ are standard Gaussians in $\reals^d$, and that the scalar coefficients $a_t$ are fixed constants. In this idealized scenario, the iterates $\bw_T$ would form a zero-mean Gaussian random walk, the marginal distribution of $\bw_T$ would be isotropic, and as a result, the alignment $\rho(\bw_T^\top \,,\, U)$ would concentrate around $\sqrt{p/d}$. Since we assume $p\ll d$, it follows that with high probability, there would be no significant alignment and the algorithm would fail. 

However, with actual SGD, the situation is much more complicated (even for Gaussian inputs). This is because the random variables $a_t$ are generally non-constant, and very much depend on $\bx_t$ as well as on the previous inputs. 
Our main technical contribution is to prove that, despite these dependencies, the idealized scenario described above still roughly holds in many cases.
In particular, if $\E_t [a_t\bx_t] = \E_t[\nabla_{\bw_t} \Lcal(\theta_t)] \approx 0$ (where $\E_t$ denotes the expectation conditioned on past events) and if $a_t, \bx_t$ are ``sufficiently well-behaved'' then we show that for any $\bu \in\R^d$, $\frac{\abs{\bw_T ^\top \bu}}{\norm{\bw_T}} \lesssim \frac{1}{\sqrt{d}}$. In this manner, if the alignment is small at initialization, then it remains small for many iterations, and the loss remains nearly trivial.

\paragraph{Gradient condition number} 
To quantify what ``well-behaved'' means, it turns out that it suffices for the coefficients $a_t$ in \eqref{eq:grad_decomp} to not be dominated by rare heavy-tailed events. 
Under this condition, the martingale processes governing $\bw_T^\top \bu$ and $\norm{\bw_T}$ resemble those of an isotropic random walk as described earlier, so that the alignment remains small. In our  width-$1$ setting, the norm of the weights grows roughly as $\norm{\bw_t} \gtrsim \sqrt{\E[a_t^2] \cdot d}$, whereas high-probability upper bounds on $\abs{\bw_t^\top \bu}$ are on the order of $\sup |a_t|$ (or more generally, the sub-Gaussian norm). Thus, when the ratio $\sup a_t^2 / \E[a_t^2]$ is of constant order, $\bw_t$  cannot align strongly with $\bu$. 

We identify a quantity $\kappaT$ (see Definition \ref{def: subgauss_mds}), which we call the gradient condition number, that bounds the width-$m$ analogue of this ratio.
We then show that the intuition above holds as long as $\kappaT$ does not scale strongly with the input dimension $d$, which is generally expected, since $\kappaT$ depends on an $m$-dimensional projection of the inputs rather than directly on the input dimension. 
In practice, many tricks such as batch normalization, skip-connections, and adaptive learning rates are used to keep the stochastic gradients well-behaved. 
This allows us to capture a wide range of possible architectures and predictors, where our dependence on the loss function $\ell$ and the architecture $h$ is only through $\kappaT$.

\section{General Lower Bound}\label{sec: main}
We start by stating the assumptions under which we establish a general lower bound on the rate at which SGD can discover task-relevant directions.
Our goal here is to identify a small number of structural conditions that isolate the source of optimization complexity described in Section~\ref{sec: capturing_hardness}.
In particular, the assumptions below are designed to
(i) prevent rare pathological events from dominating the dynamics, and
(ii) control the scale of the stochastic gradients, when the input dimension $d$ is large, and in a manner that reflects what is common in practice. We begin with our assumptions on the input distribution:\\

\begin{assumption}[Sub-Gaussian inputs with norm concentration]\label{ass: inputs_all}
    \begin{enumerate}[label={(\Alph*)}, ref={1.\Alph*}]       
        \item \label{ass: subgauss_inputs} The inputs $\bx$ are mean-zero, and there exists some $K_1 > 0$ such that for any unit vector $\bv$, 
        \[
            \forall \gamma>0,~\Pr\left(|\bv^\top\bx|\geq \gamma \right)~\leq~ 2\exp\left(-\frac{\gamma^2}{K_1^2}\right)~.
        \]
        
        \item \label{ass: norm_conc}
        There exists some $\alpha_2 > 0, K_2 > 0$ such that $\E\left[\norm{\bx}^2\right] = \alpha_2 d$, and 
        \begin{align*}
            \forall \gamma>0,~\Pr\left(\abs{\norm{\bx}^2 - \alpha_2d} \geq \gamma\right) \leq 2\exp\left(-\min\left(\frac{\gamma^2}{K_2^4 d}, \frac{\gamma}{K_2^2}\right)\right).
        \end{align*}
    \end{enumerate}
\end{assumption}

Both conditions are satisfied, for example, for well-conditioned Gaussians and uniform distributions on the cube $\{-1,1\}^d$ or the (scaled) sphere $\sqrt{d} \cdot \Sphere^{d-1}$. More generally, assumption~\ref{ass: subgauss_inputs} is standard in high-dimensional settings and ensures that all one-dimensional projections of the input are of constant order with high probability. Assumption~\ref{ass: norm_conc} ensures that the overall scale of the input remains stable across samples, so that SGD updates are not dominated by rare large-norm inputs. In many common settings, this condition follows directly from Assumption~\ref{ass: subgauss_inputs} via the Hanson-Wright inequality \citep{vershynin2025high}.
We state it separately to emphasize its role in controlling the magnitude of stochastic gradients uniformly over time.
Finally, we note that the precise scaling of the inputs is not fundamental. Indeed, if we consider instead rescaled inputs of the form $c \cdot \bx$ for some $c > 0$, then rescaling the learning rate to $\eta / c$ for the weights $W_t$ leads to the same SGD dynamics, and our guarantees can be translated accordingly.

Our next assumption concerns the initialization of the first layer weights. In practice, the weights are typically sampled i.i.d. from a Gaussian or a uniform distribution with variance roughly $1/d$. This ensures that $f_{\theta_0}$ and its gradients are well-behaved at initialization, even when there are many parameters. The following assumption indeed holds for such standard initializations.
\begin{assumption}[Standard initialization]\label{ass: init}
    There exists some $K_3 > 0$ such that the entries of $W_0$ are i.i.d. random variables with zero mean, variance $1/d$, and
    \begin{align*}
        \Pr\left(|(W_0)_{ij}|\geq \gamma \right)~\leq~ 2\exp\left(-\frac{\gamma^2}{dK_3^2}\right)~.
    \end{align*}
\end{assumption}
Such initializations do not encode any prior information about the task subspace $U$.
Consequently, when $d \gg m,p$, the alignment $\|P_{W_0}P_U\|_{\op}$ is small at initialization with high probability. Furthermore, we note that the exact scaling of the coordinates is not fundamental for our results. 

The final assumption controls the maximal scale of gradients with respect to $W\bx$:
\begin{assumption}[Bounded gradients]\label{ass: bounded}
    There exists some $G > 0$ such that for any $\theta$ and for (almost-surely) any $\bx$,   $\norm{\nabla_{W\bx} \ell(\theta ; \bx)}_2 \leq G$. 
\end{assumption}
We emphasize that the gradient here is with respect to $W\bx$, not the parameters $\theta$. Assumption \ref{ass: bounded} is satisfied in many common settings. 
For example, consider the case of Lipschitz losses and neural network predictors with activations that have bounded derivatives (such as ReLU or tanh). Then Assumption \ref{ass: bounded} holds as long as the spectral norms of the weights in the deeper layer (i.e. $\bar{\theta}$) remain bounded, which often occurs in practice.

The exact value of $G$ is not important for our results, as it can be absorbed into the learning rate.
What matters is that typical gradient magnitudes are not negligible compared to this bound.
If $\E_\bx[\|\nabla_{W\bx}\ell(\theta ; \bx)\|_2]\ll G$, then gradients may effectively be heavy-tailed, meaning that occasional rare events can have a significant impact on the SGD trajectory. 
As discussed in the previous section, we therefore introduce the following quantity that will be used in our bounds: 
\begin{definition}
Given any $T\in \N$, the \emph{gradient condition number} $\kappaT$ is defined as
\begin{align*}
    \kappaT := \frac{G^2}{\inf_{\bv \in \Sphere^{m-1}}
    \min_{t \leq T}
    \E_t\left[
        \big( \bv^\top \nabla_{W_{t-1}\bx_t}\ell(\theta_{t-1};\bx_t) \big)^2
    \right]} ~.
\end{align*}
\end{definition}
As before, we emphasize that the gradient is with respect to $W\bx$, and is an $m$-dimensional vector. As such, $\kappaT$ may depend on the width parameter $m$, but importantly, it should not depend on the input dimension $d$. As we will see later, one can indeed show this occurs in many cases.
Representing the general result in terms of $\kappaT$ allows us to avoid specific assumptions about the predictor and the target function. 

As discussed previously, our bounds will be based on the fact that as long as the task-alignment is small, then so is the population gradient.
We capture this dependence through an increasing function $\psi:[0,1]\to[0,\infty)$ such that
\[
    \|\nabla_W \Lcal(\theta)\|_F \leq \psi\left(\|P_W P_U\|_{\op}\right)
    \qquad\text{for all }\theta.
\]
In many structured learning problems, $\psi(r)$ decays rapidly as $r\to 0$, reflecting the fact that
when the learned representation is nearly orthogonal to the task subspace, the population gradient is very small and
provides only a weak signal.
Concrete examples of such functions $\psi$ will be provided in the next section. 
We now state the main theorem, with the full proof appearing in appendix \ref{app:thmmainproof}.

\begin{restatable}{theorem}{main}\label{thm: main}
    Under Assumptions \ref{ass: inputs_all}-\ref{ass: bounded}, let $\delta > 0$, $\bar \kappa\geq 1$, and let $\psi:[0, 1]\to [0,\infty)$ be an increasing function  such that $\norm{\nabla_{W} \Lcal(\theta)}_F \leq \psi\left(\norm{P_W P_U}_\op\right)$ for all $\theta$. There exist constants $C, c, c' > 0$ (that may depend on $K_1, K_2, K_3, \alpha_2, G$), such that if $d \geq c \bar \kappa^2 m^2 \log\left(\frac{Tdp}{\delta}\right)^2$, $\eta \leq \frac{c'}{\bar \kappa^2\sqrt{md \log\left(Tdp/\delta\right)}}$ and 
    \[
        T \leq \frac{1}{\psi\left(C \sqrt{\frac{\bar \kappa^{} mp \log\left(Tdp/\delta\right)}{d}} \right)^2}~,
    \]
    then conditioned on $\kappaT \leq \bar \kappa$ it holds with probability at least $1-\delta$ that
    \[
        \forall \,t\, \in [T] ~, \qquad 
        \norm{P_{W_t} P_U}_\op \leq C \sqrt{\frac{\bar \kappa^{} mp \log\left(Tdp/\delta
        \right)}{d}}~.
    \]
\end{restatable}

The theorem shows that if the population gradient can be bounded using the alignment,
SGD fails to significantly increase alignment with the task subspace across many iterations.
In particular, until the alignment exceeds the scale $\Theta(\sqrt{\kappaT mp/d})$, the learned representation
remains effectively uncorrelated with the task directions.
The appearance of the scale $1/\sqrt{d}$ reflects the typical alignment between random subspaces in
high dimensions, while the square arises from the accumulation of stochastic noise over time. The requirement that $m\leq d$ can in some cases be circumvented. For example, in the subsequent subsections, for two-layer networks trained with the correlation loss we will allow $m$ to be polynomially large in $d$. This is done by applying \thmref{thm: main} to each neuron individually, and using a union bound. Lastly, we note that the theorem is stated using a value $\bar \kappa \geq 1$ which serves as an inferred upper bound on $\kappaT$. 

The proof formalizes the intuition presented earlier in Section \ref{sec: capturing_hardness}. 
In the following sections, we instantiate the function $\psi$ for a variety of learning problems,
yielding explicit lower bounds on the samples/iterations required for SGD to begin discovering task-relevant
features.

\section{Applications}\label{sec:applications}

In this section, we instantiate our general theorem to concrete settings, leading to more explicit bounds. 
In line with many recent works providing upper bounds for multi-index models, we focus here on Gaussian inputs and the correlation loss, which allows for convenient formulas of the population loss and its gradient (just for example, \citet{damian2022neural, bietti2023learning, ren2024learning}). Nevertheless, the results are applicable to other distributions as well. 

\subsection{Hermite Polynomials}

To study the efficiency of SGD with Gaussian inputs in various settings, it is often convenient to study the \emph{Hermite expansion} of $f^\star$. We provide here a brief exposition on Hermite polynomials, and defer the reader to Appendix \ref{sec: hermite} for more details. To make our results applicable beyond single-index models and two-layer networks, we will need to work with the high-dimensional version of Hermite polynomials. 
Let $\mu(\bx)$ be the PDF of the standard Gaussian distribution in $\reals^d$, and $L_2$ (or $L_2(\mu)$) be the space of functions with $\norm{f}_{L_2}:=\E_{\bx\sim \mu}[\norm{f(\bx)}^2]^{1/2} < \infty$. We define the normalized (probabilist's) Hermite polynomials by $H_k(\bx) := (-1)^k \frac{\nabla^k \mu(\bx)}{\sqrt{k!} \mu(\bx)}$. The first few polynomials are given by $H_0(\bx) = 1, H_1(\bx) = \bx$, and $H_2(\bx) = \frac{\bx\bx^\top - I}{\sqrt{2}}$. Importantly, the Hermite polynomials are known to be an orthonormal basis for $L_2(\mu)$, implying that any function $g(\bx):\R^d\to \R$ can be represented (with convergence in $L_2$) as
\begin{align}\label{eq: hermite_main}
    g(\bx) = \sum_{k=0}^\infty \langle H_k(\bx), C_k \rangle \qquad \text{for} \qquad C_k:=\E_{\bx \sim \mu(\bx)}[g(\bx)H_k(\bx)].
\end{align}

\subsection{Failure to Learn Periodic Functions in Polynomial Time}\label{sec: periodic}

The simplest setting we will consider involves a special type of single-index models, where the link function is periodic. Concretely, we will consider a target function given by $f^\star(\bx) := \sin\left(\bu^\top \bx\right)$ where $\bu \in \R^d$ is an unknown vector with possibly large norm. Such target functions serve as a prototypical example of functions difficult to learn under the SQ framework. In particular, paralleling the situation with parity functions, they are known to be exponentially hard to learn with gradient methods under (small) \emph{adversarial} gradient noise (when $\norm{\bu}^2$ is polynomially large in $d$), even though they can be efficiently learned with non-SGD methods \citep{shamir2018distribution, song2021cryptographic,damian2024computational}. We also note that such functions can be seen as a generalization of parities to inputs in $\R^d$, since a parity function on a subset $\Ical\subseteq d$ coincides with the function $\sin(\gamma \sum_{i\in \Ical} x_i)$ on the Boolean hypercube, for an appropriate $\gamma>0$.

As to the class of predictors, we will consider here
two-layer networks (more general predictors will be considered later). Specifically, for any $t$ let $\bw_{t,1},\ldots, \bw_{t,m}$ be the rows of $W_t$, and suppose that $f_{\theta_t}(\bx)=\sum_{i=1}^m\sigma(\langle \bw_{t, i} \,, \,\bx \rangle)$.
The activation function $\sigma$ is assumed to be well-behaved in the following sense:
\begin{assumption}\label{ass: non_lin}
    $\sigma:\R\to \R$ is differentiable (or sub-differentiable) with (1) bounded derivative $\sigma'(\bx) \leq G_1$ for $G_1>0$, a.s any $x$, and (2) $\inf_{s>0}\E_{x\sim \Ncal(0, s^2)}[\sigma'\left(x\right)^2] \geq G_1^2/K_\sigma$ for $K_\sigma > 0$.
\end{assumption}
As we discuss in \appref{app: non_lin}, standard activation functions such as ReLU, Leaky-ReLU, GeLU, softplus, and the (logistic) sigmoid all satisfy Assumption \ref{ass: non_lin}. 

Despite SQ-style results, to the best of our knowledge, it remains unknown whether periodic functions are indeed hard to train with standard SGD. In the following theorem, we show unconditionally that it is impossible to learn such functions using SGD with polynomially (in $d$) many iterations, on a predictor class of two-layer networks as described above. 

\begin{restatable}{theorem}{periodic}\label{thm: periodic}
Let $W_0 \sim \Ncal\left(0, \frac{1}{d}I_{md}\right)$, $\bx\sim \Ncal(0, I_d)$, $\eta \leq 1/d$, $\ell$ be the correlation loss and $f_{\theta_t}(\bx)=\sum_{i=1}^m\sigma(\langle \bw_{t, i} \,, \,\bx \rangle)$ for $\sigma$ satisfying Assumption \ref{ass: non_lin}. Let $f^\star(\bx):= \sin(\bu^\top \bx)$ for some $\bu \in \R^d$ with $\norm{\bu}=\sqrt{d}$. Then there exist $d_0, C >0$ which depend only on $G_1, K_\sigma$, such that for any $d \geq d_0$ it holds with probability at least $1-2m\exp\left(-d^{1/3}\right)$ that
    \begin{align*}
        \forall \,t\, \leq \exp\left(d^{1/3}\right) ~,\qquad 
        \abs{\Lcal(\theta_t) - \Lcal(\theta_0)}
        \leq \frac{\sum_{i=1}^m \norm{\sigma(\langle \bw_{t, i} \,, \,\bx \rangle)}_{\fsq} + \norm{\sigma(\langle \bw_{0, i} \,, \,\bx \rangle)}_{\fsq}}{\exp\left(C d\right)}.
    \end{align*}
\end{restatable}

The proof can be found in \appref{app: periodic}, and follows from \thmref{thm: main}. For the proof intuition, consider for simplicity the case of $m=1$, so that $f_\theta(\bx) = \sigma(\bw^\top \bx)$. We bound $\nabla \Lcal$ as
\[
    \norm{\nabla_\bw \Lcal(\bw_t)} \leq \psi\left(\norm{P_{\bw^\top_t} P_{\bu^\top}}_\op\right) \qquad \text{where} \qquad \psi(\rho):=(1+\sqrt{d})\exp\left(-\frac{d}{2}\left(1 - \rho^2\right)\right).
\]
Moreover, we show that in the above setting, $\kappaT \lesssim K_\sigma$.
\thmref{thm: main} then implies that the alignment remains on the order of $1/\sqrt{d}$ for at least $\exp(d^{1/3})$ steps. This, in turn, implies the loss remains exponentially close to the loss at initialization. 

The choice of exponent $1/3$ is not fundamental to the result; the theorem can be proven with a lower bound of $\exp(d^\alpha)$ for any $\alpha < 1/2$. Lastly, we note that the activation function can be taken to be $\sin$, which satisfies the assumption of the theorem, and the predictor may be $\sum_{i=1}^m \sin (\bw_i^\top \bx)$. In this case, each neuron matches the form of the target function, and learning is impossible in polynomial time despite overparameterization (as long as the width is at most polynomial in $d$). 

\subsection{Information Exponent Based Bounds}\label{sec: multi-index}
We now turn to a different source of complexity, based on a quantity known as the \emph{information exponent} \citep{arous2021online, bietti2022learning} defined as the first $k\geq 1$ such that the Hermite coefficients $C_k$ of the target function $f^\star$ do not vanish:
\begin{definition}[Information exponent]
    Let $f^\star(\bx)\in L_2(\mu)$, the \emph{information exponent} of $f^\star$ is 
    \[
        k_\star := \min\{k \geq 1 \mid \E_{\bx\sim \mu(\bx)}[f^\star(\bx)H_k(\bx)] \neq \zero\}.
    \]
\end{definition}

The information exponent quantifies the lowest-order statistical interaction between the target and the input distribution. When $k_\star > 1$, all correlations between $f^\star(\bx)$ and polynomials of degree strictly smaller than $k_\star$ vanish. It is by now well-known that the information exponent governs lower-bounds either for single-index models trained with spherical SGD, or in certain SQ-based frameworks (CSQ in particular, see \citet{arous2021online, damian2022neural, abbe2023sgd}). Here, we show that a $\tilde\Omega_d(d^{\max(k_\star -1, 1)})$ holds for standard SGD\footnote{With ``pre-processing'' labels or other algorithmic modifications such as batch reuse, one can obtain better bounds \citep{damian2024smoothing, arnaboldi2024repetita, lee2024neural, damian2025generative}. However, as we focus here on standard SGD with the correlation loss, this is a different setting. For bounds that are robust to such modifications, the generative exponent or related quantities are more suitable \citep{damian2025generative, diakonikolas2026algorithms}}. For most functions, $k_\star \leq 2$, but there are nevertheless many cases where $k_\star > 2$, and SGD provably struggles when the dimension is large. For example, if $f^\star(\bx) = \langle \bu, \bx\rangle^2 \exp(-\langle \bu, \bx\rangle^2)$ for some nonzero $\bu\in\R^d$ then $k_\star = 4$ \citep{damian2024computational}.

We note that this setting technically differs from that of periodic functions as studied in the previous section. For periodic functions, the information exponent is $k_\star = 1$, and the hardness comes from the fact that the gradients are exponentially small in $d$ when the target vector  $\bu$ is of norm $\geq \sqrt{d}$. Here, in line with the existing literature on the information exponent, we will treat the magnitude of $\bu$ (or more generally $U$ in the target function) as fixed independent of $d$, in which case the complexity does depend on the information exponent.

\subsubsection{Single Index Models}\label{sec: single-index}
We first consider here learning a single index model of the form $f^\star(\bx) = \phi(\bu^\top \bx)$ (for $\bu\in\R^d$) using two-layer networks $f_{\theta_t}(\bx)=\sum_{i=1}^m\sigma(\langle \bw_{t, i} \,, \,\bx \rangle)$. In the subsequent Section \ref{sec: mi_deep}, we will consider the general multi-index model and more general predictors, such as deep networks.

Notably, \citet{arous2021online} showed that for learning single index models with a small learning rate and spherical SGD, the information exponent characterizes the rate of convergence. However, their results require a small learning rate $\eta \lesssim d^{-\max(k_\star / 2, 1)}$, so that gradient descent behaves approximately like gradient flow (its continuous-time counterpart). It is well-known that the dynamics of SGD can differ significantly from those of gradient flow in many practical settings \citep{cohen2021gradient, andriushchenko2023sgd}. 
Furthermore, as demonstrated in \citet{mousavi2023gradient}, their simplification of spherical SGD can fail to capture the learning dynamics of standard SGD in some settings.
Recently, \citet{braun2025learning} were able to extend single-index \emph{upper} bounds to standard SGD, highlighting that moving beyond spherical SGD is possible but entails highly nontrivial complications. The following theorem, proved in \appref{app: si_inf_exp}, provides the desired \emph{lower} bound for standard SGD.

\begin{restatable}{theorem}{singleidx}\label{thm: si}
    Let $\delta, \epsilon \in (0,1)$, $W_0 \sim \Ncal\left(0, \frac{1}{d}I_{md}\right)$, $\bx\sim \Ncal(0, I_d)$, and $\ell$ be the correlation loss. 
    Let $f_{\theta_t}(\bx)=\sum_{i=1}^m\sigma(\langle \bw_{t, i} \,, \,\bx \rangle)$ for $\sigma$ satisfying Assumption \ref{ass: non_lin}. 
    Let $f^\star(\bx):= \phi(\bu^\top \bx)$ for $\bu \in \R^d$, where $\norm{f^\star}_{L_\infty} \leq G_2$, $\phi' \in L_2(\mu)$ and $f^\star$ having information exponent $k_\star$. 
    There exist $d_0=\tilde{\Theta}\left(k_\star^2 \epsilon^{-2/k_\star}\right)$, $\eta_0 = \tilde{\Theta}\left((dk_\star)^{-1/2}\right)$ and 
    \[
    T = \left(\frac{d}{\tilde \Theta\left(k_\star\right)}\right)^{k_\star - 1} + \Theta\left(d\epsilon^2\right),
    \]
    such that for any $d\geq d_0 $ and $\eta \leq \eta_0$ it holds with probability at least $1-\delta$ that
    \begin{align*}
        \forall \,t \, \leq T
        ~,\qquad 
        \abs{\Lcal(\theta_t) - \Lcal(\theta_0)} 
        \leq \norm{f^\star}_{\fsq}\left(\sum_{i=1}^m \norm{\sigma(\langle \bw_{t, i} \,, \,\bx \rangle)}_{\fsq} + \norm{\sigma(\langle \bw_{0, i} \,, \,\bx \rangle)}_{\fsq}\right) \cdot\epsilon.
    \end{align*}
    Here, $\Theta$ hides constants depending on $G_1, G_2$, $K_\sigma$, $\norm{f^\star}_\fsq$ and $\norm{\nabla f^\star}_{L_2}$, and $\tilde \Theta$ additionally hides polylogarithmic dependencies on the parameters (including $1/\delta$ and $m$).
\end{restatable}
In words, for vanilla SGD, the sample complexity is governed by the information exponent, where roughly $\tilde\Omega_d(d^{\max(k_\star -1, 1)})$ iterations are needed for SGD to converge. This matches the bounds of \cite{arous2021online} for spherical SGD and small learning rates. The proof follows from the more general \thmref{thm: mi_general} that will be presented in the next subsection. Compared with more general architectures, here we leverage the fact that the dynamics of each neuron in the predictor are uncorrelated, and that Assumption \ref{ass: non_lin} causes $\kappaT$ to remain bounded.

\subsubsection{Multi-Index Models and Deep Network Predictors}\label{sec: mi_deep}
In this section, we consider more generally learning target functions of the form $f^\star(\bx)=\phi(U\bx)$ for $U\in \reals^{p\times d}$, $p>1$, using general predictors of the form $f_\theta(\bx) = h(W\bx \,;\, \bar\theta)$ (including potentially deep neural networks). The formal theorem (proved in \appref{app: inf_exp}) is the following:

\begin{restatable}{theorem}{migeneral}\label{thm: mi_general}
     Let $\delta, \epsilon \in (0,1)$, $\bar \kappa \geq 1$, $W_0 \sim \Ncal\left(0, \frac{1}{d}I_{md}\right)$, $\bx\sim \Ncal(0, I_d)$, and $\ell$ be the correlation loss.
    Let $f_\theta(\bx):=h(W\bx \,;\, \bar \theta)$ be a predictor satisfying $\norm{\nabla_{W\bx} h}_{L_\infty}\leq G_1$.
    Let $f^\star(\bx):=\phi(U\bx)$ for $U\in\R^{p\times d}$, where $\norm{f^\star}_{L_\infty} \leq G_2$, $\nabla \phi \in L_2(\mu)$, and $f^\star$ has information exponent $k_\star$. 
    There exist
    $d_0=\tilde{\Theta}\left(\bar \kappa^2 m^2 p^2 k_\star^2 \,\epsilon^{-2/k_\star}\right)$, $\eta_0 = \tilde{\Theta}\left(\bar \kappa^{-2}(md k_\star )^{-1/2}\right)$ such that for any $d\geq d_0$, $\eta \leq \eta_0$ and
    \[
    T = \left(\frac{d}{\tilde \Theta\left(\bar \kappa mpk_\star\right)}\right)^{k_\star - 1} + \Theta\left(\frac{d\epsilon^2}{p}\right),
    \]
    conditioned on $\kappaT \leq \bar \kappa$ it holds with probability at least $1-\delta$ that:
    \begin{align*}
        \forall \,t\, \leq T
        ~,\qquad 
        \abs{\Lcal(\theta_t) - \Lcal(\theta_0)} 
        \leq \norm{f^\star}_{\fsq}\left(\norm{f_{\theta_t}}_{\fsq} + \norm{f_{\theta_0}}_{\fsq}\right) \cdot \epsilon.
    \end{align*}
    Here, $\Theta$ hides constants depending on $G_1, G_2$ and $\norm{\nabla f^\star}_{L_2}$, and $\tilde \Theta$ additionally hides polylogarithmic dependencies on the parameters (including $1/\delta$).
\end{restatable}

The theorem shows that standard SGD generally requires $\tilde\Omega_d(d^{\max(k_\star -1, 1)})$ iterations to converge. For example, consider the target function $f^\star(\bx) := \prod_{i=1}^p \varphi (\bu_i^\top \bx)$, where $\varphi$ is some smooth, bounded, zero-mean function, and $\bu_1,\ldots, \bu_p$ are unknown orthonormal vectors. As we show in \appref{app: prod}, in such a case $k_\star = p$, and thus SGD requires $\tilde\Omega_d(d^{\max(p - 1, 1)})$ iterations to converge.

Compared with the previous results, the price we pay for this added generality is twofold. 
First, we require $d \gg m$, as opposed to the $d \gg \polylog(m)$ requirement of two-layer networks, reflecting that in general deep networks, neurons can interact in highly non-trivial ways.
Second, we can no longer unconditionally bound the gradient condition number $\kappaT$, as we did for two-layer networks. Nevertheless, as discussed in Section \ref{sec: capturing_hardness}, it is reasonable to assume that this parameter (which depends on the $m$-dimensional gradients w.r.t. $W\bx$) does not scale with $d$, which is enough for our results. We also note that if the predictor is a two-layer network, $\kappaT$ can be explicitly bounded even for general multi-index target functions. However, we omit this particular result, as many interesting multi-index targets cannot be captured by small-width two-layer predictors. 

To prove the theorem, we first show in \lemref{lem: mi_bounds_complete} that 
\[
    \norm{\nabla_{W} \Lcal(\theta)}_F \leq \psi\left(\norm{P_W P_U}_\op\right)
    \qquad \text{where} \qquad
    \psi(r) := G_1 \norm{\nabla f^\star(\bx)}_{\fsq}\left(\sqrt{m}r + 1\right) r^{k_\star - 1}~.
\]
Intuitively, the SGD noise is larger than the ``signal'' given by $\norm{\nabla_\bw \Lcal(\theta_t)}$ by a factor depending on $k_\star$.
If we ignore all factors other than the dimension, \thmref{thm: main} suggest that if $T \approx d^{k_\star - 1}$ then $\norm{P_{W_t} P_U}_\op$ remains on the order of $1/\sqrt{d}$ for all $t\in[T]$. This argument is, of course, made precise in the proof. The small alignment then bounds the loss as $\Lcal(\theta) \lesssim \norm{f^\star}_{\fsq}\norm{f_{\theta_t}}_{\fsq}\norm{P_W P_U}_{\op}^{k_\star}$. Lastly, the above bound is vacuous when $k_\star = 1$. To handle the case of $k_\star = 1$, we use a small variant of \thmref{thm: main}, which shows that $\Omega(\epsilon^2 d)$ samples are needed in general. 

\section{Discussion and Some Open Questions}
In this paper, we developed a framework for studying the complexity of standard vanilla SGD when learning high-dimensional single-index and multi-index models. Unlike SQ-based results, which rely on noise assumptions that may not hold in practice, our analysis directly addresses the data-dependent, anisotropic noise of SGD, and contributes to the broader understanding of computational-statistical gaps in high-dimensional learning.  

A natural direction for future work is to extend these techniques beyond vanilla SGD, and possibly, identifying sources of hardness applicable to all gradient-based learning methods. Another interesting direction is to explicitly bound the gradient condition number $\kappaT$ for more general architectures, and to understand whether the gradients tend to be well-behaved (as captured by $\kappaT$) for ``any'' standard predictors. As a sanity check, we note that this does not seem to occur for the synthetic neural network construction of \citet{abbe2020universality, abbe2021power}, discussed in the introduction, which learns parities via SGD\footnote{Their construction involves data-memorization mechanisms, where gradients of certain neurons vanish at some time point, resulting in quantities such as $\kappaT$ diverging to infinity.}. However, understanding whether $\kappaT$ is large only for ``pathological'' constructions remains an open question. Finally, it would be interesting to extend this framework to other classes of learning problems beyond multi-index models.


\acks{
Research was partially supported by PBC-VATAT via the Weizmann Data Science Research Center.
}

\newpage

\bibliography{bib}

\newpage
\appendix


\tableofcontents
\newpage
\allowdisplaybreaks

\section{Limitations of the Statistical Queries Framework}\label{app: sq}
In this appendix, we discuss in more depth some limitations of the SQ framework for analyzing gradient-based learning methods, and SGD in particular. 

As mentioned in the introduction, the Statistical Queries or SQ framework \citep{kearns1998efficient} is a well-established framework for studying noise-tolerant learning algorithms. In this model, rather than directly receiving data samples, it is assumed that the algorithm can only receive (approximate) expectations of various quantities associated with the data. In a nutshell, given an underlying data distribution $\Dcal$, the algorithm can iteratively choose a query function $\phi_t$ mapping to $[0,1]$, and receive a value $v_t$ such that $|\E_{\bz\sim\Dcal}(\phi_t(\bz))-v_t|\leq \tau$, where $\tau$ is a tolerance parameter. In this framework, one can often provide information-theoretic lower bounds on the number of queries or the required tolerance $\tau$ in order to succeed in learning. 

Although originally designed to study noise-tolerant algorithms, the SQ framework has also been increasingly used in recent years to provide evidence of computational and statistical complexity for various learning and estimation tasks (as discussed in the introduction). The underlying rationale is that a ``natural'' algorithm should be able to withstand some amount of noise. Thus, if SQ lower bounds imply that many statistical queries or a very small tolerance are required, this provides evidence that the runtime or sample complexity of ``natural'' algorithms will be large (to execute many queries or estimate expectations to sufficiently high accuracy from an empirical sample). In particular, stochastic gradient descent (SGD) has often been viewed as a ``noise-tolerant'' algorithm, in the sense that it iteratively descends along noisy approximations of the gradient of the population loss. 

In contrast, algorithms that can circumvent SQ lower bounds are considered to be necessarily brittle and impractical. The classic example for this, as mentioned in the introduction, is learning parity functions $\bx\mapsto f_{I}(\bx):=\sum_{i\in I}x_i~\text{mod}~ 2$ for $\bx$ in the Boolean hypercube and $I\subseteq [d]$ \citep{kearns1998efficient}. Given $\tilde{\Ocal}(d)$ i.i.d. samples $(\bx,f_{I}(\bx))$, where $\bx$ is sampled uniformly at random, $f_I$ can be identified in polynomial time by reducing the problem to solving a set of linear equations over the GF(2) field, using Gaussian elimination. However, such an algorithm is extremely noise-intolerant and does not resemble any standard learning method. Moreover, it is well-known that noise-tolerant algorithms (that can be simulated in the SQ framework) provably cannot solve this problem. More recently, \citet{abbe2020universality,abbe2021power} showed that parity learning can also be solved with SGD on a certain neural network, but the architecture is extremely unnatural (causing SGD to simulate arbitrary Turing machines, and Gaussian elimination in particular). So again, it seems that SQ lower bounds can be circumvented only via ``unnatural'' methods. 

However, as acknowledged in many papers on this topic, the connection between SQ lower bounds and computational/statistical complexity is merely a heuristic: We have no formal theorem establishing that SQ lower bounds necessarily apply to all ``natural'' algorithms commonly used in practice, such as SGD. Moreover, SQ lower bounds rely on the noise being adversarial (or at best, stochastic with some particular convenient structure). This generally differs from the noise encountered by such algorithms on actual samples.
%

To further motivate our research question, we note that there are already a few examples in the literature where SQ-based lower bounds provably mis-predict the complexity of various learning and estimation problems when using standard gradient-based methods. In particular:
\begin{itemize}
    \item \citet{li2019mean} provides SQ lower bounds on the problem of estimating the mean of random matrices, when the estimation error is measured in terms of Schatten-$p$ norms. In particular, Theorem 19 in that paper implies that estimating the mean of a random $d\times d$ matrix in the unit ball up to some fixed error (all with respect to the operator norm), using statistical queries, requires either exponentially many (in $\text{poly}(d)$) queries, or the error tolerance is $1/\text{poly}(d)$, implying $\text{poly}(d)$ predicted sample complexity. However, by the matrix Hoeffding bound \citep{tropp2012user}, the average of only $\Ocal(\log(d))$ samples suffices to approximate\ the mean (in operator norm) up to any constant error. Moreover, SGD can easily simulate the computation of the average (say by applying it on the stochastic problem $\min_{W} \E[-\text{trace}(WX)]$ for $\Ocal(\log(d))$ iterations with an appropriate step size). Thus, we see that SQ lower bounds predict $\text{poly}(d)$ sample complexity for this problem, whereas the actual sample complexity is $\Ocal(\log(d))$, easily attainable with SGD. 
    \item \citet{dudeja2021statistical} discuss SQ lower bounds for tensor PCA problems. A specific instantiation of their results is for the so-called spiked Wigner model: Namely, we are given i.i.d. samples of a $d\times d$ random matrix, whose distribution is 
    \[
    \bv\bv^\top+\Xi
    \]
    for some unknown unit vector $\bv$ and where $\Xi$ has i.i.d. standard Gaussian entries. The paper presents an SQ lower bound that implies a predicted sample complexity of order $d^2$. However, the correct sample complexity, both information-theoretically and via gradient-based methods, is only of order $d$. As the authors discuss in their paper, the issue appears to be the adversarial nature of the noise assumed in SQ lower bounds, which does not correctly capture the noise structure encountered by actual stochastic gradient methods for this problem. In particular, by considering gradient-based methods on the empirical version of the problem, $\max_{\bw:\norm{\bw}\leq 1} \bw^\top \tilde{X}\bw$ (where $\tilde{X}$ is the empirical matrix average), gradient-based methods will generally initialize close to the saddle point at the origin, where the gradient tends to be small, and hence can be zeroed out with adversarial perturbations, preventing any information about $\bv$ to leak. However, for analyzing actual gradient methods on the empirical average, the noise has a more favorable structure, which leads to faster convergence
    \item In a somewhat different direction, recent literature on single-index and multi-index learning used correlational statistical query (CSQ) arguments (a weaker variant of the SQ framework), to argue about the tightness of the information exponent \citep{arous2021online} in characterizing the complexity of learning such problems using gradient methods. However, as noted in \cite{dandibenefits,joshi2024complexity, lee2024neural, arnaboldi2024repetita}, these lower bounds can be easily circumvented by making small modifications to the gradient-based learning procedure, leading to faster convergence. While only showing an issue with the CSQ framework, which is weaker than SQ, it does show how lower bounds using SQ-based arguments can be brittle and depend on the precise assumptions used.
\end{itemize}

In what follows, let us introduce yet another, self-contained example of how SQ-based lower bounds can lead to incorrect predictions for simple supervised learning problems. To that end, it will be convenient to slightly expand the SQ framework to vector-valued queries (so it can directly capture, e.g., computation of stochastic gradients). Concretely, given some tolerance parameter $\tau$, we assume that the algorithm can iteratively choose a query function $\phi$, mapping training examples $(\bx,y)$ to the Euclidean unit ball, and get a vector $\bv$ such that
\[
\norm{\E_{(\bx,y)}[\phi(\bx,y)]-\bv}_2\leq \tau~. 
\]

Now, let us consider specifically the Boolean function class 
\[
\Fcal=\{f_i(\bx)=x_i, i\in [d]\}~~,~~ \bx\in\Xcal_d:=\{-1,+1\}^d~.
\]
Moreover, let $\Dcal_i$ be the distribution of $(\bx,y)=(\bx,f_i(\bx))$, where $\bx$ is uniform on $\Xcal_d$. 

Since our function class contains only $d$ candidate target functions, it is easy to see that information theoretically, only $\Ocal(\log(d))$ samples from any $\Dcal_i$ are required to identify $f_i$. Moreover, \citet[Theorem 9]{amid2020winnowing} show that using a simple sparsely-connected two-layer neural network, SGD on the squared loss identifies $i$ after $\Ocal(\log(d))$ iterations. However, as the following theorem shows, the SQ framework predicts $\text{poly}(d)$ complexity, which is exponentially worse:


\begin{theorem}\label{thm:sqbad}
    For any $\tau\in (0,1)$, there is an SQ oracle as above, such that any SQ algorithm with tolerance $\tau$ will require at least $\frac{1}{2}\tau^2 d$ queries to identify $f_i\in \Fcal$ with probability $\geq \frac{1}{2}+\frac{1}{d}$.
\end{theorem}
In particular, any choice of $\tau$ will imply either $\poly(1/d)$ tolerance (leading to predicted sample complexity polynomial in $d$, in order to estimate the expectation up to that tolerance), or $\poly(d)$ queries (again leading to predicted sample complexity polynomial in $d$, for SGD methods). The proof is based on standard SQ arguments and is presented in the following section of the appendix (Appendix \ref{app:thmsqbad}).

\begin{remark}\label{remark:vstat}
    A stronger oracle common in the SQ literature, namely the VSTAT oracle, requires the tolerance to be smaller if the query tends to have small values, to capture the fact that the expectation of random vectors with small moments can be estimated more accurately. In our vector-valued setting, such a requirement can take the form
    $
    \norm{\E[\phi(\bx,y)]-\bv}\leq \tau \cdot \sqrt{\E[\norm{\phi(\bx,y)^2}]}
    $. 
    It can be easily verified that the theorem and its proof apply directly, even to such a stronger oracle.  
\end{remark}

The reason that this SQ framework leads to a wrong prediction is that its Euclidean geometry is not commensurate with the geometry of the stochastic noise in this particular problem. One way to see this is to consider perhaps the most natural statistical query for this problem, namely $\phi(\bx,y)=y\bx/\sqrt{d}$. Assuming the target function is $f_i(\bx)=y=x_i$ for some $i$, we observe that $\phi(\bx,y)$ is a random vector, whose $i$-th coordinate is $1/\sqrt{d}$ with probability one, whereas all the other coordinates are independent and uniform on $\{-\frac{1}{\sqrt{d}},+\frac{1}{\sqrt{d}}\}^d$. Thus, averaging $\phi(\bx,y)$ over $\Ocal(\log(d))$ samples and picking the coordinate which equals $1$ will identify $i$ with high probability. However, the SQ framework we use here views $\phi(\bx,y)$ as a random vector in the unit Euclidean ball, where we need to identify its mean from several possible alternatives, each only $\Ocal(1/\sqrt{d})$-far from the origin (in Euclidean distance). In general, estimating the mean of a random vector in the unit Euclidean ball up to accuracy $\Ocal(1/\sqrt{d})$, requires $\Omega(d)$ samples. However, this is tight only in the worst-case over all such distributions, and ignores the fact that for our particular problem, the random vector has a more specific structure, which leads to a much better sample complexity. 

    

\subsection{Proof of Theorem \ref{thm:sqbad}}\label{app:thmsqbad}
    The proof relies on well-known SQ techniques, but we present it in a self-contained way for completeness. In what follows, we let $\Dcal_0$ denote a ``reference distribution'', which is simply the uniform distribution on $(\bx,y)\in \{-1,+1\}^{d+1}$.

    \begin{lemma}
        For any $\phi$ as above, 
        \[
\sum_{i=1}^{d}\left\|\E_{\Dcal_i}[\phi(\bx,y)]-\E_{\Dcal_0}[\phi(\bx,y)]\right\|^2~\leq~ \E_{\Dcal_0}[\norm{\phi(\bx,y)}^2]~\leq~ 1.
        \]
    \end{lemma}
    \begin{proof}
        It will suffice to prove the above in the case where $\phi(\bx,y)$ is scalar-valued, namely
        \begin{equation}\label{eq:sqbad0}
                    \sum_{i=1}^{d}\left(\E_{\Dcal_i}[\phi(\bx,y)]-\E_{\Dcal_0}[\phi(\bx,y)]\right)^2~\leq~ \E_{\Dcal_0}[\phi^2(\bx,y)]~.
        \end{equation}
        The vector-valued case follows by applying and summing the inequality for each coordinate of $\phi(\bx,y)$, and noting that by definition of $\phi$, the overall squared norm of $\phi(\bx,y)$ is at most $1$.

        Let $p_i(\bx,y),p_0(\bx,y)$ denote the probability density of $\Dcal_i,\Dcal_0$ respectively. Noting that the marginal distribution of $\bx$ is the same for $\Dcal_i,\Dcal_0$, we have
        \begin{align*}
            \E_{\Dcal_i}[\phi(\bx,y)]-\E_{\Dcal_0}[\phi(\bx,y)]~&=~ \E_{\bx}\left[\sum_{y}\phi(\bx,y)\left(p_i(y|\bx)-p_0(y|\bx)\right)\right]\\
            &=~ \E_{\bx}\left[\phi(\bx,1)\left(\mathbf{1}_{f_i(\bx)=1}-\frac{1}{2}\right)+\phi(\bx,-1)\left(\mathbf{1}_{f_i(\bx)=-1}-\frac{1}{2}\right)\right]\\
            &=~
            \E_{\bx}\left[f_i(\bx)\cdot \frac{\phi(\bx,1)-\phi(\bx,-1)}{2}\right]~.
        \end{align*}
        Letting $\psi(\bx):=\frac{\phi(\bx,1)-\phi(\bx,-1)}{2}$ and plugging in the above, it follows that
        \begin{align*}
        \sum_{i=1}^{d}&\left(\E_{\Dcal_i}[\phi(\bx,y)]-\E_{\Dcal_0}[\phi(\bx,y)]\right)^2~=~ 
            \sum_{i=1}^{d}\left( \E_{\bx}[f_i(\bx)\psi(\bx)]\right)^2~\stackrel{(1)}{\leq}~\E_{\bx}[\psi^2(\bx)]\\
            &=~\E_{\bx}\left[\left(\frac{\phi(\bx,1)-\phi(\bx,-1)}{2}\right)^2\right]~\stackrel{(2)}{\leq}~\E_{\bx}\left[\frac{\phi^2(\bx,1)+\phi^2(\bx,-1)}{2}\right]~=~
        \E_{\Dcal_0}[\phi^2(\bx,y)]~,
        \end{align*}
        where $(1)$ is by Bessel's inequality (observing that $f_1,\ldots,f_d$ are orthonormal functions in $L^2$ space equipped with the standard inner product $\inner{f_i,f_j}=\E_{\bx}[f_i(\bx)f_j(\bx)]=\E[x_i x_j]=\mathbf{1}_{i=j}$), and $(2)$ is by the elementary inequality $\left(\frac{a-b}{2}\right)^2\leq \frac{a^2+b^2}{2}$. This implies \eqref{eq:sqbad0}, from which the lemma follows.
    \end{proof}

    We now show how the lemma implies our theorem. Slightly rewriting the lemma, we have that
    \[
\frac{1}{d}\sum_{i=1}^{d}\left\|\E_{\Dcal_i}[\phi(\bx,y)]-\E_{\Dcal_0}[\phi(\bx,y)]\right\|^2 ~\leq~ \frac{1}{d}.
    \]    
    By Markov's inequality, this implies that for any fixed $\phi$, if we choose $i\in[d]$ uniformly at random, then 
    \[
    \Pr_i\left(\left\|\E_{\Dcal_i}[\phi(\bx,y)]-\E_{\Dcal_0}[\phi(\bx,y)]\right\|^2 \geq \tau^2\right)~\leq~\frac{1}{\tau^2 d}~.
    \]
    Now, let us consider how any SQ algorithm behaves when $i$ and the target function $f_i$ are chosen uniformly at random. Initially, the algorithm chooses a query $\phi_1$ that does not depend on $i$. By the inequality above, with probability at least, $1-\frac{1}{\tau^2 d}$, $\phi_1$ is such that its expectation $\E_{\Dcal_i}[\phi_1(\bx,y)]$ is $\tau$-close to the vector $E_{\Dcal_0}[\phi_1(\bx,y)]$, so the SQ oracle may return this vector. Crucially, this vector does not depend on $i$, so the algorithm does not obtain any information about $i$. As a result, its next query $\phi_2$ will also not depend on $i$, in which case (using the same reasoning), it may get $E_{\Dcal_0}[\phi_2(\bx,y)]$ with probability at least $1-\frac{1}{\tau^2 d}$. By a union bound, after at most $\tau^2 d/2$ queries, the algorithm still receives no information on $i$, with probability at least $\frac{1}{2}$. Hence, its output will be independent of $i$, so it will succeed with probability at most $1/d$. Overall, again by a union bound, we get that the algorithm cannot succeed with probability more than $\frac{1}{2}+\frac{1}{d}$.

\section{Proof of \thmref{thm: main}}\label{app:thmmainproof}
\subsection{Concentration Results}
In what follows, we let $(\bg_t)_{t=1}^{\infty}$ denote a martingale difference sequence on $\reals^d$. We let $\E_t,\Pr_t$ denote expectation and probability with respect to $\bg_t$, conditioned on $\bg_1,\ldots,\bg_{t-1}$ (for example,  $\E_t[\bg_t]=0$ for all $t$ by definition). 

\begin{definition}\label{def: subgauss_mds}
    A martingale difference sequence $(\bg_t)_{t=1}^{T}$ on $\reals^d$ (with corresponding filtration $\Fcal_t=\sigma(\bg_1,\ldots,\bg_t)$), is said to be \emph{$(\lambda_t)_{t=1}^\infty$ sub-Gaussian (for $\lambda_t\in \Fcal_{t-1}$)} if for all $t$, $\gamma>0$ and $\bv\in \reals^d$ with $\norm{\bv}=1$,
    \[
        \Pr\nolimits_t\left(|\bv^\top\bg_t|\geq \gamma \right)~\leq~ 2\exp\left(-\frac{\gamma^2}{\lambda_t^2}\right)~.
    \]
\end{definition}

\begin{lemma}[\citet{jin2019short}]\label{lem:azuma}
    Let $(\bz_j)_{j=1}^{t}$ be a martingale difference sequence on $\reals^d$ (with corresponding filtration $\Fcal_j=\sigma(\bz_1,\ldots,\bz_j)$), such that for all $\gamma>0$, $\Pr_j(\norm{\bz_j}\geq \gamma)\leq 2\exp(- \gamma^2/\sigma_j^2)$, for some fixed $\sigma_j>0$. Then for any $\delta\in (0,1)$, it holds with probability at least $1-\delta$ that
    \[  
    \norm{\sum_{j=1}^t \bz_j} \leq c\sqrt{\sum_{j=1}^t \sigma_j^2 \log\left(\frac{2d}{\delta}\right)}
    \]
    where $c>0$ is a universal constant. 
\end{lemma}

A straightforward corollary regarding one-dimensional projections of a sub-Gaussian MDS (Definition \ref{def: subgauss_mds}) is the following.
\begin{lemma}\label{lem: bound_signal}
    Let $\bg_1, \ldots, \bg_T \in \R^d$ be a $(\lambda_t)_{t=1}^\infty$ sub-Gaussian martingale difference sequence (Definition \ref{def: subgauss_mds}) with fixed $\{\lambda_t\}$, and $\bv \in \R^d$ a fixed nonzero vector, then it holds with probability at least $1-\delta$ that
    \begin{align*}
        \forall t\in [T] ~: \qquad 
        \abs{\sum_{j=1}^t \bg_j^\top \bv} \leq c\sqrt{\sum_{j=1}^t \lambda^2_j \norm{\bv}^2 \log\left(\frac{2T}{\delta}\right)} 
        ~,
    \end{align*}
    where $c$ is an absolute constant.
\end{lemma}
\begin{proof}
    Fix $\delta'>0$ to be decided later. Since $(\bg_t)_t$ is a sub-Gaussian MDS (Definition \ref{def: subgauss_mds}) it holds that for any $t \in [T]$, 
    \begin{align*}
        \Pr\nolimits_t\left(\abs{\bg_t^\top \bv}\geq \gamma\right) = \Pr\nolimits_t\left(\abs{\bg_t^\top \frac{\bv}{\norm{\bv}}} \geq \frac{\gamma}{\norm{\bv}}\right) \leq 2\exp\left(-\frac{\gamma^2}{\lambda^2_t \norm{\bv}^2}\right) ~.
    \end{align*}
    For any $t\in[T]$, applying \lemref{lem:azuma} (in dimension 1) up to time $t$, it holds with probability at least $1-\delta'$ that for some absolute constant $c>0$,
    \begin{align*}
     \abs{\sum_{j=1}^t \bg_j^\top \bv} \leq c\sqrt{\sum_{j=1}^t \lambda^2_j \norm{\bv}^2 \log\left(\frac{2}{\delta'}\right)} ~.
    \end{align*}  
    Taking $\delta'=\delta/T$ and applying a union bound over all $t\in [T]$ completes the proof 
\end{proof}

\begin{lemma}[\citet{jin2019short}]\label{lem:martcon}
    Let $(\bz_j)_{j=1}^{t}$ be a martingale difference sequence on $\reals^d$ (with corresponding filtration $\Fcal_j=\sigma(\bz_1,\ldots,\bz_j)$), such that for all $\gamma>0$, $\Pr_j(\norm{\bz_j}\geq \gamma)\leq 2\exp(- \gamma^2/\sigma_j^2)$ for some possibly-random $\sigma_j\in \Fcal_{j-1}$. Then for any $\delta\in (0,1)$ and $B,b>0$ with $B\geq eb$, it holds with probability at least $1-\delta$ that
    \[  
    \text{either} \sum_{j=1}^{t}\sigma_j^2\geq B~~~~\text{or}~~~~\left\|\sum_{j=1}^{t}\bz_j\right\|~\leq~c\cdot\sqrt{\max\left\{\sum_{j=1}^{t}\sigma_j^2,b\right\}\cdot\left(\log\left(\frac{2d}{\delta}\right)+\log\log\left(\frac{B}{b}\right)\right)}~,
    \]
    where $c>0$ is a universal constant. 
\end{lemma}

A straightforward corollary of this, in a form more suitable for our purposes, is the following:

\begin{lemma}\label{lem:martcon2}
    Let $(\bz_j)_{j=1}^{T}$ be a martingale difference sequence on $\reals^d$ (with corresponding filtration $\Fcal_j=\sigma(\bz_1,\ldots,\bz_j)$), such that for all $\gamma>0$, $\Pr_j(\norm{\bz_j}\geq \gamma)\leq 2\exp(- \gamma^2/\sigma_j^2)$ for some possibly-random $\sigma_j\in \Fcal_{j-1}$. For each $j\in[T]$ let $\beta_j > 0$ be deterministic.
    Then for any $\delta\in (0,1)$, it holds with probability at least $1-\delta$ that for all $t\in [T]$, either $\sum_{j=1}^{t}\sigma_j^2\geq \sum_{j=1}^t \beta_j^2$ or
    \begin{align*}
    \left\|
    \sum_{j=1}^{t}\bz_j\right\|
    ~\leq~c
    \cdot \sqrt{\sum_{j=1}^{t}\beta_j^2 \cdot\left(\log\left(\frac{2Td}{\delta}\right) + \log\log\left(e\frac{\sum_{j=1}^t \beta_j^2}{\beta_1^2}\right)\right)}~.
    \end{align*}  
\end{lemma}
\begin{proof}
    Fix some $t\in [T]$ and $\delta'>0$ to be decided later. Let $B=\sum_{j=1}^t \beta_j^2$ and $b=\beta_1^2 / e$. It trivially holds that $B\geq eb$, so applying \lemref{lem:martcon} with these choices, we get that  with probability at least $1-\delta'$, either $\sum_{j=1}^{t}\sigma_j^2\geq \sum_{j=1}^t \beta_j^2$ or
    \begin{align*}
    \left\|
    \sum_{j=1}^{t}\bz_j\right\|
    ~\leq~c
    \cdot\sqrt{\max\left\{\sum_{j=1}^{t}\sigma_j^2, \frac{\beta_1^2}{e}\right\}\cdot\left(\log\left(\frac{2d}{\delta'}\right) + \log\log\left(e\frac{\sum_{j=1}^t \beta_j^2}{\beta_1^2}\right)\right)}~.
    \end{align*}  
    Note that in the event that $\sum_{j=1}^{t}\sigma_j^2\leq \sum_{j=1}^t \beta_j^2$, we may upper bound $\max\left\{\sum_{j=1}^{t}\sigma_j^2, \frac{\beta_1^2}{e}\right\} \leq \sum_{j=1}^t \beta_j^2$. Taking $\delta' = \delta / T$ and applying a union bound over all $t\in [T]$ finishes the proof. 
\end{proof}

\begin{lemma}[Structure at Initialization]\label{lem: w0}
    Under Assumption \ref{ass: init}, let $\bw_0 := W_0^\top \bv$ for some fixed unit vector $\bv$. There exists an absolute constant $c>0$  any $\delta > 0$ such that 
    \begin{enumerate}
        \item For all nonzero $\bu \in \R^d$, 
        \begin{align*}
        \Pr\left(\abs{\langle \bw_0, \bu \rangle} \geq \gamma \right) \leq 2\exp\left(-\frac{\gamma^2 d}{c^2K_3^2 \norm{\bu}^2}\right).
        \end{align*}
    \end{enumerate}
    Moreover, all of the following hold simultaneously with probability at least $1-\delta$:
    \begin{enumerate}[start=2]
        \item For all $\bu \in \R^d$, $\abs{\langle \bw_0, \bu \rangle} \leq cK_3 \norm{\bu} \sqrt{\frac{\log\left(\frac{4}{\delta}\right)}{d}}$, 
        \item $\abs{\norm{\bw_0} -1} \leq  cK_3^2 \sqrt{\frac{\log\left(\frac{4}{\delta}\right)}{d}}$.
    \end{enumerate}
\end{lemma}
\begin{proof}
    (1) By \citet{vershynin2010introduction}[Lemma 5.24], Assumption \ref{ass: init} implies that for some absolute constant $C>0$, any $(W_0^\top)_i$ (where $i\in[m]$) and any $\gamma > 0$
    \begin{align*}
        \Pr\left(\abs{\langle (W_0^\top)_i, \bu \rangle} \geq \gamma \right) \leq 2\exp\left(-\frac{\gamma^2 d}{C^2K_3^2 \norm{\bu}^2}\right).
    \end{align*}
    
    As such, by the sub-Gaussian Hoeffding inequality \citep{vershynin2025high}[Theorem 2.7.3] it holds that
    \begin{align*}
        \Pr\left(\abs{\langle \bw_0, \bu \rangle} \geq \gamma \right) = & \Pr\left(\abs{\sum_{i=1}^m v_i \langle (W_0^\top)_i, \bu \rangle} \geq \gamma \right)
        \leq 2\exp\left(-\frac{\gamma^2 d}{C^2K_3^2 \norm{\bu}^2 \norm{\bv}^2} \right) \\ 
        = & 2\exp\left(-\frac{\gamma^2 d}{C^2K_3^2 \norm{\bu}^2} \right),
    \end{align*}
    where we replaced $C$ by another absolute constant.
    Taking $\gamma = CK_3 \norm{\bu} \sqrt{\frac{\log\left(\frac{4}{\delta}\right)}{d}}$ gives (2) with probability at least $1-\delta/2$.
    For (3), first note that $\sqrt{d}\bw_0$ has i.i.d. variance 1 coordinates since for any $j\in[d]$, $\var(\sum_i (W_0)_{i,j} v_i) = \sum_i \var((W_0)_{i,j}) v_i^2 = \sum_i v_i^2 = 1$). Furthermore, (1) implies that $\sqrt{d}\bw_0$ is $CK_3$ sub-Gaussian, in the sense that for any nonzero $\bu\in\R^d$
    \begin{align*}
        \Pr\left(\abs{\langle \sqrt{d}\bw_0, \bu \rangle} \geq \gamma \right) \leq 2\exp\left(-\frac{\gamma^2}{C^2K_3^2 \norm{\bu}^2}\right).
    \end{align*}
    As such, \citet{vershynin2025high}[Theorem 3.1.1] implies that for any $\gamma > 0$
    \begin{align*}
        \Pr\left(\abs{\norm{\bw_0} - 1} \geq \frac{\gamma}{\sqrt{d}} \right) = \Pr\left(\abs{\norm{\sqrt{d}\bw_0} - \sqrt{d}} \geq \gamma \right) \leq 2\exp\left(-\frac{\gamma^2}{C' K_3^4}\right)
    \end{align*}
    where $C'$ is another universal constant. 
    Taking $\gamma = C'K_3^2 \sqrt{\log\left(\frac{4}{\delta}\right)}$ gives (3) with probability at least $1-\delta/2$. 
    
    As a final note, in the lemma formulation, the constant $c$ is taken to be the maximum of all constants that we need.
\end{proof}

We will also use the following bound on the spectral norm of $W_0$:
\begin{lemma}[\cite{vershynin2025high} Theorem 4.6.1]\label{lem: op_w0}
    Under Assumption \ref{ass: init}, and assuming $m \leq d$, there exists an absolute constant $C>0$, such that for any $\delta>0$ it holds with probability at least $1-\delta$ that
    \begin{align*}
        1 -  CK_3^2\left(\sqrt{\frac{m}{d}} + \sqrt{\frac{\log\left(\frac{2}{\delta}\right)}{d}}\right) 
        \leq s_m(W_0) \leq s_1(W_0) 
        \leq 1 +  CK_3^2\left(\sqrt{\frac{m}{d}} + \sqrt{\frac{\log\left(\frac{2}{\delta}\right)}{d}}\right),
    \end{align*}
    where $s_i(W_0)$ denote the singular values of $W_0$ in decreasing order.
\end{lemma}

\subsection{Norm Bounds}
Throughout the proofs, $\bx_t$ will always denote i.i.d inputs as described in Section \ref{sec:prelim}. $W_t$ will denote the parameters at time $t$. The lemmas in this subsection will provide bounds on the norm of the weights throughout training.
\begin{lemma}\label{lem: norm_step1}
    Under Assumption \ref{ass: norm_conc}, for any $\delta > 0$, with probability at least $1-\delta$ the following holds for all $t\in[T]$:
    \begin{align*}
        \abs{\norm{\bx_t}^2 - \alpha_2 d}  
        \leq 
        K_2^2 \left(\sqrt{d\log\left(\frac{2T}{\delta}\right)} + \log\left(\frac{2T}{\delta}\right)\right).
    \end{align*}
    In particular, if $d\geq \max\left(\frac{16 K_2^4}{\alpha_2^2} , \frac{4 K_2^2}{\alpha_2}\right) \log\left(\frac{2T}{\delta}\right)$, then the above event implies that for all $t\in[T]$,
    \begin{align*}
        \frac{\alpha_2 d}{2} \leq \norm{\bx_t}^2 \leq \frac{3 \alpha_2 d}{2}.
    \end{align*}
\end{lemma}
\begin{proof}
    By Assumption \ref{ass: norm_conc}, for any $t\in [T]$, $\delta'>0$
    \begin{align*}
        \Pr\left(\abs{\norm{\bx_t}^2 - \alpha_2 d} \geq \max\left(K_2^2 \sqrt{d\log\left(\frac{2}{\delta'}\right)} ~,~ K_2^2 \log\left(\frac{2}{\delta'}\right)\right) \right) \leq \delta'.
    \end{align*}
    Taking $\delta' = \delta / T$ and applying a union bound over all $t \in [T]$, it follows that with probability at least $1-\delta$, for all $t\in [T]$
    \begin{align*}
        \abs{\norm{\bx_t}^2 - \alpha_2 d}  
        \leq 
        K_2^2 \left(\sqrt{d\log\left(\frac{2T}{\delta}\right)} + \log\left(\frac{2T}{\delta}\right)\right).
    \end{align*}
    For the ``moreover'' part, assume $d\geq \max\left(\frac{16 K_2^4}{\alpha_2^2} , \frac{4 K_2^2}{\alpha_2}\right) \log\left(\frac{2T}{\delta}\right)$. Then
    \begin{align*}
        \log\left(\frac{2T}{\delta}\right)\leq \frac{\alpha_2^2 d}{16 K_2^4}
        \qquad\text{and}\qquad
        \log\left(\frac{2T}{\delta}\right)\leq \frac{\alpha_2 d}{4 K_2^2}.
    \end{align*}
    Therefore,
    \begin{align*}
        \abs{\norm{\bx_t}^2 - \alpha_2 d}  
        \leq
        K_2^2 \left(\sqrt{d\log\left(\frac{2T}{\delta}\right)} + \log\left(\frac{2T}{\delta}\right)\right) 
        \leq  K_2^2 \left(\sqrt{\frac{\alpha_2^2 d^2}{16 K_2^4}} + \frac{\alpha_2 d}{4 K_2^2}\right) = \frac{\alpha_2 d}{2}. 
    \end{align*}
\end{proof}

\begin{lemma}\label{lem: norm_lower}
    Let $T\in\mathbb N$, and let $(a_t)_{t=1}^T$ be a real-valued stochastic process adapted to a filtration $(\mathcal F_t)_{t\geq1}$. Assume that for all $t\in[T]$ there is some fixed $\beta_t>0$ such that $\abs{a_t}\leq \beta_t$ a.s. Then under Assumption \ref{ass: norm_conc}, for any $\delta > 0$, if $d\geq \max\left(\frac{16 K_2^4}{\alpha_2^2} , \frac{4K_2^2}{\alpha_2}\right) \log\left(\frac{2T}{\delta}\right)$, then with probability at least $1-\delta$ it holds for all $t\in[T]$ that 
    \begin{align*}
        \sum_{j=1}^t a_j^2\norm{\bx_j}^2 
        \geq \frac{\alpha_2 d}{2} \left(\sum_{j=1}^t \E_j\left[a_j^2\right] - c\sqrt{\sum_{j=1}^t \beta_j^4 \log\left(\frac{4T}{\delta}\right)}\right),
    \end{align*}
    where $c>0$ is a universal constant.
\end{lemma}
\begin{proof}
    Applying Lemma \ref{lem: norm_step1} with $\delta/2$, it holds with probability at least $1-\delta/2$ that for all $t\in [T]$
    \begin{align}\label{eq: norm_lower_p1}
        \sum_{j=1}^t a_j^2\norm{\bx_j}^2 \geq \frac{\alpha_2 d}{2}\sum_{j=1}^t a_j^2.
    \end{align}
    Now $z_j:= a_j^2 - \E_j[a_j^2]$ is a real valued martingale difference sequence, and using that by assumption, $\abs{a_j}^2 \leq \beta_j^2$ for some $\beta_j > 0$, we have $\abs{z_j}\leq 2\beta_j^2$ and thus the conditions of \lemref{lem:azuma} are satisfied (with $\sigma_j = C' \beta_j^2$ for a suitable universal constant $C'>0$; see for example \citet{vershynin2025high}[Example 2.6.5]). So from \lemref{lem:azuma} and a union bound over $t\in[T]$, it follows that for some absolute constant $c>0$, with probability at least $1-\delta / 2$ it holds that for all $t\in[T]$, 
    \begin{align*}
        \abs{\sum_{j=1}^t a_j^2 - \sum_{j=1}^t\E_j[a_j^2]} \leq c\sqrt{\sum_{j=1}^t \beta_j^4 \log\left(\frac{4T}{\delta}\right)}.
    \end{align*}
    Plugging this into \eqref{eq: norm_lower_p1} completes the proof.
\end{proof}

\begin{lemma}\label{lem: delta_bound}
    Let $\bg_1, \ldots, \bg_T \in \R^d$ be a $(\lambda_t)_{t=1}^\infty$ sub-Gaussian martingale difference sequence (Definition \ref{def: subgauss_mds}) with fixed $\{\lambda_t\}$, let $\bw_t := \sum_{j=1}^t \bg_j$ for all $t\in[T]$ and $\Delta_t := \bw_{t-1}^\top \bg_t$ for all $t\geq 2$. Then there exists some absolute constant $c>0$, such that for any $\delta > 0$ it holds with probability at least $1-\delta$ that for all $t\in [T]$
    \begin{align*}
        \abs{\sum_{j=2}^t \Delta_j} 
        \leq c^2 \sum_{j=1}^t \lambda^2_j \sqrt{d \log\left(\frac{2dT}{\delta}\right) \cdot \left(\log\left(\frac{2T}{\delta}\right) + \log\log\left(eT^2 \max_{i,j \in [T]} \frac{\lambda^2_i\lambda^2_j}{\lambda^2_1\lambda^2_2}\right)\right)}~. 
    \end{align*}
\end{lemma}

\begin{proof}
    Fix some $\delta'>0$ to be decided later. By Definition \ref{def: subgauss_mds} and \cite{jin2019short}[Lemma 1], there exists an absolute constant $c'>0$ such that for any $t \in [T]$, $\gamma > 0$
    \begin{align*}
        \Pr\nolimits_t\left(\norm{\bg_t}\geq \gamma\right) 
        \leq 2\exp\left(-\frac{\gamma^2}{c'd\lambda^2_t}\right) ~.
    \end{align*}
    For any $t\in[T]$, applying \lemref{lem:azuma} up to time $t$, it holds with probability at least $1-\delta'$ that for some absolute constant $c>0$,
    \begin{align}\label{eq: w_norm}
     \norm{\bw_t} = \norm{\sum_{j=1}^t \bg_j} \leq c\sqrt{d \sum_{j=1}^t \lambda^2_j \log\left(\frac{2d}{\delta'}\right)} ~.
    \end{align}  
    Let $\Ecal$ be the event for which \eqref{eq: w_norm} holds simultaneously for all $t\in[T]$. By \eqref{eq: w_norm} and the union bound, $\Ecal$ holds with probability at least $1-\delta'T$.
    Note that $(\Delta_t)_{t=2}^{T}$ is a martingale difference sequence (since $\E_t[\bw_{t-1}^\top \bg_t] = \bw_{t-1}^\top \E_t[\bg_t]= 0$). We now wish to apply \lemref{lem:martcon2} on this sequence. First, notice that by Definition \ref{def: subgauss_mds}, for any $\gamma>0$ and any $j$,
    \begin{align*}
    \Pr\nolimits_j \left(|\Delta_j|\geq \gamma\right)
    = & ~
    \Pr\nolimits_j\left(|\bw_{j-1}^\top \bg_j|\geq \gamma\right)
    =
    \Pr\nolimits_j\left(\abs{\frac{\bw_{j-1}^\top}{\norm{\bw_{j-1}}} \bg_j}\geq \frac{\gamma}{\norm{\bw_{j-1}}}\right) \\
    \leq & ~ 2\exp\left(-\frac{\gamma^2}{\lambda^2_j \norm{\bw_{j-1}}^2}\right)~ 
    .
    \end{align*}
     As such, applying \lemref{lem:martcon2} with $\sigma_j:=\lambda_j \norm{\bw_{j-1}}$ and $\beta_j := 2c\lambda_j \sqrt{d \sum_{i=1}^{j-1} \lambda^2_i \log\left(\frac{2d}{\delta'}\right)}$, (where $\beta_j$ is deterministic since $\lambda_j$ are deterministic), it holds with probability at least $1-\delta'$ that for all $t\in[T]$, either $\sum_{j=1}^{t}\sigma_j^2\geq \sum_{j=1}^t \beta_j^2$ or
    \begin{align}\label{eq: delta_bound}
     \abs{\sum_{j=2}^t \Delta_j} 
     \leq &
     c^2 \cdot\sqrt{2d\sum_{j=2}^{t} \lambda^2_j \sum_{i=1}^{j-1} \lambda^2_i \log\left(\frac{2d}{\delta'}\right) \cdot \left(\log\left(\frac{2T}{\delta'}\right) + \log\log\left(\frac{e\sum_{j=2}^{t} \lambda^2_j \sum_{i=1}^{j-1} \lambda^2_i}{\lambda^2_1\lambda^2_2}\right)\right)} \nonumber\\ 
     \leq & c^2 \sum_{j=1}^t \lambda^2_j \sqrt{2d \log\left(\frac{2d}{\delta'}\right) \cdot \left(\log\left(\frac{2T}{\delta'}\right) + \log\log\left(eT^2 \max_{i,j \in [T]} \frac{\lambda^2_i\lambda^2_j}{\lambda^2_1\lambda^2_2}\right)\right)}.
    \end{align}   
    Moreover, on the event $\Ecal$ we have $\sigma_j \leq c\lambda_j \sqrt{d \sum_{i=1}^{j-1}\lambda_i^2 \log\left(\frac{2d}{\delta'}\right)}=\beta_j/2$ for all $j$, and hence the alternative $\sum_{j=1}^{t}\sigma_j^2\geq \sum_{j=1}^t \beta_j^2$ cannot occur. As such, taking $\delta' = \delta / (T+1)$ ensures that both $\Ecal$ and \eqref{eq: delta_bound} hold for all $t\in[T]$ with probability at least $1-\delta$. Adjusting the constant $c$ completes the proof.
\end{proof}

\begin{lemma}\label{lem: lower_norm}
    Let $T\in\mathbb N$, and let $(a_t)_{t=1}^T$ be a real-valued stochastic process adapted to a filtration $(\mathcal F_t)_{t\geq1}$. Let $\bw_0:=W_0^\top \bv$ for some fixed unit vector $\bv$. Assume that for all $t\in[T]$ there is some $\beta > 0$ such that $\abs{a_t}\leq \beta$ a.s. Then under Assumptions \ref{ass: subgauss_inputs}, \ref{ass: norm_conc}, \ref{ass: init}, for any $\delta >0$, $\kappa \geq 1$ there exist constants $C, c, c', c''>0$ that may depend on $K_1, K_2, K_3, \alpha_2, \beta$, such that if $d \geq c \kappa^2\log\left(\frac{T}{\delta}\right)^2$, $\eta \leq c'/\sqrt{d}$, then conditioned on $\min_{j \leq T}\E_j[a_j^2] \geq \beta^2/\kappa$, with probability at least $1-\delta$ it holds for all $t\in\left[c'' \kappa^2 \log\left(T/\delta \right), T\right]$ that 
    \begin{align*}
        \norm{\bw_0 + \eta \sum_{j=1}^t a_j \bx_j} \geq C\left(1 + \eta\sqrt{\frac{dt}{\kappa}}\right) - 2\eta t \max_{j \leq t}\norm{\E_j\left[a_j\bx_j\right]}.
    \end{align*}
\end{lemma}
\begin{proof}
    Throughout the proof we use the convention that $C_i, c_i$ will denote constants that may depend on $K_1, K_2, K_3, \alpha_2$ and $\beta$ (but importantly, \textbf{not} on $d, t, \eta$ and $\kappa$). First, note that for any $p,q\in\R$, $(p-q)^2 \geq p^2 / 2 - 2q^2$, and thus by the reverse triangle inequality, for any two vectors $\bv_1,\bv_2$,
    \begin{align}\label{eq: young}
        \norm{\bv_1 + \bv_2}^2 \geq \left(\norm{\bv_1} - \norm{\bv_2}\right)^2 \geq \frac{1}{2}\norm{\bv_1}^2 - 2\norm{\bv_2}^2.
    \end{align}
    Let $a_0 = 1/(\sqrt{d}\eta)$, $\bx_0 = \sqrt{d}\bw_0$, $\bg_0:=a_0\bx_0=\frac{1}{\eta}\bw_0$, and for any $t\geq1$ let $\bmu_t:=\E_t[a_t \bx_t]$, $\bg_t:=a_t \bx_t - \bmu_t$ and $\bw_t := \sum_{j=0}^t \bg_j$. We decompose the norm as
    \begin{align}\label{eq: master_lower}
        \norm{\sum_{j=0}^t a_j\bx_j}^2 
        = & \norm{\sum_{j=0}^t \bg_j + \sum_{j=1}^t \bmu_j}^2 
        \geq_{(1)} \frac{1}{2}\norm{\sum_{j=0}^t\bg_j}^2 - 2\norm{\sum_{j=1}^t\bmu_j}^2 \nonumber \\ 
        \geq& \frac{1}{2}\sum_{j=0}^t\norm{\bg_j}^2 +\sum_{0\leq i < j \leq t} \langle \bg_i, \bg_j \rangle - 2t^2 \max_{j\leq t}\norm{\bmu_j}^2 \nonumber \\ 
        \geq&_{(3)}
        \frac{1}{4}\sum_{j=0}^t\norm{ a_j\bx_j}^2 - \sum_{j=1}^t\norm{\bmu_j}^2 + \sum_{0\leq i < j \leq t} \langle \bg_i, \bg_j \rangle - 2t^2 \max_{j\leq t}\norm{\bmu_j}^2 \nonumber \\ 
        \geq & \frac{1}{4\eta^2}\norm{\bw_0}^2 + \frac{1}{4}\sum_{j=1}^t\norm{ a_j\bx_j}^2 + \sum_{j=1}^t \langle \bw_{j-1}, \bg_j \rangle - 3t^2 \max_{j\leq t}\norm{\bmu_j}^2,
    \end{align}
    where $(1)$ and $(3)$ used \eqref{eq: young}.

    We now proceed to bound \eqref{eq: master_lower} using the previous lemmas. First, for the initialization term, \lemref{lem: w0} states that there exists some constant $C_0$ (depending on $K_3$) such that with probability at least $1-\delta / 3$
    \begin{align}\label{eq: w0}
        \frac{1}{4\eta^2}\norm{\bw_0}^2 \geq \frac{1}{4\eta^2}\left(1 - C_0 \sqrt{\frac{\log\left(\frac{12}{\delta}\right)}{d}}\right)^2 \geq \frac{1}{16\eta^2},
    \end{align}
    where we used that $d \geq 4C_0^2 \log\left(12/\delta\right)$.
    
    Next, to bound $\sum_{j=1}^t\norm{ a_j\bx_j}^2$, we use \lemref{lem: norm_lower}, so that for some constants $C_2, C_2'>0$ (depending on $K_2, \alpha_2$, $\beta$) and any $\delta > 0$, if $d\geq C_2' \log\left(\frac{12T}{\delta}\right)$, then with probability at least $1-\delta / 3$, it holds for all $t\in[T]$ that
    \begin{align*}
        \sum_{j=1}^t\norm{ a_j\bx_j}^2
        \geq \frac{\alpha_2 d}{2} \left(\sum_{j=1}^t \E_j\left[a_j^2\right] - C_2 \beta^2 \sqrt{t \log\left(\frac{12T}{\delta}\right)}\right).      
    \end{align*}
    In particular, if $t \geq 4C_2^2 \kappa^2 \log\left(12T/\delta \right)$ and if $\min_{j \leq t}\E_j[a_j^2] \geq \beta^2/\kappa$, then there exists a constant $C_3 > 0$ such that
    \begin{align}\label{eq: sq_norm}
        \sum_{j=1}^t\norm{ a_j\bx_j}^2
        \geq & \frac{\alpha_2 \beta^2 d}{2} \left(\frac{t}{\kappa}- C_2 \sqrt{t \log\left(\frac{12T}{\delta}\right)}\right) 
        =  d t \cdot \frac{\alpha_2 \beta^2}{2 \kappa} \left(1 - C_2 \sqrt{\frac{\log\left(\frac{12T}{\delta}\right)\kappa^2 }{t} }\right) \nonumber \\ 
        \geq & C_3 \frac{dt}{\kappa}.   
    \end{align}

    We now move to $\sum_{j=1}^t \langle \bw_{j-1}, \bg_j \rangle$. For any $j\geq 1$, since $|a_j|\leq \beta$ and $\bx_j$ is $K_1$ sub-Gaussian (Assumption~\ref{ass: subgauss_inputs}), by \citet{vershynin2010introduction}[Lemma 2.7.8], the centered gradients $(\bg_j)_{j=1}^\infty$ form a sub-Gaussian MDS (Definition \ref{def: subgauss_mds}) with parameters $\lambda_j:=c_1$ for some constant $c_1 > 0$ (depending on $\beta, K_1$). 
    Furthermore by \lemref{lem: w0}, $\bg_0=\frac{1}{\eta}\bw_0$ is sub-Gaussian with parameter $\lambda_0:=c_0 /(\sqrt{d}\eta)$ (for some $c_0$ depending on $K_3$). We now wish to apply \lemref{lem: delta_bound} for $(\bg_t)_{t=0}^T$ (note that the lemma is stated with the first time being $t=1$, this is equivalent up to re-indexing). It holds that
    $\sum_{j=0}^t \lambda_j^2 = \frac{c_0^2}{d\eta^2} + c_1 t$, and $\max_{i,j \leq T} \frac{\lambda^2_i\lambda^2_j}{\lambda^2_0\lambda^2_1} = \max(\lambda_1^2 / \lambda_0^2 , 1) \leq \max(\frac{c_1^2}{c_0^2}d \eta^2, 1 )$, which is equal to $1$ when $\eta \leq c_0/ (c_1 \sqrt{d})$. 
    As such, \lemref{lem: delta_bound} states that with probability at least $1-\delta/3$, it holds for all $t\in[T]$ that for some constant $c'>0$, 
    \begin{align}\label{eq: deltas}
        \abs{\sum_{j=1}^t \langle \bw_{j-1}, \bg_j \rangle} 
        \leq & c'\left(\frac{c_0}{\sqrt{d}\eta^2} + c_1t\sqrt{d}\right) \sqrt{\log\left(\frac{6dT}{\delta}\right) \cdot \left(\log\left(\frac{6T}{\delta}\right) + \log\log\left(eT^2\right)\right)} \nonumber \\ 
        \leq& C_4\left(\frac{1}{\sqrt{d}\eta^2} + t\sqrt{d}\right) \log\left(\frac{6dT}{\delta}\right),
    \end{align}
    where $C_4$ is a suitable constant.

    So plugging \eqref{eq: w0}, \eqref{eq: sq_norm}, \eqref{eq: deltas} back into \eqref{eq: master_lower}, it holds that for a suitable $c''>0$ that may depend on $K_1,K_2,K_3,\alpha_2, \beta$, with probability at least $1-\delta$ it holds for any integer $t\in\left[c'' \kappa^2 \log\left(T/\delta \right), T\right]$, that if $\min_{j \leq t}\E_j[a_j^2] \geq \beta^2/\kappa$ then
    \begin{align}\label{eq: almost_final_norm}
        \norm{\sum_{j=0}^t a_j\bx_j}^2 
        \geq & \frac{1}{16\eta^2} + \frac{C_3}{\kappa}dt - C_4\left(\frac{1}{\sqrt{d}\eta^2} + t\sqrt{d}\right) \log\left(\frac{6dT}{\delta}\right) - 3t^2 \max_{j \leq t}\norm{\bmu_j}^2 \nonumber\\ 
        = & \frac{1}{16\eta^2}\left(1 - \frac{16C_4\log\left(\frac{6dT}{\delta}\right)}{\sqrt{d}}\right) + \frac{C_3}{\kappa}dt \left(1 - \frac{C_4 \kappa \log\left(\frac{6dT}{\delta}\right)}{C_3\sqrt{d}}\right) 
        - 3t^2 \max_{j \leq t}\norm{\bmu_j}^2 \nonumber\\ 
        \geq& C\left(\frac{1}{\eta^2} + \frac{dt}{\kappa}\right) - 3t^2 \max_{j \leq t}\norm{\bmu_j}^2,
    \end{align}
    where we used that $d \geq c \kappa^2\log^2\left(\frac{6dT}{\delta}\right)$ and $C, c >0$ are suitable constants that may depend on $K_1, K_2, K_3, \alpha_2, \beta$.     

    Finally, using the inequalities $\sqrt{u-v} \geq \sqrt{u} - \sqrt{v}$ and $u^2 + v^2 \geq (u+v)^2/2$ for any $u,v \geq 0$
    we obtain
    \begin{align*}
        \norm{\sum_{j=0}^t a_j\bx_j}
        \geq & \sqrt{C\left(\frac{1}{\eta^2} + \frac{dt}{\kappa}\right)} - \sqrt{3t^2 \max_{j \leq t}\norm{\bmu_j}^2} \\ 
        \geq & \sqrt{\frac{C}{2}}\left(\frac{1}{\eta} + \sqrt{\frac{dt}{\kappa}}\right) - 2t\max_{j\leq t}\|\bmu_j\|.
    \end{align*}
    Recalling that $\sum_{j=0}^t a_j\bx_j = \frac{1}{\eta}\bw_0 + \sum_{j=1}^t a_j\bx_j$ and adjusting constants completes the proof. Note that up to replacing $c$ by a suitable larger constant, the condition on $d$ is equivalent to $d \geq c \kappa^2\log^2\left(\frac{T}{\delta}\right)$ (which is how it's written in the lemma statement).
\end{proof}

\subsection{Covering Arguments}
Our proof will utilize relatively standard covering arguments. Given $\epsilon \in (0,1)$, an $\epsilon$-net
$\Ncal$ of $\Sphere^{m-1}$ is a finite subset of the unit sphere $\Sphere^{m-1}$ such that for every
$\bv \in \Sphere^{m-1}$ there exists $\bv' \in \mathcal N$ satisfying $\|\bv - \bv'\|_2 \leq \epsilon$.

A classical result in high-dimensional geometry (see, e.g., \citep{vershynin2025high}[Corollary 4.2.11]) states that one
can construct such a net with cardinality
\[
|\mathcal N| \leq \left(1 + \frac{2}{\epsilon}\right)^m .
\]
As a consequence, uniform control over all directions in $\Sphere^{m-1}$ can be reduced to control over
the finite set $\mathcal N$, at the cost of an $\epsilon$-dependent slack. So we can derive bounds that hold uniformly over all directions by applying
union bounds over the net $\mathcal N$, and this will be used to extend the alignment control results from the single-neuron case to high-width settings.

\begin{lemma}\label{lem: covering}
    For any $\epsilon \geq 0$, there exists a set $\Ncal \subseteq \Sphere^{m-1}$ of cardinality at most $(3/\epsilon)^m$ such that for any $t \geq 0$, and nonzero vector $\bu \in \R^d$,
    \begin{align*}
        \frac{\norm{P_{W_t} \bu}}{\norm{\bu}} 
        \leq \frac{\sup_{\bv \in \Ncal}\frac{\abs{\bu^\top W_t^\top \bv}}{\norm{\bu}} + \epsilon \norm{W_t}_\op}{s_m(W_t)} 
        \leq \frac{\sup_{\bv \in \Ncal}\frac{\abs{\bu^\top W_t^\top \bv}}{\norm{\bu}} + \epsilon \norm{W_t}_\op}{\inf_{\bv \in \Ncal} \norm{W_t^\top \bv} - \epsilon\norm{W_t}_\op},
    \end{align*}
    where the last inequality assumes the denominator is non-negative. 
\end{lemma}
\begin{proof}
    We will apply a covering argument. Let $\Ncal$ be an $\epsilon$-net of the sphere $\Sphere^{m-1}$ with $\epsilon\in (0, 1/2]$. By \citep{vershynin2025high}[Corollary 4.2.11], we may take $\Ncal$ to have cardinality at most $\abs{\Ncal} \leq \left(3/\epsilon\right)^m$. 
    
    Fix any $t\geq 0$, and note that $\norm{P_{W_t} \bu} / \norm{\bu}$ can be rewritten as
    \begin{align}\label{eq: corr_u_cover}
        \norm{P_{W_t} \bu} = \sup_{\bv \in \Sphere^{m-1}, \bv^\top W_t \neq \zero} \frac{\abs{\bu^\top W_t^\top \bv}}{\norm{W_t^\top \bv} \norm{\bu}} \leq \sup_{\bv \in \Sphere^{m-1}} \frac{\abs{\bu^\top W_t^\top \bv}}{s_m(W_t) \norm{\bu}},
    \end{align}
    
    Now we may bound $s_m(W_t)$ as follows:
    \begin{align*}
        s_m(W_t) 
        = & \inf_{\bv'\in \Sphere^{m-1}} \norm{W_t^\top \bv'} \geq \inf_{\bv'\in\Sphere^{m-1}} \inf_{\bv \in \Ncal} \norm{W_t^\top \bv} - \norm{W_t^\top}_\op\norm{\bv - \bv'} \\ 
        \geq & \inf_{\bv \in \Ncal} \norm{W_t^\top \bv} - \epsilon\norm{W_t}_\op 
        .
    \end{align*}
    
    Similarly, we may bound $\sup_{\bv \in \Sphere^{m-1}}\frac{\abs{\bu^\top W_t^\top \bv}}{\norm{\bu}}$ as
    \begin{align*}
        \sup_{\bv \in \Sphere^{m-1}}\frac{\abs{\bu^\top W_t^\top \bv}}{\norm{\bu}} 
        \leq & \sup_{\bv'\in\Sphere^{m-1}} \sup_{\bv \in \Ncal}\frac{\abs{\bu^\top W_t^\top \bv}}{\norm{\bu}} + \frac{\norm{\bu^\top W_t^\top}}{\norm{\bu}}\norm{\bv - \bv'} \\ 
        \leq & \sup_{\bv \in \Ncal}\frac{\abs{\bu^\top W_t^\top \bv}}{\norm{\bu}} + \epsilon \norm{W_t}_\op.
    \end{align*}

    Plugging these inequalities back into \eqref{eq: corr_u_cover} completes the proof.
\end{proof}

\subsection{Proof of the Theorem}
We are now ready to prove the theorem. We begin with the following proposition, which combines results from the previous lemmas and does most of the work:

\begin{proposition}\label{prop: snr}
    Under Assumptions \ref{ass: inputs_all}, \ref{ass: init}, \ref{ass: bounded}, let $\delta > 0$, $\kappa \geq 1$, $T\in \N$ and $\bu \in \R^d\setminus \{0\}$. There exist constants $C, c, c' > 0$ that may depend on $K_1, K_2, K_3, \alpha_2, G$, such that if $d \geq c \kappa^2 m^2 \log\left(\frac{Td}{\delta}\right)^2$, $\eta \leq \frac{c'}{\kappa^2\sqrt{md \log\left(Td/\delta'\right)}}$ then conditioned on $\kappaT \leq \kappa$, with probability at least $1-\delta$, it holds for all $t\in[T] \cup\{0\}$ that
    \begin{align*}
        \frac{\|P_{W_t}\bu\|}{\|\bu\|} \leq C \frac{\sqrt{\frac{m\log\left(\frac{Td}{\delta}\right)}{d}} + \frac{\eta t \max_{j \leq t}\norm{\nabla_W \Lcal(\theta_{j-1})}_F}{1 + \eta\sqrt{dt}}}{\frac{1}{s\sqrt{\kappa}} - \frac{\eta t \max_{j \leq t}\norm{\nabla_W \Lcal(\theta_{j-1})}_F}{ 1 + \eta\sqrt{dt}}},
    \end{align*}
    as long as the denominator is positive.
\end{proposition}
\begin{proof}
    Fix $\epsilon > 0$ and $\delta'>0$ which will be specified later, by Lemma \ref{lem: covering} there exists a set $\Ncal \subseteq \Sphere^{m-1}$ of cardinality at most $(3/\epsilon)^m$ such that for any $t \geq 0$, 
    \begin{align}\label{eq: corr_base}
        \frac{\norm{P_{W_t} \bu}}{\norm{\bu}} \leq \frac{\sup_{\bv \in \Ncal}\frac{\abs{\bu^\top W_t^\top \bv}}{\norm{\bu}} + \epsilon \norm{W_t}_\op}{\inf_{\bv \in \Ncal} \norm{W_t^\top \bv} - \epsilon\norm{W_t}_\op}.
    \end{align}

    For any $t\geq 1$, the SGD weight updates for $W_t$ are given by 
    \begin{align*}
        W_t - W_{t-1} = -\eta \nabla_{W_{t-1}} \ell\left(\theta_{t-1} ; \bx_t\right) = -\eta \nabla_{W_{t-1}\bx_t} \ell\left(\theta_{t-1} ; \bx_t\right) \bx_t^\top.
    \end{align*}
    So for any unit vector $\bv$, letting $a_t(\bv):=
    \left\langle \nabla_{W_{t-1}\bx_t} \ell\left(\theta_{t-1} ; \bx_t\right), \bv \right\rangle \in \reals$ it holds that 
    \begin{align*}
        \left(W_t - W_{t-1}\right)^\top \bv = \eta a_t(\bv) \bx_t , \qquad \text{and} \qquad \left(W_t - W_0\right)^\top \bv = \eta \sum_{j=1}^t a_j(\bv) \bx_j.
    \end{align*}

    By Assumption~\ref{ass: bounded}, the random variables $a_t(\bv)$ are uniformly bounded as 
    $\sup_{\bv \in \Sphere^{m-1}}\abs{a_t(\bv)} \leq \beta := G$. To bound \eqref{eq: corr_base}, we will need to bound a few different terms: 
    
    \paragraph{Upper bounding $\norm{W_t}_{\op}$ .} 
        First note, that by \lemref{lem: op_w0}, for some universal constant $C_0 > 0$ it holds with probability at least $1-\delta'$ that
        \begin{align}\label{eq: sp_w0}
            \forall i \in [m], \abs{s_i\left(W_0\right) - 1} \leq C_0K_3^2 \left(\sqrt{\frac{m}{d}} + \sqrt{\frac{\log\left(\frac{2}{\delta'}\right)}{d}}\right) \leq \frac{1}{2},
        \end{align}
        where the last inequality holds whenever $d\geq (1 + \frac{1}{C_0^2 K_3^4})\max\left(m, \log\left(2/\delta'\right)\right)$.
        Moreover, by applying Lemma \ref{lem: norm_step1} to bound $\norm{\bx_t}$ for every $t\in [T]$, when $d\geq \max\left(\frac{16 K_2^4}{\alpha_2^2} , \frac{4K_2^2}{\alpha_2}\right) \log\left(\frac{2T}{\delta'}\right)$ it follows that with probability at least $1-\delta'$, for all $t\in [T]$, 
        \begin{align}\label{eq: norm_drift}
            \norm{W_t - W_0}_F \leq & \sqrt{m}\max_{i\in [m]}\norm{\left(W_t - W_0\right)^\top \be_i} = \eta\sqrt{m}\max_{i\in [m]} \norm{\sum_{j=1}^t a_j(\be_i) \bx_j} \nonumber\\ 
            \leq & \eta\sqrt{m}\max_{i\in [m]} \sum_{j=1}^t \abs{a_j(\be_i)}\norm{\bx_j} 
            \leq \beta \eta t \sqrt{\frac{3 \alpha_2 dm}{2}}.
        \end{align}

        Combining these and using that $\eta \leq 1/\sqrt{md}$ we obtain that for a suitable constant $C_1>0$ (that may depend on $K_3, \beta, \alpha_2$) it holds with probability at least $1-2\delta'$ that for all $t\in[T]$,
        \begin{align}\label{eq: op_wt}
            \norm{W_t}_{\op} \leq \norm{W_0}_\op + \norm{W_t - W_0}_F \leq C_1 t.
        \end{align}
    
    \paragraph{Upper Bounding $\sup_{\bv \in \Ncal}\frac{\abs{\bu^\top W_t^\top \bv}}{\norm{\bu}}$ . }
    
    Consider for now a fixed $\bv \in \Ncal$, and define the conditional drift and centered increments
    \[
    \bg_t := a_t(\bv)\bx_t , \qquad \bmu_t := \E_t[\bg_t] = -\nabla_W \Lcal(\theta_{t-1})^\top \bv ,\qquad \bar\bg_t := \bg_t-\bmu_t.
    \]
    To apply \lemref{lem: bound_signal}, first note that the centered gradients $(\bar \bg_j)_{j=1}^\infty$ form a sub-Gaussian MDS (Definition \ref{def: subgauss_mds}) with parameters $\lambda_j \leq c_2 \beta K_1$ for some universal constant $c_2 > 0$ (since $|a_j|\leq \beta$ and $\bx_j$ is $K_1$ sub-Gaussian). So it holds with probability at least $1 - \delta' / \abs{\Ncal}$ that for all $t\in[T]$,
        \begin{align}\label{eq: num_bound_helper}
            \frac{\eta}{\norm{\bu}} \abs{\langle (W_t - W_0)^\top \bv, \bu \rangle} 
            = & \frac{\eta}{\norm{\bu}} \abs{\left\langle \sum_{j=1}^t \bg_j, \bu \right\rangle}
            \leq \frac{\eta}{\norm{\bu}} \left(\abs{\sum_{j=1}^t \langle \bmu_j, \bu \rangle} + \abs{\sum_{j=1}^t \langle \bar \bg_j, \bu \rangle}\right) \nonumber \\ 
            \leq & \eta\sum_{j=1}^t\norm{\bmu_j} + \frac{\eta}{\norm{\bu}} \abs{\sum_{j=1}^t \langle \bar \bg_j, \bu \rangle} \nonumber\\ 
            \leq & \eta t \max_{j \leq t}\norm{\bmu_j} + \eta c_3K_1\beta \sqrt{t \log\left(\frac{2T \abs{\Ncal}}{\delta'}\right)} ~. 
        \end{align}

    We will now extend the bound in \eqref{eq: num_bound_helper}, which holds for a fixed $\bv$, to all $\bv\in\Ncal$ using a union bound. Since $\abs{\Ncal} \leq (3/\epsilon)^m$ we obtain that with probability at least $1-\delta'$
    \begin{align*}
        \sup_{\bv \in \Ncal}\frac{\abs{\bu^\top \left(W_t - W_0\right)^\top \bv}}{\norm{\bu}} \leq \eta t \max_{j \leq t}\norm{\nabla_W \Lcal(\theta_{j-1})}_F + \eta c_3K_1\beta\sqrt{t m \log\left(\frac{6T}{\epsilon \delta'}\right)}.
    \end{align*}

    Moreover, by \lemref{lem: w0}, and the union bound for some constant $c_3'>0$, it holds with probability at least $1-\delta'$ that for all $\bv \in \Ncal$,
    \begin{align*}
        \sup_{\bv \in \Ncal} \frac{\abs{\bu^\top W_0^\top \bv}}{\norm{\bu}} \leq c_3' K_3 \sqrt{\frac{m\log\left(\frac{4}{\epsilon\delta'}\right)}{d}}.
    \end{align*}
    
    So by combining the previous equations, we obtain that for a suitable constant $C_3>0$, it holds with probability at least $1-2\delta'$ that for all $t\in[T]$,
    \begin{align}\label{eq: num_bound}
        \sup_{\bv \in \Ncal} \frac{\abs{\bu^\top W_t^\top \bv}}{\norm{\bu}} 
        = & \sup_{\bv \in \Ncal} \frac{\abs{\bu^\top W_0^\top \bv}}{\norm{\bu}} + \sup_{\bv \in \Ncal} \frac{\abs{\bu^\top (W_t - W_0)^\top \bv}}{\norm{\bu}} \nonumber\\ 
        \leq & C_3 \sqrt{m\log\left(\frac{T}{\epsilon \delta'}\right)} \left(\frac{1}{\sqrt{d}} + \eta\sqrt{t}\right) + \eta t \max_{j \leq t}\norm{\nabla_W \Lcal(\theta_{j-1})}_F.
    \end{align}

    \paragraph{Lower bounding $\inf_{\bv \in \Ncal} \norm{W_t^\top \bv}$ .}
    
    We apply \lemref{lem: lower_norm} to $\norm{W_t^\top \bv}$ for every $\bv\in\Ncal$ (in the notation of the lemma, $\bw_0= W_0^\top \bv$, and $a_t$ are given by $a_t(\bv)$) and take a union bound. 
    So for some constants $C_4, c_4, c_4', c_4''>0$ that may depend on $K_1, K_2, K_3, \alpha_2, \beta$, since $d \geq c_4 \kappa^2m^2 \log\left(\frac{T}{\epsilon \delta'}\right)^2$, if $\eta \leq c_4'/\sqrt{d}$ and  $\min_{t\leq T}\E_t[a_t(\bv)^2]\geq \beta^2 / \kappa$ for all $\bv\in\Ncal$, then with probability at least $1-\delta'$ it holds for all $t \in \left[\left\lceil c_4'' \kappa^2 m\log\left(\frac{T}{\epsilon\delta'}\right)\right\rceil ~,~ T\right]$ that 
    \begin{align}\label{eq: den_bound}
            \inf_{\bv \in \Ncal} \norm{W_t^\top \bv} \geq C_4\left(1 + \eta\sqrt{\frac{dt}{\kappa}}\right) - 2\eta t \max_{j \leq t}\norm{\nabla_W \Lcal(\theta_{j-1})}_F.
    \end{align}

    \paragraph{Putting everything together.}
    
    Let $\Ecal$ be the event where all of the above hold (specifically, \eqref{eq: op_wt}, \eqref{eq: num_bound}, \eqref{eq: den_bound}). Taking $\delta'=\delta/5$ suffices to ensure that this occurs with probability at least $1-\delta$. Also, take $\epsilon = \frac{1}{C_1 T \sqrt{d}}$. We split the proof into the early iterations and the later iterations. First, for the later iterations, when $\Ecal$ indeed occurs, then \eqref{eq: op_wt}, \eqref{eq: num_bound}, \eqref{eq: den_bound} imply that for all $t \in \left[\left\lceil c_4'' \kappa^2 \log\left(\frac{C_1T^2 \sqrt{d}}{\delta'}\right) \right\rceil, T\right]$,
    \begin{align}\label{eq: snr_bound}
        \frac{\norm{P_{W_t} \bu}}{\norm{\bu}} 
        \leq & \frac{\sup_{\bv \in \Ncal}\frac{\abs{\bu^\top W_t^\top \bv}}{\norm{\bu}} + \epsilon \norm{W_t}_\op}{\inf_{\bv \in \Ncal} \norm{W_t^\top \bv} - \epsilon\norm{W_t}_\op} \nonumber\\ 
        \leq & \frac{C_3 \sqrt{m\log\left(\frac{T}{\epsilon \delta'}\right)} \left(\frac{1}{\sqrt{d}} + \eta\sqrt{t}\right) + \eta t \max_{j \leq t}\norm{\nabla_W \Lcal(\theta_{j-1})}_F + \epsilon \cdot C_1t}{ C_4\left(1 + \eta\sqrt{\frac{dt}{\kappa}}\right) - 2\eta t \max_{j \leq t}\norm{\nabla_W \Lcal(\theta_{j-1})}_F - \epsilon \cdot C_1t} \nonumber\\ 
        \leq&_{(*)} C \frac{\sqrt{m\log\left(\frac{Td}{\delta}\right)} \left(\frac{1}{\sqrt{d}} + \eta\sqrt{t}\right) + \eta t \max_{j \leq t}\norm{\nabla_W \Lcal(\theta_{j-1})}_F}{\frac{1}{2\sqrt{\kappa}}\left(1 + \eta\sqrt{dt}\right) - \eta t \max_{j \leq t}\norm{\nabla_W \Lcal(\theta_{j-1})}_F} \nonumber\\ 
        =& C \frac{\sqrt{\frac{m\log\left(\frac{Td}{\delta}\right)}{d}} + \frac{\eta t \max_{j \leq t}\norm{\nabla_W \Lcal(\theta_{j-1})}_F}{1 + \eta\sqrt{dt}}}{\frac{1}{2\sqrt{\kappa}} - \frac{\eta t \max_{j \leq t}\norm{\nabla_W \Lcal(\theta_{j-1})}_F}{ 1 + \eta\sqrt{dt}}},
    \end{align} 
    where $C>0$ is a suitable constant, and in (*) we used that $\kappa \geq 1$. We note that \eqref{eq: snr_bound} assumes that $\frac{\eta t \max_{j \leq t}\norm{\nabla_W \Lcal(\theta_{j-1})}_F}{ 1 + \eta\sqrt{dt}} < \frac{1}{2\sqrt{\kappa}}$ (so that the denominator is well defined).
    
    We now prove the desired result for $t \leq \left\lceil c_4'' \kappa^2 \log\left(\frac{C_1T^2 \sqrt{d}}{\delta'}\right) \right\rceil$. Taking $\eta \leq \frac{c'}{\kappa^2\sqrt{md \log\left(Td/\delta'\right)}}$ for a suitable constant $c'>0$ suffices to ensure that for such values of $t$, $\eta \sqrt{dt}\leq 1/8$. Moreover, under $\Ecal$, \eqref{eq: norm_drift}, $\norm{W_t - W_0}_F \leq 1/4$. As such, by \eqref{eq: sp_w0},
    \begin{align*}
        s_m \left(W_t\right) \geq s_m\left(W_0\right) - \norm{W_t - W_0}_F \geq \frac{1}{4} 
        \geq \frac{1}{8}\left(1 + \eta \sqrt{\frac{dt}{\kappa}}\right) - 2\eta t \max_{j \leq t}\norm{\nabla_W \Lcal(\theta_{j-1})}_F.
    \end{align*}

    So using the first inequality in of \lemref{lem: covering}, under $\Ecal$ it holds for all $t \leq \left\lceil c_4'' \kappa^2 \log\left(\frac{T}{\delta'}\right) \right\rceil$ that
    \begin{align*}
        \frac{\norm{P_{W_t} \bu}}{\norm{\bu}} 
        \leq & \frac{\sup_{\bv \in \Ncal}\frac{\abs{\bu^\top W_t^\top \bv}}{\norm{\bu}} + \epsilon \norm{W_t}_\op}{s_m(W_t)} 
        \leq  \frac{\sup_{\bv \in \Ncal}\frac{\abs{\bu^\top W_t^\top \bv}}{\norm{\bu}} + \epsilon \norm{W_t}_\op}{\frac{1}{8}\left(1 + \eta \sqrt{\frac{dt}{\kappa}}\right) - 2\eta t \max_{j \leq t}\norm{\nabla_W \Lcal(\theta_{j-1})}_F}
        \\
        &\stackrel{(*)}{\leq} C \frac{\sqrt{\frac{m\log\left(\frac{Td}{\delta}\right)}{d}} + \frac{\eta t \max_{j \leq t}\norm{\nabla_W \Lcal(\theta_{j-1})}_F}{1 + \eta\sqrt{dt}}}{ \frac{1}{2\sqrt{\kappa}} - \frac{\eta t \max_{j \leq t}\norm{\nabla_W \Lcal(\theta_{j-1})}_F}{1 + \eta\sqrt{dt}}},
    \end{align*}
    where (*) is as in \eqref{eq: snr_bound}, and $C$ is possibly replaced by a larger constant where needed.
\end{proof}

\main*
\begin{proof}
    First, let $U=OSV^\top$ be the SVD decomposition of $U$, where $O$ and $V$ are orthogonal matrices, and $S$ is a $p\times d$ rectangular diagonal matrix. Let $q$ be the rank of $U$ and note that $S^\dagger S$ is a $d\times d$ diagonal matrix whose first $q$ diagonal entries are $1$ and the rest are $0$. Let $\tilde \bu_1,\ldots, \tilde \bu_q$ be the first $q$ columns of the matrix $V$.
    Denoting by $\dagger$ the (Moore-Penrose) pseudoinverse of a matrix and letting $\rho_t := \norm{P_{W_t} P_U}_\op$, it holds that
    \begin{align*}
        \rho_t := & \norm{P_{W_t} P_U}_\op 
        = \norm{P_{W_t} U^\dagger U}_\op 
        = \norm{P_{W_t} V S^\dagger S V^\top}_\op \\ 
        = & \norm{P_{W_t} V S^\dagger S}_\op \leq \sqrt{q} \max_{i \in [q]} \norm{P_{W_t} \tilde \bu_i}
    \end{align*}
    Let $\mu_j:=\norm{\nabla_{W} \Lcal(\theta_{j-1})}_F$, then by applying \propref{prop: snr} (whose assumptions hold under the assumptions of this theorem) to each $\tilde \bu_i$ (who have unit norm) and applying a union bound, then conditioned on $\kappaT \leq \bar\kappa$, there is an event $\Ecal$ that holds with probability at least $1-\delta$, such that for some $C>0$ that may depend on $K_1, K_2, K_3, \alpha_2, G$,
    it holds for all $t\in[T]\cup\{0\}$, that 
    \begin{align}\label{eq: r_t}
        \rho_t \leq \sqrt{p} \cdot \max_{i\in [q]}\frac{\norm{P_{W_t} \tilde \bu_i}}{\norm{\tilde \bu_i}} \leq \frac{1}{8}C \sqrt{p} \frac{\sqrt{\frac{m\log\left(\frac{Tdp}{\delta}\right)}{d}} + \frac{\eta t \max_{j \leq t}\mu_j}{1 + \eta\sqrt{dt}}}{\frac{1}{2\sqrt{\bar\kappa}} - \frac{\eta t \max_{j \leq t}\mu_j}{ 1 + \eta\sqrt{dt}}},
    \end{align}
    as long as the denominator is positive.
    The remainder of this proof assumes that this event $\Ecal$ indeed occurs. 
    Now note that for any $t$ such that 
    $\max_{j \leq t}\mu_j \leq \frac{1 + \eta \sqrt{dt}}{\eta t \sqrt{d}}$ 
    it holds that 
    \begin{align}\label{eq: condition_t}
        \rho_t \leq \frac{1}{8}C \sqrt{p} \frac{\sqrt{\frac{m\log\left(\frac{Tdp}{\delta}\right)}{d}} + \frac{1}{\sqrt{d}}}{\frac{1}{2\sqrt{\bar \kappa}} - \frac{1}{\sqrt{d}}} 
        \leq_{(*)}  \frac{1}{8}C \sqrt{p} \cdot \frac{2\sqrt{\frac{m\log\left(\frac{Tdp}{\delta}\right)}{d}}}{\frac{1}{4\sqrt{\bar \kappa}}}  
        \leq C \sqrt{\frac{\bar\kappa mp\log\left(\frac{Tdp}{\delta}\right)}{d}},
    \end{align}
    where $(*)$ uses that $d \geq 16\bar \kappa$. 
    Furthermore, by assumption we know that $\mu_{t+1} \leq \psi\left(\rho_t\right)$, so in such a case, 
    \[
    \mu_{t+1} \leq \psi\left(C \sqrt{\frac{\bar \kappa mp\log\left(\frac{Tdp}{\delta}\right)}{d}}\right).
    \]
    Now consider $T\in\N$ such that 
    \[
        T \leq \frac{1}{\psi\left(C \sqrt{\frac{\bar\kappa mp \log\left(Tdp/\delta\right)}{d}} \right)^2}~,
    \]
    We now show that this directly implies by induction on $t$ that for all $t\leq T$, 
    \begin{align*}
        \rho_t \leq C \sqrt{\frac{\bar \kappa mp\log\left(\frac{Tdp}{\delta}\right)}{d}}.
    \end{align*}
    \begin{itemize}
        \item The base case is trivial by plugging $t=0$ into \eqref{eq: r_t}.
        \item Induction Step: consider $t \geq 1$ so that by the induction assumption it holds that for all $j\leq t-1, \rho_j \leq C \sqrt{\frac{\bar \kappa mp\log\left(\frac{Tdp}{\delta}\right)}{d}}$. Then
        \begin{align*}
            \max_{j\leq t}\mu_{j} \leq \max_{j\leq t-1}\psi\left(\rho_j\right) \leq \psi\left(C \sqrt{\frac{\bar \kappa mp\log\left(\frac{Tdp}{\delta}\right)}{d}}\right),
        \end{align*}
        where we used here that $\psi$ is increasing.
        By the definition of $T$, and the fact that $t\leq T$, this implies that $\max_{j \leq t}\mu_j \leq \frac{1}{\sqrt{T}}\leq \frac{1+\eta \sqrt{dt}}{\eta t \sqrt{d}}$ and thus by \eqref{eq: condition_t}, $\rho_{t+1} \leq C \sqrt{\frac{\bar \kappa mp\log\left(\frac{Tdp}{\delta}\right)}{d}}$ as required.
    \end{itemize}
\end{proof}

\subsection{Alternative \texorpdfstring{$\Omega(d)$}{Omega(d)} Bound}
The bound provided in \thmref{thm: main} will be useful in cases where the population gradients are very small. Nevertheless, in our setting, one can obtain an $\Omega(d)$ bound which is generally applicable. Specifically, this will be used later in the proof of \thmref{thm: mi_general} to handle situations when $\norm{\nabla_W\Lcal}_F$ is on the order of a constant.

\begin{proposition}\label{prop: d_bound}
    Under Assumptions \ref{ass: inputs_all}-\ref{ass: bounded}, let $\delta > 0$, $\epsilon \in (0, 1/2)$, $\bar\kappa \geq 1$. Then there exist constants $c, c', C_1 > 0$ (that may depend on $K_1, K_2, K_3, \alpha_2, G$), such that if $d\geq c \bar \kappa^2 p^2 m^2 \epsilon^{-2} \log\left(\frac{d^2}{\epsilon^2\delta}\right)^2$, $\eta \leq \frac{c'}{\bar \kappa^2\sqrt{md \log\left(d^2/\epsilon^2\delta\right)}}$ and $T \leq \frac{d\epsilon^2}{p C_1}$, the conditioned on $\kappaT \leq \bar\kappa$ it holds with probability at least $1-\delta$ (under the same event as \thmref{thm: main}), that for all $t\in[T]$
        \begin{align*}
            \norm{P_{W_t} P_U}_\op \leq \epsilon~.
        \end{align*}
\end{proposition}
\begin{proof}
    Fix $T\in\N$ to be determined later and let $\mu_j:=\norm{\nabla_{W} \Lcal(\theta_{j-1})}_F$. We continue from \eqref{eq: r_t} within the proof of \thmref{thm: main}, which states that conditioned on $\kappaT \leq \bar\kappa$, there is an event $\Ecal$ that holds with probability at least $1-\delta$, such that for some $C>0$ that may depend on $K_1, K_2, K_3, \alpha_2, G$,
    it holds for all $t\in[T]\cup\{0\}$, that 
    \begin{align}\label{eq: r_t2}
        \rho_t := \norm{P_{W_t} P_U}_\op 
        \leq C \sqrt{p} \frac{\sqrt{\frac{m\log\left(\frac{Tdp}{\delta}\right)}{d}} + \frac{\eta t \max_{j \leq t}\mu_j}{1 + \eta\sqrt{dt}}}{\frac{1}{2\sqrt{\bar \kappa}} - \frac{\eta t \max_{j \leq t}\mu_j}{ 1 + \eta\sqrt{dt}}},
    \end{align}
    as long as the denominator is positive.
    The remainder of this proof assumes that this event $\Ecal$ indeed occurs. 
    We now upper bound $\mu_j$. First, observe that 
    \begin{align*}
    \norm{\nabla\L(\theta_j)}_F^2 = & \norm{\E_\bx\left[\nabla_{W_{j-1}\bx} \ell\left(\theta_{j-1} ; \bx\right) \bx^\top\right]}_F^2 = \sum_{i=1}^m \norm{\E_\bx\left[\nabla_{W_{j-1}\bx} \ell\left(\theta_{j-1} ; \bx\right)_i \bx\right]}^2 \\ 
    = & \sum_{i=1}^m \sup_{\bv_i \in \Sphere^{d-1}}\left(\E_\bx\left[\nabla_{W_{j-1}\bx} \ell\left(\theta_{j-1} ; \bx\right)_i \bx^\top \bv_i\right]\right)^2 \\ 
    \leq & \sum_{i=1}^m \sup_{\bv_i \in \Sphere^{d-1}}\E_\bx\left[\left(\nabla_{W_{j-1}\bx} \ell\left(\theta_{j-1} ; \bx\right)_i\right)^2 \right] \E\left[(\bx^\top \bv_i)^2\right] \\
    = & \left(\sum_{i=1}^m \E_\bx\left[\left(\nabla_{W_{j-1}\bx} \ell\left(\theta_{j-1} ; \bx\right)_i\right)^2 \right]\right) \cdot \sup_{\bv \in \Sphere^{d-1}} \E_\bx\left[(\bx^\top \bv)^2\right],
    \end{align*}
    where $(*)$ follows from apply Cauchy-Schwarz to the expected value (not to the sum). Assumption \ref{ass: bounded} implies that the first term is upper bounded by $G$. Moreover, Assumption \ref{ass: init} implies that the second term is upper bounded by $C_1 K_1$ for some absolute constant $C_1>0$ (since the sub-Gaussian norm upper bounds the variance, see \cite{vershynin2025high}). So overall, we obtain
    \begin{align*}
        \mu_j \leq C_1 K_1 G.
    \end{align*}
    As such, \eqref{eq: r_t2} yields
    \begin{align*}
        \rho_t \leq C \sqrt{p} \frac{\sqrt{\frac{m\log\left(\frac{Tdp}{\delta}\right)}{d}} + \frac{C_1 K_1 G\eta t}{1 + \eta\sqrt{dt}}}{\frac{1}{2\sqrt{\bar \kappa}} - \frac{C_1 K_1 G\eta t}{ 1 + \eta\sqrt{dt}}} 
        \leq C \sqrt{p} \frac{\sqrt{\frac{m\log\left(\frac{Tdp}{\delta}\right)}{d}} + \frac{C_1 K_1 G \sqrt{T}}{\sqrt{d}}}{\frac{1}{2\sqrt{\bar \kappa}} - \frac{C_1 K_1 G \sqrt{T}}{\sqrt{d}}},
    \end{align*}
    where the last inequality follows by lower bounding $1+\eta \sqrt{dt}$ with $\eta \sqrt{dt}$, and upper bounding $\sqrt{t}$ with $\sqrt{T}$.
    Now take $T = \left\lfloor \frac{d \epsilon^2}{4p \max\left(\bar \kappa(C_1K_1G)^2, 1\right)} \right\rfloor$, we obtain
    \begin{align*}
        \rho_t \leq C \sqrt{p} \frac{\sqrt{\frac{m\log\left(\frac{d^2}{\epsilon^2\delta}\right)}{d}} + \frac{\epsilon}{2\sqrt{\bar \kappa p}}}{\frac{1}{2\sqrt{\bar \kappa}} - \frac{\epsilon}{2\sqrt{\bar \kappa p}}} 
        = \frac{2C}{\sqrt{p} - \epsilon} \left(\sqrt{\frac{\bar \kappa p^2 m\log\left(\frac{d^2}{\epsilon^2\delta}\right)}{d}} + \frac{\epsilon}{2}\right).
    \end{align*}
    By the assumed lower bound on $d$ in the proposition statement, the square root term is at most $\epsilon/2$. Moreover, since $\epsilon \in (0, 1/2)$, $\sqrt{p} - \epsilon > 1/2$. So the alignment is bounded as $\rho_t \leq 4C \epsilon$. If $4C>1$, we may replace $\epsilon$ by $\epsilon/4C$ (and adjust the absolute constants in the proposition statement) to obtain the desired result. 
\end{proof}

\section{Hermite Polynomials and the Structure of Gaussian Multi-Index Models}\label{sec: hermite}
Throughout this section, we assume that the inputs $\bx$ are standard Gaussians $\Ncal(\zero, I_d)$. Let $\mu(\bx)$ be the PDF of the standard Gaussian in $d$ dimensions, and $H_k(\bx) := (-1)^k \frac{\nabla^k \mu(\bx)}{\sqrt{k!} \mu(\bx)}$ denote the normalized (probabilist's) Hermite polynomials. We will not specify the dimension $d$, as the dimension should be inferred by the dimension of the inputs $\bx$. We refer the reader to \citet{grad1949note} for background on Hermite polynomials in high dimensions.

We let $L_2(\mu)$ denote the space of square integrable functions, with $\norm{f}_{L_2} = \E_{\bx\sim\mu}\left[f(\bx)^2\right]^{1/2}$. We may also extend this to vector valued functions as $\norm{f}_{L_2} = \E_{\bx\sim\mu}\bigl[\norm{f(\bx)}^2\bigr]^{1/2}$. We denote the $L_\infty$ norm by $\norm{f}_{L_\infty} := \sup_{\bx \in \R^d} \abs{f(\bx)}$ (the $\sup$ can be taken with probability $1$). To ease notation in the proofs, we let
\begin{align*}
    \Fcal_d := \{f: \R^d \to \R \mid f\text{ is weakly differentiable and } \E[f(\bx)^2] < \infty ~,~ \E[\norm{\nabla f(\bx)}^2] < \infty\}.
\end{align*}

\paragraph{Tensors and basic operations.}
For $k\geq1$, we identify a (real) \emph{$k$-tensor over $\R^d$} with an array
\[
A = (A_{i_1,\dots,i_k})_{i_1,\dots,i_k\in[d]} \in (\R^d)^{\otimes k}.
\]
We use the Frobenius inner product and norm:
\[
\langle A,B\rangle := \sum_{i_1,\dots,i_k=1}^d A_{i_1,\dots,i_k} B_{i_1,\dots,i_k},
\qquad \|A\|_F := \sqrt{\langle A,A\rangle} ~.
\]
We denote by $\sym(A)$ the symmetrization of a $k-$tensor, defined as $\sym(A):=\frac{1}{k!}\sum_{\pi \in S_k} \pi(A)$ where $S_k$ is the group of all permutations on $[k]$ and $\pi(A)$ denotes the tensor obtained by permuting indices of $A$ according to $\pi$. 

\paragraph{Tensor action of a matrix.}
Let $M\in\R^{m\times d}$ and let $A\in(\R^d)^{\otimes k}$. We define the tensor $M^{\otimes k}A \in (\R^m)^{\otimes k}$ by
\[
(M^{\otimes k}A)_{i_1,\dots,i_k} := 
\sum_{j_1,\dots,j_k=1}^d M_{i_1 j_1}\cdots M_{i_k j_k}A_{j_1,\dots,j_k}.
\]
Intuitively, $M^{\otimes k}$ applies the linear map $\bv\mapsto M\bv$ to each index of the tensor. For $k=1$ this is just matrix-vector multiplication, and for $k=2$ this is $M A M^\top$. This operation is linear in $A$, and satisfies the composition rule
\[
(MN)^{\otimes k}A = M^{\otimes k}(N^{\otimes k}A),
\qquad \forall M\in\R^{m\times p}, N\in\R^{p \times d}.
\]
When $M\in\R^{d\times d}$ is square and symmetric, the operator $M^{\otimes k}$
is self-adjoint with respect to the Frobenius inner product on
$(\R^{d})^{\otimes k}$, meaning that for all $k$-tensors $A,B\in(\R^{d})^{\otimes k}$,
\[
\langle A, M^{\otimes k}B\rangle
=
\langle M^{\otimes k}A, B\rangle.
\]
We also denote by $\norm{M^{\otimes k}}_{\op}$ the induced operator norm with respect to the Frobenius norm. It is well known that $\norm{M^{\otimes k}}_{\op} = \norm{M}_{\op}^k$.

\paragraph{Basic Facts for Hermite Polynomials.}

Some known facts include:
\begin{itemize}
    \item  \textbf{Hermite Expansion}: If $f\in L_2(\mu)$, then
    \begin{align}\label{eq: expansion}
        f(\bx) = \sum_{k=0}^\infty \langle H_k(\bx), C_k \rangle \qquad \text{for} \qquad C_k:=\E[f(\bx)H_k(\bx)],
     \end{align}
     where convergence is in $L_2(\mu)$, and $C_k$ are symmetric.
     
    \item \textbf{Orthonormal Basis}: $H_k$ form a complete and orthonormal basis for $L_2(\mu)$, in the sense that for any $j$-tensor $A$ and $k$-tensor $B$:
    \begin{align}\label{eq: orth}
        \E[\langle H_j(\bx), A \rangle \langle H_k(\bx), 
        B \rangle] = \delta_{jk} \langle \sym(A), \sym(B) \rangle.
    \end{align}

     \item \textbf{Generating Function:} If $f\in L_2(\mu)$ then for any $\bv, \bx \in \R^d$,
     \begin{align}\label{eq: generating}
        \exp\left(\langle \bv,\bx\rangle-\tfrac12\|\bv\|^2\right)
        =
        \sum_{j=0}^\infty \frac{1}{\sqrt{j !}}
        \left\langle H_j(\bx), \bv^{\otimes j}\right\rangle.
     \end{align}
     This is because the left-hand side is equal to $\mu(\bx-\bv)/\mu(\bx)$, so Taylor expanding $\mu(\bx-\bv)$ and using the definition of $H_j$ gives the desired equality.

     \item \textbf{Relation to Dimension one:} Let $\bv \in \R^d$ be a vector with unit norm, then for any $\bx \in \R^d$, $k \geq 0$,
     \begin{align}\label{eq: dim_match}
         H_k\left(\bv^\top \bx\right) = \left \langle H_k(\bx), \bv^{\otimes k} \right \rangle,
     \end{align}
     where on the left-hand side we use the one-dimensional Hermite polynomial.

     \item \textbf{Recurrence Relation:}
      For any index $i \in [d]$:
      \begin{align}\label{eq: recurrence}
          x_i H_k(\bx) - \frac{\partial}{\partial x_i} H_k(\bx) = \sqrt{k+1} \left(H_{k+1}(\bx)\right)_{i},
      \end{align}
      where $(H_{k+1}(\bx))_i$ denotes the slice of the tensor with the first index fixed to $i$.

     \item \textbf{Derivative Formula:} 
     For a function smooth function $f$ that is $k$ times weakly-differentiable and if $\E[\norm{\nabla^k f(\bx)}_F^2] < \infty$ with Hermite coefficient tensors $C_k$, 
     \begin{align}
     \label{eq: coeffs_grad}
         C_k = \frac{1}{\sqrt{k!}}\E[\nabla^k f(\bx)]
     \end{align}
     We note that this property cannot be used for ReLU and Leaky ReLU. Hence, we will avoid using it. Nevertheless, it helps conceptually to understand why some of the lemmas are true.
\end{itemize}

\paragraph{Useful Lemmas.}
We note that while there is extensive literature on the Hermite polynomials in dimension $1$, there are far fewer sources for higher dimensions. It is possible that variants of a couple of the following lemmas appear in the literature (especially \lemref{lem:grad_shifts_hermite}). Nevertheless, for completeness, we prove everything we need. 

We begin with a lemma that relates to the Hermite expansion in $d$ dimensions and one dimension:
\begin{lemma}\label{lem: si_hermite}
    Let $\sigma\in \Fcal_1$, and let $b_k(\norm{\bw}) := \E_{x\sim \Ncal(0, 1)}[\sigma(\norm{\bw}x)H_k(x)]$ be the one dimensional Hermite expansion of $\sigma(\norm{\bw}x)$. Then the $d-$dimensional Hermite expansion of $f_\bw(\bx):=\sigma(\bw^\top \bx)$ is given by 
    \begin{align*}
        f_\bw(\bx) = \sum_{k=0}^\infty \left\langle H_k(\bx), B_k(\bw) \right \rangle , \qquad \text{where} \qquad B_k(\bw) =  b_k(\norm{\bw})\left(\frac{\bw}{\norm{\bw}}\right)^{\otimes k}. 
    \end{align*}
\end{lemma}
\begin{proof}
    For any $s>0$, consider the Hermite expansion of $\sigma(s x)$, with coefficients given by $b_k(s):=\E_{x\sim \Ncal(0, 1)}[\sigma(s x)H_k(x)]$. Following \eqref{eq: dim_match}, the one dimensional Hermite polynomials satisfy $H_k\left(\frac{~\bw^\top}{\norm{\bw}} \bx\right) = \left \langle H_k\left(\bx\right), \left(\frac{\bw}{\norm{\bw}}\right)^{\otimes k} \right\rangle$. Using this and the one-dimensional Hermite expansion of the function $\sigma(sx)$,
    \begin{align*}
        f_\bw(\bx) =& \sigma\left(\norm{\bw} \cdot \frac{~\bw^\top}{\norm{\bw}} \bx\right) 
        = \sum_{k=0}^\infty b_k(\norm{\bw}) H_k\left(\frac{~\bw^\top}{\norm{\bw}} \bx\right) \\
        = & \sum_{k=0}^\infty \left \langle H_k\left(\bx\right), b_k(\norm{\bw}) \left(\frac{\bw}{\norm{\bw}}\right)^{\otimes k} \right\rangle
    \end{align*}
    At the same time, $f_\bw(\bx)$ has a Hermite expansion $f_\bw(\bx)=\sum_{k=0}^\infty \left \langle H_k\left(\bx\right), B_k(\bw) \right\rangle$. By the uniqueness of the Hermite expansion (since Hermite polynomials form a complete and orthonormal basis), it follows that $B_k(\bw) = b_k(\norm{\bw}) \left(\frac{\bw}{\norm{\bw}}\right)^{\otimes k}$ which completes the proof.
\end{proof}

We now show that the Hermite coefficients lie in a subspace determined by the span of the rows of $W$.
\begin{lemma}\label{lem: hermite_proj_W}
Let $W\in\R^{m\times d}$ and $f_W(\bx) := \sigma(W\bx)$ where $\sigma:\R^m\to\R$ and $f_W\in\Fcal_d$. Let $A_k(W) := \E[f_W(\bx)H_k(\bx)] \in (\R^d)^{\otimes k}$ be the $k$-th Hermite coefficient tensor of $f_W$, then
    \begin{align*}
        A_k(W) = P_W^{\otimes k} A_k(W).
     \end{align*}
\end{lemma}
\begin{proof}
    Let $P:=P_W$ and $Q:=I-P_W$, and let $r:=\rank(W)$. We decompose $\bx$ into $P\bx$ and $Q\bx$, which (because $\bx$ is Gaussian) are independent Gaussians supported on orthogonal subspaces of dimensions $r$ and $d-r$ respectively. Using this and that $f_W(\bx)=f_W(P\bx)$,
    \begin{align*}
        A_k(W)
        = &\ \E_\bx\left[f_W(P\bx)H_k(P\bx + Q\bx)\right] \\ 
        = &\ \E_{P\bx}\left[f_W(P\bx) \cdot \E_{Q\bx}\left[H_k(P\bx + Q\bx)\right]\right].
    \end{align*}
    We now show that for every fixed $\bz\in\R^d$ in the range of $P$ (i.e., $P\bz=\bz$),
    \begin{align*}
        \E_{\by\sim\Ncal(0,Q)}\left[H_k(\bz+\by)\right] = P^{\otimes k}H_k(P\bz).
    \end{align*}
    To this end, fix $\bv\in\R^d$ and use \eqref{eq: generating} so that
    \begin{align}\label{eq: hermite_genfun}
        \exp\left(\langle \bv,\bx\rangle-\tfrac12\|\bv\|^2\right)
        =
        \sum_{j=0}^\infty \frac{1}{\sqrt{j !}}
        \left\langle H_j(\bx), \bv^{\otimes j}\right\rangle,
        \qquad \forall \bx \in\R^d.
    \end{align}
    Applying \eqref{eq: hermite_genfun} with $\bx=\bz+\by$ and taking expectation over $\by\sim\Ncal(0,Q)$ yields
    \begin{align*}
        \sum_{j=0}^\infty \frac{1}{\sqrt{j !}}
        \left\langle \E_{\by}\left[H_j(\bz+\by)\right], \bv^{\otimes j}\right\rangle
        =\ 
        \E_{\by}\left[\exp\left(\langle \bv,\bz+\by\rangle-\tfrac12\|\bv\|^2\right)\right].
    \end{align*}
    Since $\E_{\by}[e^{\langle \bv,\by\rangle}] = \exp\left(\tfrac12 \bv^\top Q \bv\right)$ for $\by\sim\Ncal(0,Q)$, the right-hand side becomes
    \begin{align*}
        \exp\left(\langle \bv,\bz\rangle-\tfrac12\|\bv\|^2\right)\exp\left(\tfrac12 \bv^\top Q \bv \right)
        =
        \exp\left(\langle \bv,\bz\rangle-\tfrac12 \bv^\top P \bv\right)
        =
        \exp\left(\langle P\bv,\bz\rangle-\tfrac12\|P\bv\|^2\right),
    \end{align*}
    where the last equality uses that $P\bz = \bz$. To recap, so far we have shown
    \begin{align}\label{eq: ckpt}
        \exp\left(\langle P\bv,\bz\rangle-\tfrac12\|P\bv\|^2\right) = \sum_{j=0}^\infty \frac{1}{\sqrt{j !}}
        \left\langle \E_{\by}\left[H_j(\bz+\by)\right], \bv^{\otimes j}\right\rangle.
    \end{align}
    Now consider applying \eqref{eq: hermite_genfun} again, but now with $\bx=P\bz$ and $\bv$ replaced by $P\bv$. Then
    \begin{align*}
        \exp\left(\langle P\bv,\bz\rangle-\tfrac12\|P\bv\|^2\right)
        =
        \sum_{j=0}^\infty \frac{1}{\sqrt{j !}}
        \left\langle H_j(P\bz), (P\bv)^{\otimes j}\right\rangle
        =
        \sum_{j=0}^\infty \frac{1}{\sqrt{j !}}
        \left\langle P^{\otimes j}H_j(P\bz), \bv^{\otimes j}\right\rangle.
    \end{align*}
    Comparing this series to the one in \eqref{eq: ckpt} and matching coefficients of $\bv^{\otimes k}$ yields
    \begin{align}\label{eq: cond_expect_proj}
        \E_{\by}\left[H_k(\bz+\by)\right] = P^{\otimes k}H_k(P\bz).
    \end{align}

    Returning to $A_k(W)$ and using \eqref{eq: cond_expect_proj} with $\bz=P\bx$, we obtain
    \begin{align*}
        A_k(W)
        =\ & \E_{P\bx}\left[f_W(P\bx)\cdot P^{\otimes k}H_k(P\bx)\right]
        = P^{\otimes k}\E_{P\bx}\left[f_W(P\bx)H_k(P\bx)\right]
        = P^{\otimes k}A_k(W),
    \end{align*}
    completing the proof.
\end{proof}


\begin{lemma}\label{lem: corr_bound_matrix}
Let $W\in\R^{m\times d}$, $U\in\R^{p\times d}$, and let $f_W(\bx)=\sigma(W\bx), g_U(\bx)=\phi(U\bx)$ where $\sigma:\R^m\to\R$, $\phi:\R^p\to\R$ and $f_W, g_U \in \Fcal_d$. 
Let $\rho := \|P_W P_U\|_{\op}$ and let $B_k(U) := \E[g_U(\bx)H_k(\bx)] \in (\R^d)^{\otimes k}$ the $k$-th Hermite coefficient tensor of $g_U$. Then
\begin{align*}
    \abs{\E[f_W(\bx)g_U(\bx)]} \leq \norm{f_W(\bx)}_{\fsq} \sqrt{\sum_{k=0}^\infty \norm{B_k(U)}_F^2 \cdot  \rho^{2k}}.
\end{align*}
\end{lemma}

\begin{proof}
Let $A_k := A_k(W) := \E[f_W(\bx)H_k(\bx)] \in (\R^d)^{\otimes k}$ be the $k$-th Hermite coefficient tensor of $f_W$ and $B_k := B_k(U)$ the $k$-th Hermite coefficient tensor of $g_U$. By the Hermite expansion and the fact that $A_k, B_k$ are symmetric,
\begin{align*}
    \E[f_W(\bx)g_U(\bx)] = \sum_{k=0}^\infty \langle A_k , B_k\rangle.
\end{align*}

By Lemma \ref{lem: hermite_proj_W}, $A_k = P_W^{\otimes k} A_k$ and similarly $B_k = P_U^{\otimes k} B_k$. Therefore,
\begin{align*}
    \E[f_W(\bx)g_U(\bx)] 
    = & \sum_{k=0}^\infty \langle P_W^{\otimes k} A_k , P_U^{\otimes k} B_k\rangle 
    = \sum_{k=0}^\infty \langle A_k , P_W^{\otimes k} P_U^{\otimes k} B_k\rangle \\ 
    = & \sum_{k=0}^\infty \langle A_k , (P_W P_U)^{\otimes k} B_k\rangle.
\end{align*}

Applying Cauchy-Schwarz,
\begin{align*}
    \abs{\langle A_k , (P_W P_U)^{\otimes k} B_k\rangle} 
    \leq & \norm{A_k}_F \norm{(P_W P_U)^{\otimes k} B_k}_F 
    \leq \norm{A_k}_F \norm{(P_W P_U)^{\otimes k}}_{\op} \norm{B_k}_F \\ 
    =&_{(*)} \norm{A_k}_F \rho^k \norm{B_k}_F,
\end{align*}
where (*) uses that $\norm{M^{\otimes k}}_\op = \norm{M}_{\op}^k$ for any matrix $M$.

Summing over $k$ and applying Cauchy-Schwarz again yields
\begin{align*}
    \abs{\E[f_W(\bx)g_U(\bx)]} \leq \sum_{k=0}^\infty \norm{A_k}_F \norm{(P_W P_U)}_{\op}^k \norm{B_k}_F \leq \sqrt{\sum_{k=0}^\infty \norm{A_k}_F^2} \sqrt{\sum_{k=0}^\infty \norm{B_k}_F^2 \rho^{2k}}.
\end{align*}
Finally, using that $\sum_k\|A_k\|_F^2=\|f_W(\bx)\|_2^2$ completes the proof.
\end{proof}

\begin{lemma}[Gradient shifts Hermite coefficients]\label{lem:grad_shifts_hermite}
Let $g \in \Fcal_d$ with Hermite coefficient tensors $B_k \in (\R^d)^{\otimes k}$. Let $D_k := \E[\nabla g(\bx) \otimes H_k(\bx)] \in \R^d \otimes (\R^d)^{\otimes k}$ be the $k$-th Hermite coefficient tensor of $\nabla g(\bx)$. Then:
\[
\|D_k\|_F^2 = (k+1)\|B_{k+1}\|_F^2.
\]
\end{lemma}
\begin{proof}
    We compute the $i$-th slice of this tensor, denoted $(D_k)_{i}$, which corresponds to the Hermite coefficients of the partial derivative $\frac{\partial}{\partial x_i} g(\bx)$. Applying Stein's identity (Gaussian integration by parts), we obtain:
    \[
        (D_k)_{i} = \E\left[\frac{\partial}{\partial x_i} g(\bx) \cdot H_k(\bx)\right] = \E\left[ g(\bx) \left( x_i H_k(\bx) - \frac{\partial}{\partial x_i} H_k(\bx) \right) \right].
    \]
    Using the recurrence relation \eqref{eq: recurrence}:
    \[
        (D_k)_{i} = \sqrt{k+1} \E\left[ g(\bx) (H_{k+1}(\bx))_i \right] = \sqrt{k+1} (B_{k+1})_i.
    \]
    Squaring the Frobenius norm and summing over $i$ yields the result.
\end{proof}

\begin{lemma}\label{lem: grad_bound}
    Let $W\in\R^{m\times d}$, $U\in\R^{p\times d}$, and let $f_W(\bx)=\sigma(W\bx), g_U(\bx)=\phi(U\bx)$ where $\sigma:\R^m\to\R$, $\phi:\R^p\to\R$ and $f_W, g_U \in \Fcal_d$. 
    Let $\rho := \|P_W P_U\|_{\op}$ and let $B_k(U) := \E[g_U(\bx)H_k(\bx)] \in (\R^d)^{\otimes k}$ the $k$-th Hermite coefficient tensor of $g_U$. Then
    \begin{align*}
    \Big\|\E\left[\nabla_W f_W(\bx) g_U(\bx)\right]\Big\|_F
    \leq
    & \sqrt{\sum_{i=1}^m \sum_{j=1}^d \norm{\left(\nabla\sigma(W\bx)\right)_i \left(P_W\bx\right)_j}_\fsq^2} \cdot \sqrt{\sum_{k=0}^\infty \norm{B_k(U)}_F^2 \cdot  \rho^{2k}}
     \\
    & \quad + \sqrt{\sum_{i=1}^m \norm{\left(\nabla\sigma(W\bx)\right)_i}_\fsq^2} \sqrt{ \sum_{k=0}^\infty (k+1)\norm{B_{k+1}(U)}_F^2 \rho^{2k}}.
    \end{align*}
\end{lemma}
\begin{proof}
    By the chain rule,
    \[
    \E\left[\nabla_W f_W(\bx) g_U(\bx)\right]
    = \E\left[g_U(\bx)\nabla\sigma(W\bx)\bx^\top\right].
    \]
    
    We decompose $\bx$ into its projection onto $W$ and its orthogonal complement, so that:
    \begin{align}\label{eq: m_bounds}
        \E\left[\nabla_W f_W(\bx) g_U(\bx)\right]
        = & \underset{=:M_1}{\underbrace{
        \E\left[g_U(\bx)\nabla\sigma(W\bx) \bx^\top P_W\right]}}
        + \underset{=:M_2}{\underbrace{\E\left[g_U(\bx)\nabla\sigma(W\bx) \bx^\top (I-P_W)\right]}}.
    \end{align}
    To bound $M_1$, first note that
    \begin{align*}
        \norm{M_1}_F = & \sqrt{\sum_{i=1}^m \sum_{j=1}^d \left(\E\left[g_U(\bx)\left(\nabla\sigma(W\bx)\right)_i \left(P_W\bx\right)_j \right]\right)^2}.
    \end{align*}
    Applying Lemma \ref{lem: corr_bound_matrix} to the functions $\left(\nabla \sigma(W\bx) \right)_i \left(P_W\bx\right)_j$ and $g_U(\bx)$, it holds that 
    \begin{align*}
        \norm{M_1}_F \leq & \sqrt{\sum_{i=1}^m \sum_{j=1}^d \norm{\left(\nabla\sigma(W\bx)\right)_i \left(P_W\bx\right)_j}_\fsq^2} \cdot \sqrt{\sum_{k=0}^\infty \norm{B_k(U)}_F^2 \cdot  \rho^{2k}}.
    \end{align*}
    
    To bound $M_2$, first we will use Stein's lemma, which states that for any function $\psi$ and matrix $A$, $\E[\psi(\bx)\bx^\top A] = \E\left[\nabla_\bx \psi(\bx)^\top \right] A$. Applying this to $g_U(\bx)\left(\nabla\sigma(W\bx)\right)_i$ and using that $\nabla_\bx \left(\nabla\sigma(W\bx)\right)_i (I-P_W)=0$, we obtain
    \begin{align*}
        \norm{M_2}_F^2 = & \sum_{i=1}^m\norm{\E\left[g_U(\bx)\left(\nabla\sigma(W\bx)\right)_i \bx^\top (I-P_W)\right]}^2 \nonumber\\
        = & \sum_{i=1}^m\norm{\E\left[\left(\nabla\sigma(W\bx)\right)_i \nabla g_U(\bx)^\top\right](I-P_W) }^2 \nonumber\\ 
        \leq & \sum_{i=1}^m \sum_{j=1}^d \left(\E\left[\left(\nabla\sigma(W\bx)\right)_i \left(\nabla g_U(\bx)\right)_j\right]\right)^2 \norm{(I-P_W)}_\op^2.
    \end{align*}
    Define the (vector-valued) Hermite coefficient tensors for $\nabla \phi(U\bx)$ by
    \begin{align*}
        D_k(U) := \E[(\nabla g_U(\bx)) \otimes H_k(\bx)] \in \R^d \otimes (\R^d)^{\otimes k}.
    \end{align*}
    Note that by \lemref{lem:grad_shifts_hermite}, $\norm{D_k(U)}_F = \sqrt{k+1}\norm{B_{k+1}(U)}_F$. 
    Applying Lemma \ref{lem: corr_bound_matrix} for the correlation between $\left(\nabla\sigma(W\bx)\right)_i$ and $\left(\nabla g_U(\bx)\right)_j$ yields
    \begin{align*}
        \norm{M_2} 
        \leq & \norm{\left(I-P_W\right)}_\op \sqrt{\sum_{i=1}^m \norm{\left(\nabla\sigma(W\bx)\right)_i}_\fsq^2} \sqrt{\sum_{k=0}^\infty \norm{D_k(U)}_F^2 \rho^{2k}} \\
        \leq & \sqrt{\sum_{i=1}^m \norm{\left(\nabla\sigma(W\bx)\right)_i}_\fsq^2} \sqrt{\sum_{k=0}^\infty (k+1) \norm{B_{k+1}(U)}_F^2 \rho^{2k}}.
    \end{align*}
    Plugging our bounds on $M_1$ and $M_2$ back into \eqref{eq: m_bounds} completes the proof.
\end{proof}

So far, we have not used much information on $f_W$. We now show that if $\sigma$ has bounded gradients (analogously to Assumption \ref{ass: bounded}), then the bounds from the previous lemma simplify as follows:
\begin{lemma}\label{lem: grad_bound2}
    Under the setting and notation of \lemref{lem: grad_bound}, assume further that there exists some $G_1 > 0$ such that for a.s. any $\bx$, $\norm{\nabla \sigma(W\bx)} \leq G_1$.
    Then
    \begin{align*}
        \Big\|\E\left[\nabla_W f_W(\bx) g_U(\bx)\right]\Big\|_F
        \leq & G_1 \sqrt{m\sum_{k=0}^\infty \norm{B_k(U)}_F^2 \rho^{2k}} +  G_1 \sqrt{\sum_{k=0}^\infty (k+1)\norm{B_{k+1}(U)}_F^2 \rho^{2k}}.
    \end{align*}
\end{lemma}
\begin{proof}
    Starting from \lemref{lem: grad_bound}, it holds that
    \begin{align}\label{eq: grad_raw}
        \Big\|\E\left[\nabla_W f_W(\bx) g_U(\bx)\right]\Big\|_F
        \leq
        & \sqrt{\sum_{i=1}^m \sum_{j=1}^d \norm{\left(\nabla\sigma(W\bx)\right)_i \left(P_W\bx\right)_j}_\fsq^2} \cdot \sqrt{\sum_{k=0}^\infty \norm{B_k(U)}_F^2 \cdot  \rho^{2k}}
         \nonumber \\
        & \quad + \sqrt{\sum_{i=1}^m \norm{\left(\nabla\sigma(W\bx)\right)_i}_\fsq^2} \sqrt{ \sum_{k=0}^\infty (k+1)\norm{B_{k+1}(U)}_F^2 \rho^{2k}}.
    \end{align}
    For the first term, the boundedness assumption of $\nabla \sigma$ implies that
    \begin{align*}
        \sqrt{\sum_{i=1}^m \sum_{j=1}^d \norm{\left(\nabla\sigma(W\bx)\right)_i \left(P_W\bx\right)_j}_\fsq^2} 
        = & \sqrt{\int_{\bx} \left(\sum_{i=1}^m \nabla\sigma(W\bx)_i^2\right) \left(\sum_{j=1}^d (P_W \bx)_j^2\right) ~ \mu(\bx)d\bx} \\
        \leq & G_1 \sqrt{\E\left[\norm{P_W\bx}^2\right]} = G_1 \sqrt{\text{rank}(W)} \leq G_1 \sqrt{m}.
    \end{align*}
    For the second summand in \eqref{eq: grad_raw}, the boundedness assumption implies that
    \begin{align*}
        \sqrt{\sum_{i=1}^m \norm{\left(\nabla\sigma(W\bx)\right)_i}_\fsq^2} \leq G_1.
    \end{align*}
    Substituting these bounds back into \eqref{eq: grad_raw} yields the desired bound.
\end{proof}

The following lower bound on the variance will be useful for lower-bounding the gradient condition number.
\begin{lemma}\label{lem:var_bound}
    Let $\sigma, \phi \in \Fcal_1$, $\bw, \bu \in \R^d$, let $f(\bx):= \sigma(\bw^\top \bx)$, $g(\bx):= \phi(\bu^\top \bx)$ and assume that $\norm{f}_{L_\infty} < G_1, \norm{g}_{L_\infty} < G_2$ for $G_1,G_2>0$. Then for $\rho:= \norm{P_{\bw^\top} P_{\bu^\top}}_\op$,
    \begin{align*}
        \abs{\E\left[f(\bx)^2 g(\bx)^2\right]} \geq \norm{f}_{\fsq}^2\norm{g}_\fsq^2 - 4G_1^2 G_2^2 \rho.
    \end{align*}
\end{lemma}
\begin{proof}
    First, note that
    \begin{align*}
        \E\left[f(\bx)^2 g(\bx)^2\right]  
        = & \E\left[f(\bx)^2\right] \E\left[g(\bx)^2\right] + \left(\E\left[f(\bx)^2g(\bx)^2\right] -  \E\left[f(\bx)^2\right] \E\left[g(\bx)^2\right]\right) 
        \\ 
        = & \norm{f}_{\fsq}^2\norm{g}_\fsq^2 + \E\left[\left(f(\bx)^2 - \E[f(\bx)^2]\right) \left(g(\bx)^2 - \E[g(\bx)^2]\right) \right]
    \end{align*}
    Now consider the functions $\hat f(\bx):=f(\bx)^2 - \E[f(\bx)^2]$ and $\hat g(\bx):=g(\bx)^2 - \E[g(\bx)^2]$, then
    \begin{align*}
        \abs{\E\left[f(\bx)^2 g(\bx)^2\right]} \geq \norm{f}_{\fsq}^2\norm{g}_\fsq^2 - \abs{\E[\hat f(\bx)\hat g(\bx)]}.
    \end{align*}
    Let $b_k$ denote the Hermite coefficients of $\hat g$. \lemref{lem: corr_bound_matrix} implies that 
    \begin{align*}
        \abs{\E\left[f(\bx)^2 g(\bx)^2\right]} 
        \geq & \norm{f}_{\fsq}^2\norm{g}_\fsq^2 - \norm{\hat f}_\fsq \sqrt{\sum_{k=0}^\infty b_k^2 \rho^{2k}} \\
        \geq & \norm{f}_{\fsq}^2\norm{g}_\fsq^2 - \norm{\hat f}_\fsq \sqrt{\sum_{k=1}^\infty b_k^2 } \cdot \rho \\
        = & \norm{f}_{\fsq}^2\norm{g}_\fsq^2 - \norm{\hat f}_\fsq \norm{\hat g}_\fsq \rho,
    \end{align*}
    where we used that $b_0=0$ (since $\E[\hat g(\bx)]=0$), that $\rho \in [0,1]$ and that $\norm{\hat g}_\fsq^2 = \sum_{k=1}^\infty b_k^2$. Using that $f(\bx)^2 \leq G_1^2$ and $g(\bx)^2\leq G_2^2$ completes the proof.
\end{proof}

\section{Applications: Proofs for Section \ref{sec:applications}}
\subsection{Remark on Assumption \ref{ass: non_lin}} \label{app: non_lin}
As discussed in the paper, many natural non-linearities satisfy Assumption \ref{ass: non_lin}. The only non trivial property in the assumption is that $\E_{x\sim \Ncal(0, s^2)}[\sigma'(x)^2] \geq G_1^2/\kappa$ for some $\kappa > 0$ and any $s>0$ (where $G_1=\norm{\sigma'}_{L_\infty}$. Let us demonstrate this for a few different choices of $\sigma$:
\begin{itemize}
    \item \emph{ReLU / Leaky ReLU ($\sigma(x) = \max(x, \alpha x)$ for $\alpha \in [0, 1)$).} Then $G_1=1$ and $\sigma'(x)$ is $1$ for $x> 0$ and $\alpha$ for $x\leq 0$. So $\E_{x\sim \Ncal(0, s^2)}[\sigma'(x)^2] = \frac{1 + \alpha^2}{2}$. 

    \item \emph{Softplus ($\sigma(x) = \ln(1 + e^x)$).} The derivative is the sigmoid function $\sigma'(x) = \frac{1}{1 + e^{-x}}$. Here, $\sigma'(x) \in (0, 1)$ and specifically $\sigma'(x) \geq 1/2$ for all $x \geq 0$. Thus, $\E_{x\sim \Ncal(0, s^2)}[\sigma'(x)^2] \geq \Pr(x \geq 0) \cdot (1/2)^2 = 1/8$, satisfying the assumption with $\kappa = 8$.

    \item \emph{GeLU ($\sigma(x) = x \Phi(x)$ where $\Phi$ is the standard Gaussian CDF).} The derivative is $\sigma'(x) = \Phi(x) + x \phi(x)$, where $\phi$ is the standard Gaussian PDF. It holds that $\sigma'(x) > 1/2$ for any $x\geq 0$ and $\abs{\sigma'(x)}\leq 3/2$. Similar to the sigmoid case, one can take $\kappa = 12$.

    \item \emph{$\sin(x)$. } The derivative is $\cos(x)$ and by \cite{song2021cryptographic}[Lemma I.10], 
    \begin{align*}
        \E_{x\sim \Ncal(0, s^2)}[\cos(x)^2] = \frac{1}{2}\left(1 + \exp\left(-2s^2 \right)\right) \geq \frac{1}{2},
    \end{align*}
\end{itemize}

\subsection{Additional Lemmas for Single Index Models}
We provide a ``bootstrapping'' lemma that ensures that $\kappaT$ are indeed well bounded. We will use the fact that the gradient condition is small when the predictor is relatively uncorrelated with the target function. We will prove this in the case of $p=m=1$.

\begin{lemma}\label{lem: bootstrap}
    Let $m=p=1$, $\ell$ be the correlation loss and $f_{\theta_t}(\bx)=\sigma(\langle \bw_{t} \,, \,\bx \rangle)$ for $\sigma$ (sub-)differentiable. Let $f^\star(\bx):= \phi(\bu^\top \bx)$ for $\bu \in \R^d$, where $\phi' \in L_2(\mu)$. Under Assumptions \ref{ass: inputs_all}-\ref{ass: bounded}, let $\delta > 0$, $T\in \N$, and let $\psi:[0, 1]\to [0,\infty)$ be an increasing function  such that $\norm{\nabla_{W} \Lcal(\theta)}_F \leq \psi\left(\norm{P_W P_U}_\op\right)$ for all $\theta$.Fix $R \in (0, 1]$ and for any $r\geq 0$ let 
    \[
    \bar \kappa(r) := \frac{G^2}{\inf_{\bw \,:\, \norm{P_{\bw^\top} P_{\bu^\top}}_\op \leq r}
    ~\E_\bx\left[\sigma'(\bw^\top \bx)^2 f^\star(\bx)^2\right]} ~.
    \]
    There exist constants $C, c, c' > 0$ (that may depend on $K_1, K_2, K_3, \alpha_2, G$), such that if $d \geq c  \frac{\bar \kappa(R)^2}{R^2}\log\left(\frac{Td}{\delta}\right)^2$, $\eta \leq \frac{c'}{\bar \kappa(R)^2\sqrt{d \log\left(Td/\delta\right)}}$ and $T \leq \frac{1}{\psi\left(C \sqrt{\frac{\bar \kappa(R)^{} \log\left(Td/\delta\right)}{d}} \right)^2}~$,
    it holds with probability at least $1-\delta$ that for all $t\in[T]$
    \[
        \kappaT \leq \bar\kappa(R)~.
    \]
\end{lemma}
\begin{proof}
    For all $t$ let $\rho_t:=\norm{P_{\bw^\top_t} P_{\bu^\top}}_\op$. 
    First, by \lemref{lem: w0}, for some absolute constant $c_1\geq 0$ it holds with probability at least $1-\delta/2$ that
    \begin{align}\label{eq: rho_0}
        \rho_0 \leq \frac{c_1K_3 \sqrt{\frac{\log\left(\frac{8}{\delta}\right)}{d}}}{1 - c_1K_3^2 \sqrt{\frac{\log\left(\frac{8}{\delta}\right)}{d}}} \leq R,
    \end{align}
    where the last inequality used the lower bound on $d$ from the lemma statement (up to a choice of constants). 
    Since $\ell$ is the correlation loss and $m=1$, by definition (and using that $\bx_t$ is independent of the prior history), for any $\tau \in [T]$, 
    \[
    \kappa_\tau = \frac{G^2}{\min_{t\leq\tau}\E_\bx\left[\sigma'(\bw^\top_{t-1} \bx)^2 f^\star(\bx)^2\right]},
    \]
    In particular, by the definition of $\bar \kappa$, under the above event, it holds that $\kappa_1 \leq \bar \kappa(R)$. Now we apply \thmref{thm: main} to every $\tau \in [T]$ (meaning that the value of $T$ in \thmref{thm: main} is taken to be $\tau$), with failure probability $\delta / 2T$ and with $\bar\kappa$ taken to be $\bar \kappa(R)$. Let $\Ecal$ be the event that \eqref{eq: rho_0} holds as well as all instantiations of \thmref{thm: main}. $\Ecal$ holds with probability at least $1-\delta$ by the union bound, which we will assume is the case in the remainder of the proof. Note that for every $\tau$ such that $\kappa_{\tau} \leq \bar \kappa(R)$, \thmref{thm: main} implies that for some $C>0$ (that may depend on $K_1,K_2,K_3,\alpha_2,G$),
    \[
    \rho_{\tau + 1} \leq C\sqrt{\frac{\bar \kappa(R)\log(2T^2d/\delta)}{d}} \leq R,
    \]
    where the last inequality uses the lower bound on $d$ (and possibly an adjustment of $c$). This in turn implies $\kappa_{t+1} \leq \bar \kappa(R)$ and the proof follows from induction. 
\end{proof}

We now provide an upper bound for the condition number when learning with two-layer periodic functions. 
\begin{lemma}\label{lem: kappa_bound}
    Let $W_0 \sim \Ncal\left(0, \frac{1}{d}I_{md}\right)$, $\bx\sim \Ncal(0, I_d)$, $m=p=1$, $\ell$ be the correlation loss and $f_{\theta_t}(\bx)=\sigma(\langle \bw_{t} \,, \,\bx \rangle)$ for $\sigma$ satisfying Assumption \ref{ass: non_lin}. 
    Let $f^\star(\bx):= \phi(\bu^\top \bx)$ for $\bu \in \R^d$, where $\norm{f^\star}_{L_\infty} \leq G_2$, $\phi' \in L_2(\mu)$. 
    Under Assumptions \ref{ass: inputs_all}-\ref{ass: init}, let $\delta > 0$, $T\in \N$, and let $\psi:[0, 1]\to [0,\infty)$ be an increasing function  such that $\norm{\nabla_{W} \Lcal(\theta)}_F \leq \psi\left(\norm{P_W P_U}_\op\right)$ for all $\theta$.
    There exist constants $C, c, c' > 0$ (that may depend on $G_1, G_2, \kappa, \norm{f^\star}_\fsq)$, such that if $d \geq c  \log\left(\frac{Td}{\delta}\right)^2$, $\eta \leq \frac{c'}{\sqrt{d \log\left(Td/\delta\right)}}$ and $T \leq \frac{1}{\psi\left(C \sqrt{\frac{ \log\left(Td/\delta\right)}{d}} \right)^2}~$,
    it holds with probability at least $1-\delta$ that
    \[
    \kappaT \leq \frac{2K_\sigma G_2^2}{\norm{f^\star}_\fsq^2},
    \]   
    where $\kappa$ is the value in Assumption \ref{ass: non_lin}.
\end{lemma}
\begin{proof}
    First, note that the inputs and initialization satisfy Assumptions \ref{ass: subgauss_inputs}, \ref{ass: norm_conc}, \ref{ass: init} with constants $K_1, K_2, K_3, \alpha_2 = \Theta(1)$. Moreover, by the choice of the correlation loss, by Assumption \ref{ass: non_lin} and by the assumption that $\norm{f^\star}_{L_\infty} \leq G_2$, it follows that Assumption \ref{ass: bounded} is satisfied with $G=G_1G_2$. For any $r\geq 0$ let 
    \[
    \bar \kappa(r) := \frac{G_1^2G_2^2}{\inf_{\{\theta ~\mid~ \norm{P_{\bw^\top} P_{\bu^\top}}_\op \leq r\}}
    \E_\bx\left[\sigma'(\bw^\top \bx)^2 f^\star(\bx)^2\right]} ~.
    \]    
    By applying \lemref{lem:var_bound} (to the functions $\sigma'$ and $\phi$) whose conditions are satisfied by the conditions of this lemma, it holds for any $\bw$ that 
    \begin{align*}
        \E_\bx\left[\sigma'(\bw^\top \bx)^2 f^\star(\bx)^2\right] 
        \geq &\norm{\sigma'(\bw^\top \bx)}_\fsq^2 \norm{f^\star}_\fsq^2 - 4G_1^2 G_2^2 \norm{P_{\bw^\top} P_{\bu^\top}}_\op \\ 
        \geq & \frac{G_1^2}{K_\sigma} \norm{f^\star}_\fsq^2 - 4G_1^2 G_2^2 \norm{P_{\bw^\top} P_{\bu^\top}}_\op,
    \end{align*}
    where the last inequality follows from Assumption \ref{ass: non_lin}. Using this and the definition of $\bar\kappa$, it follows that for $R:= \frac{\norm{f^\star}_\fsq^2}{8K_\sigma G_2^2}$,
    \begin{align*}
        \bar \kappa(R) \leq \frac{G_1^2G_2^2}{G_1^2G_2^2 \left(\frac{\norm{f^\star}_\fsq^2}{K_\sigma G_2^2} - 4R\right)} = \frac{1}{\frac{\norm{f^\star}_\fsq^2}{2K_\sigma G_2^2}} = \frac{2K_\sigma G_2^2}{\norm{f^\star}_\fsq^2}.
    \end{align*}
    Applying \lemref{lem: bootstrap} completes the proof.
\end{proof}

\subsection{Periodic Functions: Proof of \thmref{thm: periodic}}\label{app: periodic}
We now prove our bound on the correlation loss and the gradient with sinusoidal target functions.
\begin{lemma}\label{lem: periodic}
    Let $\sigma \in \Fcal_1$ be differentiable and satisfy $\abs{\sigma'(x)}\leq G_1$ for $G_1>0$ and a.s. any $x$. Let $f_\bw(\bx):= \sigma(\bw^\top \bx)$, then for any $\bw\in \R^d$ and $\rho:= \norm{P_{\bw^\top} P_{\bu^\top}}_\op$,
    \begin{align*}
        \abs{\E\left[f_\bw(\bx)\sin\left(\bu^\top \bx\right)\right]} 
        \leq & \norm{f_\bw(\bx)}_{\fsq}\exp\left(-\frac{\norm{\bu}^2}{2} \left(1 - \rho^2 \right)\right). 
    \end{align*}
    and 
    \begin{align*}
        \norm{\E\left[\nabla_{\bw}f_\bw(\bx) \sin\left(\bu^\top \bx\right)\right]} 
        \leq & G_1\left(1 + \norm{\bu}\right)\exp\left(-\frac{\norm{\bu}^2}{2} \left(1 - \rho^{2} \right)\right). 
    \end{align*}
\end{lemma}
\begin{proof}
    Let $g(\bx):= \sin(\bu^\top \bx)$ with corresponding Hermite coefficients $B_k$, then \lemref{lem: corr_bound_matrix} and \lemref{lem: grad_bound2} imply that
    \begin{align}\label{eq: per_loss}
        \abs{\E[f_\bw(\bx)g(\bx)]} \leq \norm{f_\bw(\bx)}_{\fsq} \sqrt{\sum_{k=0}^\infty \norm{B_k}_F^2 \cdot  \rho^{2k}},
    \end{align}
    and
    \begin{align}\label{eq: per_grad}
        \Big\|\E\left[\nabla_\bw f_\bw(\bx) g(\bx)\right]\Big\|_F
        \leq & G_1 \sqrt{\sum_{k=0}^\infty \norm{B_k}_F^2 \rho^{2k}} +  G_1 \sqrt{\sum_{k=0}^\infty (k+1)\norm{B_{k+1}}_F^2 \rho^{2k}}.
    \end{align}
    By \lemref{lem: si_hermite}, $B_k = b_k\left(\frac{\bu}{\norm{\bu}}\right)^{\otimes k}$ with $b_k:=\E_{x\sim \Ncal(0,1)}\left[\sin(\norm{\bu} x) H_k(x)\right]$. \citet[Lemma I.10]{song2021cryptographic} provide the Hermite coefficients for cosines, which together with \lemref{lem:grad_shifts_hermite} (that relates the expansion of a function to that of its derivative) implies,
    \begin{align*}
        b_k^2 = 
        \begin{cases}
            0 & k \text{ even} \\
            \frac{\norm{\bu}^{2k}}{k!} \exp\left(-\norm{\bu}^2\right) & k \text{ odd}.
        \end{cases}
    \end{align*}
    It follows that 
    \begin{align*}
        \sum_{k=0}^\infty \norm{B_k}_F^2 \cdot  \rho^{2k} 
        = &\sum_{k \text{ odd}} \frac{\norm{\bu}^{2k}}{k!} \exp\left(-\norm{\bu}^2\right) \rho^{2k} 
        = \sinh \left(\norm{\bu}^2\rho^2\right) \exp\left(-\norm{\bu}^2\right) \\ 
        \leq & \exp\left(-\norm{\bu}^2 \left(1 - \rho^{2} \right)\right), 
    \end{align*}
    where we used that $\sinh(x)\leq \exp(x)$. Likewise, 
    \begin{align*}
        \sum_{k=0}^\infty (k+1)\norm{B_{k+1}}_F^2 \cdot  \rho^{2k} 
        = & \sum_{k+1 \text{ odd}} \frac{\norm{\bu}^{2(k+1)}}{k!} \exp\left(-\norm{\bu}^2\right) \rho^{2k} \\
        = & \norm{\bu}^2 \sum_{k \text{ even}} \frac{\norm{\bu}^{2k}}{k!} \exp\left(-\norm{\bu}^2\right) \rho^{2k} \\
        = & \norm{\bu}^2\cosh \left(\norm{\bu}^2\rho^2\right) \exp\left(-\norm{\bu}^2\right) \\
        \leq & \norm{\bu}^2\exp\left(-\norm{\bu}^2 \left(1 - \rho^{2} \right)\right), 
    \end{align*}
    Plugging these back into \eqref{eq: per_loss} and \eqref{eq: per_grad} (and taking square roots) completes the proof.
\end{proof}

\periodic*
\begin{proof}
    First, notice that the correlation loss decomposes into individual neurons as
    \begin{align*}
        \ell(\theta_{t-1} \,;\, \bx_t) = \sum_{i=1}^m \Bigl[ - \sigma(\bw_{t,i}^\top ~ \bx_t) f^\star(\bx_t) \Bigr] = \sum_{i=1}^m \ell\left(\bw_{t,i} \,;\, \bx_t\right),
    \end{align*}
    where $ \ell\left(\bw_{t,i} \,;\, \bx_t\right)$ denotes the loss with respect to the predictor $\sigma(\bw_{t,i}^\top ~ \bx_t)$.
    As such, consider some arbitrary $i\in[m]$ and let $\bw_t:= \bw_{t,i}$ for all $t$. Naturally, we let $\Lcal(\bw_t):= \E[\ell\left(\bw_t \,;\, \bx_t\right)]$. We will prove that with probability at least $1-2\exp\left(-d^{1/3}\right)$,
    \begin{align*}
        \forall ~ t \leq \exp\left(d^{1/3}\right) ~,\qquad 
        \abs{\Lcal(\bw_t) - \Lcal(\bw_0)}
        \leq \left(\norm{\sigma(\bw_{t,i}^\top ~ \bx_t)}_{\fsq} + \norm{\sigma(\bw_{0,i}^\top ~ \bx_t)}_{\fsq}\right)\exp\left(-C d\right).
    \end{align*}
    The theorem follows by taking a union bound over all $i \in [m]$.
    
    We now show that the conditions of Theorem \ref{thm: main} are satisfied.
    First, note that the inputs and initialization satisfy Assumptions \ref{ass: subgauss_inputs}, \ref{ass: norm_conc}, \ref{ass: init} with $K_1, K_2, K_3, \alpha_2 = \Theta(1)$. 
    Next, $\abs{\sin(\bu^\top \bx)} \leq 1$, so \ref{ass: non_lin} implies that Assumption \ref{ass: bounded} is also satisfied with $G=G_1$. Furthermore,
    \[
    \norm{f^\star}_\fsq = \E[\sin(\bu^\top \bx)^2] = \frac{1}{2}\left(1-\exp(-2\norm{\bu}^2)\right) > \frac{1}{4},
    \]
    where the last inequality uses that $\norm{\bu}^2 =d \geq 1$.
    Next, let 
    \[
    \rho_t := \norm{P_{\bw^\top_t}P_{\bu^\top}}.
    \]
    Lemma \ref{lem: periodic} together with the choice of $\norm{\bu} = \sqrt{d}$ imply that
    \begin{align*}
        \norm{\nabla_{\bw} \Lcal(\bw_t)} \leq \psi\left(\rho_t\right) \qquad \text{where} \qquad \psi(\rho) := 2G_1\sqrt{d}\exp\left(-\frac{d}{2}\left(1-\rho^2\right)\right).
    \end{align*}
    Fix $0 < \alpha < 1$ to be decided later. Let $T := \exp\left(d^\alpha\right)$ and $\delta := \exp\left(-d^{\alpha}\right)$.
    For some constants $c, c' > 0$ that will be determined later (that may depend on $K_\sigma$, $G_1$), in order to apply \lemref{lem: kappa_bound} and \thmref{thm: main}, we need
    $d \geq c \log\left(\frac{Td}{\delta}\right)^2 = c\left(\log(d) + 2d^{\alpha}\right)^2$ and $\eta \leq \frac{c'}{\sqrt{d \log\left(T/\delta\right)}} = \frac{c'}{2d^{\alpha + 1/2}}$.
    So as a minimal requirement, we will pick $\alpha$ to be smaller than $1/2$. We also need the following inequality to hold for some constant $C>0$ that will be determined later (that may depend on $K_\sigma$, $G_1$; where we plugged in that in this particular problem setup $m=p=1$),
    \begin{align}\label{eq: t_condition}
        T \leq \frac{1}{\psi\left(C \sqrt{\frac{\log\left(\frac{Td}{\delta}\right)}{d}} \right)^2} ~.
    \end{align}
    With our choice of $T$, $\delta$ and $\psi$
    \begin{align*}
        (\star) := \frac{1}{T} \left \lfloor \left(\frac{1}{\psi\left(C \sqrt{\frac{\log\left(\frac{Td}{\delta}\right)}{d}} \right)}\right)^2 \right \rfloor
        = \frac{1}{\exp(d^\alpha)} \cdot \frac{\exp\left(\frac{d}{2}\left(1 - C^2 \frac{\log(d) + 2d^{\alpha}}{d}\right)\right)}{4G_1^2d} ~.
    \end{align*}
    When $\alpha < 1$, $\lim_{d\to\infty} \left(1 - C^2 \frac{\log(d) + 2d^{\alpha}}{d}\right) = 1$, and thus for sufficiently large $d$ it holds that the value in $(\star)$ is at least $1$ (meaning that \eqref{eq: t_condition} is satisfied, as needed). So take $\alpha = 1/3$.
    We have shown that the conditions of \lemref{lem: kappa_bound} are satisfied (for suitable constants), and with probability at least $1-\delta$ it follows that $\kappaT \leq 8K_\sigma$ (where we used the bounds on the norms of $\sin(\bu^\top \bx)$). 
    As such, the conditions of \thmref{thm: main} are also satisfied (again, up to possibly adjusting the constants), and thus \thmref{thm: main} implies that with probability at least $1-\exp(-d^{\alpha})$, for all $t\leq T$
    \begin{align*}
        \rho_t \leq C \sqrt{\frac{\log(Td/\delta)}{d}} \leq  C \sqrt{\frac{\log(d) + 2d^{\alpha}}{d}}.
    \end{align*}
    Note that for $\alpha<1$ this vanishes for large $d$. As such, by \lemref{lem: periodic} (also slightly adjusting the constant $C$ and using that $\norm{\sin(\bu^\top \bx)}_\fsq \leq 1$), this implies the desired bound on $\Lcal(\bw_t)$.
\end{proof}

\subsection{Multi-Index Bounds via the Information Exponent: Proof of \thmref{thm: mi_general}}\label{app: inf_exp}
\begin{lemma}\label{lem: mi_bounds_complete}
    Let $W\in\R^{m\times d}$, $U\in\R^{p\times d}$, and let $f_W(\bx)=\sigma(W\bx), g_U(\bx)=\phi(U\bx)$ where $f_W, g_U \in \Fcal_d$, $\sigma:\R^m\to\R$, $\phi:\R^p\to\R$, $g_U$ has information exponent $k_\star$, and $\norm{\nabla\sigma(W\bx)}_{L_\infty} \leq G_1$. Let $\rho := \|P_W P_U\|_{\op}$, then
    \begin{align*}
        \abs{\E\left[f_W(\bx)g_U(\bx)\right]} \leq \norm{f_W(\bx)}_{\fsq} \norm{g_U(\bx)}_{\fsq}\rho^{k_\star}.
    \end{align*}
    and 
    \begin{align*}
        \Big\|\E\left[\nabla_W f_W(\bx) g_U(\bx)\right]\Big\|_F
        \leq & G_1 \norm{\nabla g_U(\bx)}_{\fsq}\left(\sqrt{m}\rho + 1\right) \rho^{k_\star - 1}.
    \end{align*}
\end{lemma}
\begin{proof}
    First, by the definition of $k_\star$, for all $k<k_\star$ it holds that $B_k(U)=\zero$. So by \lemref{lem: corr_bound_matrix} and using that $\rho \in [0 , 1]$,
    \begin{align}
        \abs{\E[f_\bw(\bx)g(\bx)]} \leq \norm{f_\bw(\bx)}_{\fsq} \sqrt{\sum_{k=0}^\infty \norm{B_k}_F^2 \cdot  \rho^{2k}} \leq \norm{f_W(\bx)}_{\fsq} \norm{g_U(\bx)}_{\fsq}\rho^{k_\star}.
    \end{align}
    For the second part of the theorem, by \lemref{lem: grad_bound2} and the fact that $\rho \in [0, 1]$,
    \begin{align*}
        \Big\|\E\left[\nabla_W f_W(\bx) g_U(\bx)\right]\Big\|_F
        \leq &  G_1\sqrt{m\sum_{k=0}^\infty \norm{B_k(U)}_F^2 \rho^{2k}} +  G_1 \sqrt{\sum_{k=0}^\infty (k+1)\norm{B_{k+1}(U)}_F^2 \rho^{2k}} \\ 
        = & G_1\sqrt{m\sum_{k=k_\star}^\infty \norm{B_k(U)}_F^2 \rho^{2k}} +  G_1 \sqrt{\sum_{k=k_\star - 1}^\infty (k+1)\norm{B_{k+1}(U)}_F^2 \rho^{2k}} \\ 
        \leq & G_1\sqrt{m\sum_{k=k_\star}^\infty \norm{B_k(U)}_F^2} \cdot \rho^{k_\star} +  G_1 \sqrt{\sum_{k=k_\star}^\infty k\norm{B_{k}(U)}_F^2} \cdot \rho^{k_\star - 1} \\
        \leq & G_1 \left(\sqrt{m}\rho + 1\right) \sqrt{\sum_{k=k_\star}^\infty k\norm{B_{k}(U)}_F^2} \cdot \rho^{k_\star - 1}.
    \end{align*}
    Using that 
    by \lemref{lem:grad_shifts_hermite} $\sum_{k=1}^\infty k\|B_{k}(U)\|_F^2 = \|\nabla g_U(\bx)\|_{L_2}^2$, we conclude
    \[
        \Big\|\E\left[\nabla_W f_W(\bx) g_U(\bx)\right]\Big\|_F
        \leq G_1\left(\sqrt{m}\rho + 1\right)\|\nabla g_U(\bx)\|_{L_2}\,\rho^{k_\star-1}.
    \]
\end{proof}

\migeneral*
\begin{proof}
    First, we verify the assumptions needed for \thmref{thm: main}. Assumptions \ref{ass: inputs_all} and \ref{ass: init} for the assumed Gaussian init and inputs with constants $K_1, K_2, K_3, \alpha_2 = \Theta(1)$. For Assumption \ref{ass: bounded}, we note that $\nabla_{W\bx} \ell = -\nabla_{W\bx} h(W\bx ; \bar\theta) \cdot f^\star(\bx)$. Thus $\norm{\nabla_{W\bx} \ell}_2 \leq G_1 G_2$. We set $G := G_1 G_2$.

    Next, we apply Lemma \ref{lem: mi_bounds_complete}. It implies that for any $\theta$,
    \begin{align}\label{eq: psi_bound}
        \norm{\nabla_{W} \Lcal(\theta)}_F \leq \psi\left(\norm{P_W P_U}_\op\right)
        \qquad \text{where} \qquad
        \psi(r) := G_1 \norm{\nabla f^\star}_{L_2}\left(\sqrt{m}r + 1\right) r^{k_\star - 1}.
    \end{align}
    We divide the analysis into two cases.

    \paragraph{Case 1: $k_\star > 1$.} Let 
    \begin{align*}
        d_0 := & C_1 \bar \kappa m^2 p^2 k_\star^2\epsilon^{-2/k_\star} \log\left(\frac{C_1 \bar \kappa mpk_\star}{\epsilon\delta}\right)^2 ~, \\ 
        T := & \frac{c_1}{1 + G_1^2\norm{\nabla f^\star(\bx)}_{\fsq}^2}\left(\frac{d}{C \bar\kappa mp k_\star\log\left(\frac{dp}{\delta}\right)}\right)^{k_\star - 1}, \\ 
        \eta_0 = & \frac{c_2}{\bar \kappa^2 \sqrt{mdk_\star \log\left(\frac{C_1 dp}{\delta}\right)}} ~,
    \end{align*}
    where $0 < c_1, c_2 < 1$ are small constants and $C, C_1 > 1$ are large constants that will be determined throughout the proof, all of which may depend on $G$. Note that 
    \begin{align}\label{eq: log_t_bound}
        \log(T) \leq \log(d^{k_\star - 1}) \leq k_\star \log(d). 
    \end{align}
    
    To apply, \thmref{thm: main} we need to ensure for some constants $C',c'>0$ that may depend on $G$, that $d \geq C' \bar\kappa^2 m^2 \log\left(\frac{Tdp}{\delta}\right)^2$ and $\eta \leq \frac{c'}{\bar\kappa^2\sqrt{md \log\left(Tdp/\delta\right)}}$. 
    Using \eqref{eq: log_t_bound} and the choice of $d_0$ and $\eta_0$, these are satisfied for suitable $c_1,c_2,C,C_1$ (again, these only depend on $G$). Another condition of \thmref{thm: main} is that for some $C_2>0$ that may depend on $G$,
    \begin{align}\label{eq: mi_T1}
        1 \leq & \frac{1}{T}\cdot \frac{1}{\psi\left(C_2 \sqrt{\frac{\bar \kappa mp \log\left(\frac{Tdp}{\delta}\right)}{d}} \right)^2} \\ 
        = & \frac{1}{T}\cdot \frac{d^{k_\star - 1}}{G_1^2 \norm{\nabla f^\star(\bx)}_{\fsq}^2 \left(1 + C_2\sqrt{\frac{\bar \kappa m^2p \log\left(\frac{Tdp}{\delta}\right)}{d}}\right)^2 \left(C_2^2 \bar \kappa mp \log\left(\frac{Tdp}{\delta}\right)\right)^{k_\star - 1}} \nonumber.
    \end{align}
    By the choice of $d_0$ it holds that $\sqrt{\frac{\bar \kappa m^2p \log\left(\frac{Tdp}{\delta}\right)}{d}} \leq 1$. Using this, \eqref{eq: log_t_bound} and our choice of $T$, 
    \begin{align*}
        \frac{1}{T}\cdot \frac{1}{\psi\left(C_2 \sqrt{\frac{\bar \kappa mp \log\left(\frac{Tdp}{\delta}\right)}{d}} \right)^2} 
        \geq & \frac{1}{G_1^2 \norm{\nabla f^\star(\bx)}_{\fsq}^2 \left(1+C_2\right)^2} \cdot \frac{1}{T} \left(\frac{d}{2C_2^2\bar \kappa mpk_\star \log\left(\frac{dp}{\delta}\right)}\right)^{k_\star - 1} \\
        \geq&_{(*)} \frac{C^{k_\star - 1}}{c_1 \left(1+C_2\right)^2 (2C_2)^{k_\star - 1}} \geq_{(**)} 1,
    \end{align*}
    For $(*)$, $T$, and for $(**)$, we took $c_1$ sufficiently small and $C$ sufficiently large, and we note that they only need to depend on $C_2$ for this to hold (and thus on $G$).
    
    We have shown that all the conditions for \thmref{thm: main} hold, and as such, it implies that for some constant $C_3>0$, conditioned on $\kappaT \leq \bar \kappa$, with probability at least $1-\delta$ it holds that
    \begin{align*}
        \forall t\leq T , \qquad , \quad \norm{P_{W_t} P_U}_\op \leq C_3 \sqrt{\frac{\bar \kappa mp \log\left(dp\right)}{d}} \leq \epsilon^{1/k_{\star}}~,
    \end{align*}
    where the last inequality follows from the choice of $d_0$. 
    Plugging this into \lemref{lem: mi_bounds_complete} implies the desired bound on $\Lcal(\theta_t)$. 

    \paragraph{Case 2: $k_\star = 1$.}
    Note that the choice of $T$ is suboptimal when $k_\star = 1$. If $k_\star=1$, we instead use \propref{prop: d_bound}, so that under the same event and probability as \thmref{thm: main}, and under the same choice of $d_0$, $\eta_0$ for $T' = \left\lfloor \frac{c_1 d\epsilon^2}{p}\right \rfloor$ where $c_1>0$ is sufficiently small and may depend on $G$, it holds that
    \begin{align*}
        \forall t\leq T' , \qquad , \quad \norm{P_{W_t} P_U}_\op \leq C \sqrt{\frac{\bar \kappa mp \log\left(dp\right)}{d}} \leq \epsilon = \epsilon^{1 / k_{\star}}~.
    \end{align*}
    Once again, the theorem is complete by applying \lemref{lem: mi_bounds_complete}.
\end{proof}

\paragraph{Example: Product of Functions}\label{app: prod}
Let $\phi$ be a smooth, bounded and mean-zero function (e.g. $\tanh$) and $\{\bu_i\}_{i=1}^p$ be orthonormal vectors. Consider the function $f^\star(\bx) := \prod_{i=1}^p \phi(\bu_i^\top \bx)$. Then using the fact that $\bu_i^\top \bx$ are i.i.d. Gaussian, it holds for any $k \leq p$
    \begin{align*}
        & \E\left[\nabla^k f^\star(\bx)\right] \\
        = & \E\left[\sum_{\{j_1,\ldots, j_k\} \subseteq [p]} \left(\prod_{i \notin \{j_1,\ldots, j_k \}} \phi(\bu_i^\top \bx)\right) \left(\prod_{i \in \{j_1,\ldots, j_k \}} \phi'(\bu_i^\top \bx)\right) \bu_{j_1} \otimes \ldots \otimes \bu_{j_k}\right] \\
        = & \sum_{\{j_1,\ldots, j_k\} \subseteq [p]} \left(\prod_{i \notin \{j_1,\ldots, j_k \}}\E\left[ \phi(\bu_i^\top \bx)\right]\right) 
        \left(\prod_{i \in \{j_1,\ldots, j_k \}}\E\left[ \phi'(\bu_i^\top \bx)\right]\right) 
        \bu_{j_1} \otimes \ldots \otimes \bu_{j_k} \\
        = & \begin{cases}
            0 & k \neq p \\
            \left(\E_{G\sim \Ncal(0,1)}[\phi'(G)]\right)^p \sum_{\pi \in S_p} \bu_{\pi(1)} \otimes \ldots \otimes \bu_{\pi(p)} & k=p
        \end{cases}.
    \end{align*}
    So the first $p-1$ Hermite tensors vanish.

\subsection{Single-Index Bounds via the Information Exponent: Proof of \thmref{thm: si}}\label{app: si_inf_exp}
\singleidx*

\begin{proof}
    As in the proof of \thmref{thm: periodic}, notice that the correlation loss decomposes into individual neurons as
    \begin{align*}
        \ell(\theta_{t-1} \,;\, \bx_t) = \sum_{i=1}^m \Bigl[ - \sigma(\bw_{t,i}^\top ~ \bx_t) f^\star(\bx_t) \Bigr] = \sum_{i=1}^m \ell\left(\bw_{t,i} \,;\, \bx_t\right),
    \end{align*}
    where $ \ell\left(\bw_{t,i} \,;\, \bx_t\right)$ denotes the loss with respect to the predictor $\sigma(\bw_{t,i}^\top ~ \bx_t)$.
    As such, consider for now some arbitrary $i\in[m]$ and let $\bw_t:= \bw_{t,i}$ for all $t$. Naturally, we let $\Lcal(\bw_t):= \E[\ell\left(\bw_t \,;\, \bx_t\right)]$. We will now apply \thmref{thm: mi_general} to the predictor $\sigma(\bw_{t}^\top ~ \bx)$, and afterwards take a union bound over all $m$. Following \thmref{thm: mi_general}, the $k^\star = 1$ case is trivial, so it suffices to consider $k^\star > 1$.

    Throughout the proof, $\Theta$ hides constants related to $G_1, G_2$, $\norm{f^\star}_\fsq$ and $\norm{\nabla f^\star}_{L_2}$. Unlike in the theorem statement, we will make here explicit the role of $K_\sigma$.
    Let $\bar \kappa := \frac{2K_\sigma G_2^2}{\norm{f^\star}_{\fsq}^2}$ and consider the values of $d_0, \eta_0$ and $T$ as in \thmref{thm: mi_general}, where $d_0=\tilde{\Theta}\left(\bar \kappa^2 k_\star^2 \epsilon^{-2/k_\star}\right)$, $\eta_0 = \tilde{\Theta}\left(\bar \kappa^{-2}(dk_\star)^{-1/2}\right)$ and 
    \[
    T = \left(\frac{d}{\tilde \Theta\left(\bar\kappa k_\star\right)}\right)^{k_\star - 1} + \Theta\left(d\epsilon^2\right).
    \]
    Up to possibly increasing the constants, under these choices, the conditions of \lemref{lem: kappa_bound} are satisfied, implying that with probability at least $1-\delta/2$, $\kappaT \leq \bar \kappa$. As such, by \thmref{thm: mi_general} (with $p=1$, the value of $m$ in that theorem being $1$), for any $d\geq d_0' $ and $\eta \leq \eta_0'$ it holds with probability at least $1-\delta/2$ that
    \begin{align}\label{eq: si_bound_gen}
        \forall \,t \, \leq T
        ~,\qquad 
        \abs{\Lcal(\bw_t) - \Lcal(\bw_0)} 
        \leq \norm{f^\star}_{\fsq}\left(\norm{\sigma(\langle \bw_{t, i} \,, \,\bx \rangle)}_{\fsq} + \norm{\sigma(\langle \bw_{0, i} \,, \,\bx \rangle)}_{\fsq}\right) \cdot\epsilon.
    \end{align}
    Since
    \[
    \abs{\Lcal(\theta_t) - \Lcal(\theta_0)} \leq \sum_{i=1}^m \abs{\Lcal(\bw_{t, i}) - \Lcal(\bw_{0, i})},
    \]
    applying \eqref{eq: si_bound_gen} to every $i\in[m]$ and taking a union bound completes the proof (where the polylogarithmic dependence on $1/\delta$ in \eqref{eq: si_bound_gen} implies a polylogarithmic dependence on $m$).
\end{proof}
\end{document}